\documentclass{article}

\usepackage{microtype}
\usepackage{graphicx}
\usepackage{subfigure}
\usepackage{booktabs} 
\usepackage[small]{caption}
\usepackage{hyperref}

\usepackage{icml2023}

\usepackage{natbib}

\usepackage{url}            
\usepackage{booktabs}       
\usepackage{amsfonts}       
\usepackage{nicefrac}       
\usepackage{microtype}      

\usepackage{amsmath}
\usepackage{color}
\usepackage{amssymb}
\usepackage{amsthm}
\usepackage{graphicx}
\usepackage{dsfont}

\usepackage{subfigure}

\providecommand{\zz}{\mathbf{z}}

\providecommand{\ww}{\beta}

\providecommand{\YY}{\mathbf{Y}}
\providecommand{\XX}{\mathbf{X}}

\providecommand{\ph}{{\phi}}

\providecommand{\ddelta}{\delta}
\renewcommand{\alph}{\alpha}

\makeatletter
\newcommand*{\inlineequation}[2][]{%
  \begingroup
    \refstepcounter{equation}%
    \ifx\\#1\\%
    \else
      \label{#1}%
    \fi
    \relpenalty=10000 %
    \binoppenalty=10000 %
    \ensuremath{%
      #2%
    }%
    ~\@eqnnum
  \endgroup
}
\makeatother

\newcommand{\E}{\mathbb{E}}

\newcommand{\cZ}{\mathcal{Z}}
\newcommand{\cL}{\mathcal{L}}
\newcommand{\cS}{\mathcal{S}}
\newcommand{\cO}{\mathcal O}

\newcommand{\R}{\mathbb R}

\newcommand{\cC}{\mathcal{C}}
\newcommand{\sparse}{{\rm sparse}}
\newcommand{\Gain}{{\rm Gain}}
\renewcommand{\empty}{\varnothing}

\DeclareMathOperator{\argmin}{argmin}

\DeclareMathOperator{\diag}{diag}

\DeclareMathOperator{\Ima}{Im}
\DeclareMathOperator{\argsinh}{arcsinh}
\DeclareMathOperator{\Span}{Span}

\def\<#1,#2>{\langle #1,#2\rangle}

\definecolor{myblue}{RGB}{55,126,184}

\usepackage{dsfont}

\renewcommand{\leq}{\leqslant}
\renewcommand{\geq}{\geqslant}
\renewcommand{\le}{\leqslant}
\renewcommand{\ge}{\geqslant}

\def\<{\langle}
\def\>{\rangle}
\def\|{\Vert}

\definecolor{NavyBlue}{rgb}{0.1,0.1,0.6}

\def\eps{\varepsilon}

\newcommand{\esp}[1]{\mathbb{E}\left[#1\right]}

\newcommand{\NRM}[1]{{{\left\| #1\right\|}}} 
\newcommand{\set}[1]{{{\left\{ #1\right\}}}} 
\newcommand{\proba}[1]{\mathbb{P}\left(#1\right)}

\newcommand{\half}{\frac{1}{2}}

\renewcommand{\P}{\mathbb{P}}

\newcommand{\cB}{\mathcal{B}}
\newcommand{\cN}{\mathcal{N}}

\newcommand{\balpha}{{\bm{\alpha}}}

\newcommand{\dd}{{\rm d}}
\newcommand{\cF}{\mathcal{F}}

\newcommand{\N}{\mathbb{N}}

\newcommand{\one}{\mathbf{1}}

\hypersetup{colorlinks=true, urlcolor= blue, linkcolor=blue, citecolor=blue}

\usepackage{amsfonts}
\usepackage{amsthm}
\usepackage{amsmath}
\usepackage{amssymb}
\usepackage{mathrsfs}
\usepackage{amscd}
\usepackage{caption}
\usepackage{indentfirst}
\usepackage{cleveref}
\usepackage{mathtools}
\usepackage{thmtools} 
\usepackage{thm-restate}
\usepackage{bm}
\usepackage{multicol}
\usepackage{wrapfig}

\newcommand{\appropto}{\mathrel{\vcenter{
  \offinterlineskip\halign{\hfil$##$\cr
    \propto\cr\noalign{\kern2pt}\sim\cr\noalign{\kern-2pt}}}}}

\newtheorem{theorem}{Theorem}
\newtheorem{observation}{Observation}

\newtheorem{example}{Example}
\newtheorem{assumption}{Assumption}

\newtheorem{lemma}{Lemma}
\newtheorem{proposition}{Proposition}

\newtheorem{corollary}{Corollary}

\crefname{assumption}{Assumption}{assumptions}
\crefname{corollary}{Corollary}{Corollary}
\crefname{remark}{Remark}{Remark}
\crefname{proposition}{Proposition}{Proposition}
\crefname{lemma}{Lemma}{Lemma}
\crefname{definition}{Definition}{Definition}
\crefname{example}{Example}{Example}
\crefname{claim}{Claim}{Claim}
\crefname{theorem}{Theorem}{Theorem}
\crefname{figure}{Fig.}{Fig.}
\crefname{equation}{Eq.}{Eq.}

\newcommand{\myparagraph}{\textbf}

\makeatletter

\icmltitlerunning{(S)GD over diagonal linear networks: Implicit regularisation, large stepsize and EoS}

\begin{document}

\twocolumn[
\icmltitle{(S)GD over Diagonal Linear Networks:\\  Implicit Regularisation, Large Stepsizes and Edge of Stability}





\icmlsetsymbol{equal}{*}

\begin{icmlauthorlist}
\icmlauthor{Aeiau Zzzz}{equal,to}
\icmlauthor{Bauiu C.~Yyyy}{equal,to,goo}
\icmlauthor{Cieua Vvvvv}{goo}
\icmlauthor{Iaesut Saoeu}{ed}
\icmlauthor{Fiuea Rrrr}{to}
\icmlauthor{Tateu H.~Yasehe}{ed,to,goo}
\icmlauthor{Aaoeu Iasoh}{goo}
\icmlauthor{Buiui Eueu}{ed}
\icmlauthor{Aeuia Zzzz}{ed}
\icmlauthor{Bieea C.~Yyyy}{to,goo}
\icmlauthor{Teoau Xxxx}{ed}
\icmlauthor{Eee Pppp}{ed}
\end{icmlauthorlist}

\icmlaffiliation{to}{Department of Computation, University of Torontoland, Torontoland, Canada}
\icmlaffiliation{goo}{Googol ShallowMind, New London, Michigan, USA}
\icmlaffiliation{ed}{School of Computation, University of Edenborrow, Edenborrow, United Kingdom}


\icmlkeywords{Machine Learning, ICML}

\vskip 0.3in
]




\begin{abstract}
In this paper, we investigate the impact of stochasticity and large stepsizes on the implicit regularisation of gradient descent (GD) and stochastic gradient descent (SGD) over diagonal linear networks. We prove the convergence of GD and SGD with macroscopic stepsizes on overparametrised regression problems and characterise their solutions through an implicit regularisation problem. Our crisp characterisation leads to qualitative insights about the impact of stochasticity and stepsizes on the recovered solution. Specifically, we show that large stepsizes consistently benefit SGD for sparse regression problems, while they can hinder the recovery of sparse solutions for GD. These effects are magnified for stepsizes in a tight window just below the divergence threshold, in the ``edge of stability'' regime. Our findings are supported by experimental results.
%
\end{abstract}

\section{Introduction}

The stochastic gradient descent algorithm (SGD)
\cite{robbins1951stochastic} is the foundational algorithm for almost all neural network training. Though a remarkably simple algorithm, it has led to many impressive empirical results and is a key driver of deep learning. However the performances of SGD are quite puzzling from a theoretical point of view as (1) its convergence is highly non-trivial and (2) there exist many global minimums for the training objective which generalise very poorly~\cite{zhang2016understanding}.


To explain this second point, the concept of implicit regularisation has emerged:
if overfitting is harmless in many real-world prediction tasks, it must be because the optimisation process is \textit{implicitly favoring} solutions that have good generalisation properties for the task.
The canonical example is overparametrised linear regression with more trainable parameters than number of samples: although there are infinitely many solutions that fit the samples, GD and SGD explore only a small subspace of all the possible parameters. As a result, it can be shown that they implicitly converge to the closest solution in terms of the $\ell_2$ distance, and this without explicit regularisation \cite{zhang2016understanding,gunasekar2018implicit_geometry}. 


Currently, most theoretical works on implicit regularisation have primarily focused on continuous time approximations of (S)GD where the impact of crucial hyperparameters such as the stepsize and the minibatch size are ignored. One such common simplification is to analyse gradient flow, which is a continuous time limit of GD and minibatch SGD with an infinitesimal stepsize. By definition, this analysis does not capture the effect of stepsize or stochasticity. Another approach is to approximate SGD by a stochastic gradient flow \cite{wojtowytsch2021stochastic,pesme2021implicit}, which tries to capture the noise and the stepsize using an appropriate stochastic differential equation. However, there are no theoretical guarantees that these results can be transferred to minibatch SGD as used in practice.
This is a limitation in our understanding since the performances of most deep learning models are often sensitive to the choice of stepsize and minibatch size. 
The importance of stepsize and SGD minibatch size is common knowledge in practice and has also been systematically established in controlled experiments \cite{keskar2016large,masters2018revisiting,geiping2021stochastic}. 

In this work, we aim to expand our understanding of the impact of stochasticity and stepsizes by analysing the (S)GD trajectory in $2$-layer diagonal networks (DLNs). In \Cref{fig:main_theorem}, we show that even in our simple network, there are significant differences between the nature of the solutions recovered by SGD and GD at macroscopic stepsizes. We discuss this behavior further in the later sections.



The $2$-layer diagonal linear network which we consider is a simplified neural network that has received significant attention lately~\citep{pmlr-v125-woodworth20a,vavskevivcius2019implicit,haochen2020understanding,pillaudvivien2022labelnoise}.  Despite its simplicity, it surprisingly reveals training characteristics which are observed in much more complex architectures, such as the role of the initialisation~\citep{pmlr-v125-woodworth20a}, the role of noise~\citep{pesme2021implicit,pillaudvivien2022labelnoise}, or the emergence of saddle-to-saddle dynamics~\citep{berthier2022incremental,pesme2023saddle}. It therefore serves as an ideal proxy model for gaining a deeper understanding of complex phenomenons such as the roles of stepsizes and of stochasticity as highlighted in this paper. We also point out that implicit bias and convergence for more complex architectures such as 2-layer ReLU networks, matrix multiplication are not yet fully understood, even for the simple gradient flow. Therefore studying the subtler effects of large stepsizes and stochasticity in these settings is currently out of reach.

\begin{figure}[t]
\centering
\begin{multicols}{2}
\vspace*{-10.5pt}

    \hspace*{-25pt}
  \caption{Noiseless sparse regression with a diagonal linear network using SGD and GD, with parameters initialized at the scale of $\alpha=0.1$ (Section~\ref{sec:setup}). The test losses at convergence for various stepsizes are plotted for GD and SGD. Small stepsizes correspond to gradient flow (GF) performance. We see that increasing the stepsize improves the generalisation properties of SGD, but deteriorates that of GD.  The dashed vertical lines at stepsizes $\tilde{\gamma}_{\max}^{\text{SGD}}$ and $\tilde{\gamma}_{\max}^{\text{GD}}$ denote the largest stepsizes for which SGD and GD, respectively, converge. See  \Cref{sec:setup} for the precise experimental setting.    \label{fig:main_theorem}} 
     \begin{minipage}[c]{.8\linewidth}
\hspace*{-25pt}
\vspace*{-10.5pt}
\includegraphics[trim={0cm 0 0 0cm}, clip, width=1.1\linewidth]{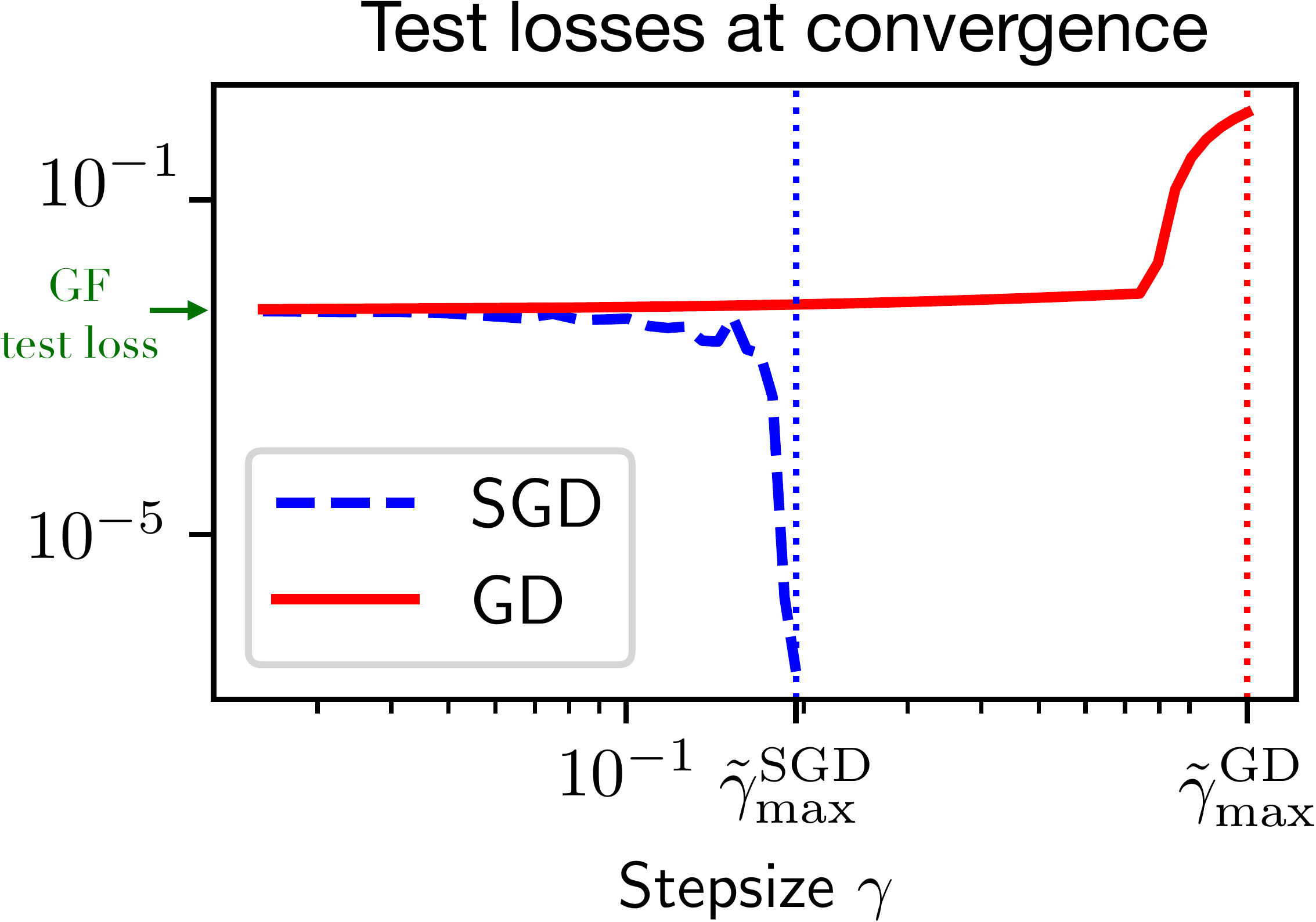}
  \end{minipage}
    \end{multicols}
\end{figure}

\subsection{Main results and paper organisation}

The overparametrised regression setting and diagonal linear networks are introduced in Section~\ref{sec:setup}.
We formulate our theoretical results (Theorems~\ref{thm:implicit_bias} and \ref{thm:conv_iterates}) in Section~\ref{sec:main_result}: we prove that for \textbf{macroscopic stepsizes}, gradient descent and stochastic gradient descent over $2$-layer diagonal linear networks converge to a zero-training loss solution $\beta^\star_\infty$. We further provide a refined characterization of $\beta^\star_\infty$ through a trajectory-dependent implicit regularisation problem, that captures the effects of hyperparameters of the algorithm, such as stepsizes and batchsizes, in useful and analysable ways. In~\Cref{sec:alphainfty} we then leverage this crisp characterisation to explain the influence of crucial parameters such as the stepsize and batch-size on the recovered solution. Importantly \textbf{our analysis shows a stark difference between the generalisation performances of GD and SGD for large stepsizes}, hence explaining the numerical results seen in \Cref{fig:main_theorem} for the sparse regression setting.
Finally, in \Cref{sec:EoS}, we use our results to shed new light on the \emph{Edge of Stability} (\emph{EoS}) phenomenon~\cite{cohen2021EoS}.

\subsection{Related works}

\myparagraph{Implicit bias.}
The concept of implicit bias from optimization algorithm in neural networks has been studied extensively in the past few years, starting with early works of \citet{telgarsky2013margins, neyshabur2014search, keskar2016large,  soudry2018implicit_separable}. The theoretical results on implicit regularisation have been extended to
multiplicative parametrisations~\citep{gunasekar2017implicit, gunasekar2018convolution}, linear networks~\citep{ji2018gradient}, and homogeneous networks~\citep{lyu2019gradient, ji2020directional,chizat2019lazy}. 
For regression loss on diagonal linear networks studied in this work, \citet{pmlr-v125-woodworth20a} demonstrate that the scale of the initialisation determines the type of solution obtained, with large initialisations yielding minimum $\ell_2$ norm solutions---the neural tangent kernel regime~\citep{jacot2018ntk} and small initialisation resulting in minimum $\ell_1$ norm solutions---the \emph{rich regime}~\citep{chizat2019lazy}.
The analysis relies on  the link between gradient descent and mirror descent  established  by \citet{pmlr-v117-ghai20a} and further explored by \citet{NEURIPS2020_024d2d69,NEURIPS2020_e9470886}.
These works focus on full batch gradient, and often in the inifitesimal stepsize limit (gradient flow), leading to general insights and results that do not take into account the effects of stochasticity and large stepsizes.

\myparagraph{The effect of stochasticity in SGD on generalisation.}
The relationship between stochasticity in SGD and generalisation has been studied in various works~\citep{mandt2016variational,hoffer2017train,chaudhari2018stochastic,kleinberg2018when,wu2019howsgd}. %
Empirically, models generated by SGD exhibit better generalisation performance than those generated by GD~\citep{keskar2017on,jastrzebski2017three,he2019controlbatchsize}. Explanations related to the flatness of the minima picked by SGD have been proposed~\citep{Hochreiter1997flatminima}. 
Label noise has been shown to influence the implicit bias of SGD~\citep{haochen2020understanding,blanc2020implicit,damian2021label,pillaudvivien2022labelnoise}  by implicitly regularising the sharp minimisers. 
Recently, studying a \emph{stochastic gradient flow} that models the noise of SGD in continuous time with Brownian diffusion, \citet{pesme2021implicit} characterised for diagonal linear networks the limit of their stochastic process as the solution of an implicit regularisation problem. 
However similar explicit characterisation of the implicit bias remains unclear for SGD with large stepsizes.

\myparagraph{The effect of stepsizes in GD and SGD.}
Recent efforts to understand how the choice of stepsizes affects the learning process and the properties of the recovered solution suggest that larger stepsizes lead to the minimisation of some notion of flatness of the loss function \cite{smith2018bayesian,keskar2017on,nacson2022implicitbiasstepsize,jastrzkebski2018width,wu2019howsgd,mulayoff2021the}, backed by empirical evidences or  stability analyses. Larger stepsizes have also been proven to be beneficial for specific architectures or problems: two-layer network \cite{li2019towards}, regression \cite{wu2021direction}, kernel regression \cite{beugnot2022largestepsizesKernels} or matrix factorisation \cite{wang2021largeLR}.
For large stepsizes, it has been observed that GD enters an \emph{Edge of Stability (EoS)} regime \cite{jastrzębski2018sharpness,cohen2021EoS}, in which the iterates and the train loss oscillate before converging to a zero-training error solution; this phenomenon has then been studied on simple toy models \cite{ahn2022eos,zhu2022eos,chen2022eos, damian2022selfstabilization} for GD. 
Recently, \cite{andriushenko2022sgd} presented empirical evidence that large stepsizes can lead to  loss stabilisation  and  towards simpler predictors.

\section{Setup and preliminaries}\label{sec:setup}

\myparagraph{Overparametrised linear regression.} 
We consider a linear regression over inputs $X=(x_1, \dots, x_n) \in (\R^d)^n$  and outputs $y=(y_1, \dots, y_n) \in \R^n$. 
We consider \textit{overparametrised} problems where  input dimension $d$ is (much) larger than the number of samples $n$. In this case, there exists infinitely many linear predictors $\beta^\star \in \R^d$ which perfectly fit the training set, \textit{i.e.}, $y_i = \langle \beta^\star, x_i \rangle$ for all $1 \leq i \leq n$. We call such vectors \textit{interpolating predictors} or \textit{interpolators} and we denote by $\mathcal{S}$ the set of all interpolators 
$ \cS = \{ \beta^\star \in \R^d \ \mathrm{s.t.} \   \langle \beta^\star, x_i \rangle = y_i, \forall i \in [n] \}$. 
Note that $\cS$ is an affine space of dimension greater than $d-n$ and equal to $\beta^\star + \mathrm{span}(x_1, \dots, x_n)^\perp$ for any $\beta^\star\in\cS$. We consider the following quadratic loss: $\cL(\beta)= \frac{1}{2n}\sum_{i=1}^n(\langle \beta, x_i \rangle-y_i)^2$, for $\beta\in\R^d$.

\myparagraph{2-layer linear diagonal network.}
We parametrise regression vectors $\beta$ as functions $\beta_w$ of trainable parameters $w \in \R^p$.
Although the final prediction function $x \mapsto \langle \beta_w, x \rangle$ is linear in the input $x$, the choice of the  parametrisation drastically changes the solution recovered by the optimisation algorithm \cite{gunasekar2018convolution}.
In the case of the linear parametrisation $\beta_w = w$ many first-order methods (SGD, GD, with or without momentum) converge towards the same solution and  the choice of stepsize does not impact the recovered solution beyond convergence.
%
In an effort to better understand the effects of stochasticity and large stepsize, we consider the next simple parametrisation, that of a $2$-layer diagonal linear neural network given by:
\begin{align}
\label{eq:DLN}
\beta_w = u\odot v \text{ where } 
w = (u,v) \in\R^{2d}\,.
\end{align}
This parametrisation can be viewed as a simple neural network $x \mapsto \langle u, \sigma(\diag(v) x) \rangle$ where the output weights are represented by $u$, the inner weights is the diagonal matrix $\diag(v)$, and the activation $\sigma$ is the identity function.
In this spirit, we refer to the entries of $w = (u, v) \in \R^{2d}$ as the \textit{weights} and to
$\beta \coloneqq u \odot v \in \R^d$ as the \textit{prediction parameter}.
%
%
Despite the simplicity of the parametrisation~\eqref{eq:DLN}, the loss function $F$ over parameters $w =(u,v)\in \R^{2d}$ is \textbf{non-convex} (and thus the corresponding optimization problem is challenging to analyse), and is given by:
\begin{align}
\label{eq:non-convex-F-loss}
F(w) \coloneqq \cL(u\odot v)= \frac{1}{2 n} \sum_{i=1}^n (y_i - \langle u \odot v, x_i \rangle)^2\,.
\end{align}

\myparagraph{Mini-batch SGD.} 
We minimise $F$ using mini-batch SGD: let $w_0=(u_0,v_0)$ and for $k\geq0$,
\begin{equation}\label{eq:SGD_recursion}
    w_{k+1} = w_{k} - \gamma_k \nabla F_{\cB_k}(w_k)\,,\quad \text{where} \quad F_{\cB_k}(w) \coloneqq  \frac{1}{2b}\sum_{i\in\cB_k}(y_i - \langle u \odot v ,x_i\rangle)^2\,,
\end{equation}
where $\gamma_k$ are stepsizes, $\cB_k\subset[n]$  are mini-batches of $b \in [n]$ distinct samples sampled uniformly and independently, and $\nabla F_{\cB_k}(w_k)$ are minibatch gradients of partial loss over $\cB_k$, $F_{\cB_k}(w) \coloneqq \cL_{\cB_k}(u \odot v)$ defined above.
%
%
%
%
%
%
Classical SGD and full-batch GD are special cases with $b=1$ and $b=n$, respectively.
For $k \geq 0$, we consider the successive prediction parameters $\beta_k \coloneqq u_k \odot v_k$ built from the weights $w_k=(u_k, v_k)$.
We analyse SGD initialised at $u_0 = \sqrt{2} \balpha \in \R_{>0}^d$ and $v_0 = \mathbf{0} \in \R^d$, resulting in $\beta_0 =\mathbf{0} \in \R^d $ independently of the chosen weight initialisation $\balpha$\footnote{In Appendix~\ref{sec:app:param}, we show that the (S)GD trajectory with this initialisation exactly matches that of another common parametrisation $\beta_w=w_+^2-w_-^2$ with initialisation $w_{+,0}=w_{-,0}=\balpha$. 
The second layer of our diagonal linear network is set to 0 in order to obtain results that are easier to interpret. However, our proof techniques can be applied directly to a general initialisation, at the cost of additional notations in our Theorems.}.

\myparagraph{Experimental details.} We consider the noiseless sparse regression setting where $(x_i)_{i \in [n]} \sim \mathcal{N}(0, I_d)$ and $y_i = \langle \beta^\star_{\ell_1}, x_i \rangle$ for some $s$-sparse vector $\beta^\star_{\ell_1}$.  We perform (S)GD over the DLN with a uniform initialisation $\balpha = \alpha \mathbf{1} \in \R^d$ where $\alpha > 0$. \Cref{fig:main_theorem} and \Cref{fig:Gain} (left) correspond to the setup $(n, d, s, \alpha) = (20, 30, 3, 0.1)$, \Cref{fig:Gain} (right) to $(n, d, s, \alpha) = (50, 100, 4, 0.1)$  and \Cref{fig:EoS} to $(n, d, s, \alpha) = (50, 100, 2, 0.1)$.
\myparagraph{Notations.} 
Let $H \coloneqq \nabla^2\cL= \frac{1}{n} \sum_i x_i x_i^\top$ denote the Hessian of $\cL$, and for a batch $\cB \subset [n]$ 
let $H_\cB \coloneqq \nabla^2 \cL_\cB =\frac{1}{|\cB|}\sum_{i\in\cB}x_ix_i^\top$ denote the Hessian of the partial loss over the batch $\cB$. 
Let $L$ denote the ``smoothness'' such that $\forall\beta$, $\Vert H_\cB \beta \Vert_2 \leq L \Vert \beta \Vert_2$, $\NRM{H_\cB \beta}_\infty\leq L\NRM{\beta}_\infty$ for all batches $\cB \subset [n]$ of size $b$.
A real function  (e.g, $\log,\exp$) applied to a vector must be understood as element-wise application, and for vectors $u,v\in\R^d$, $u^2=(u_i^2)_{i\in[d]}$, $u\odot v=(u_iv_i)_{i\in[d]}$ and $u / v = (u_i / v_i)_{i\in[d]}$. 
We write $\mathbf{1}$, $\mathbf{0}$ for the constant vectors with coordinates $1$ and $0$ respectively.
The Bregman divergence \cite{BREGMAN1967200} of a differentiable convex function $h : \R^d \to \R$ is defined as $D_h(\beta_1,\beta_2)=h(\beta_1)-(h(\beta_2) + \langle \nabla h(\beta_2),\beta_1-\beta_2\rangle)$.

\section{Implicit bias of SGD and GD} \label{sec:main_result}

We start by recalling some known results on the implicit bias of gradient flow on diagonal linear networks before presenting our main theorems on characterising the  (stochastic) gradient descent solutions (\cref{thm:implicit_bias}) as well as proving the convergence of the iterates (\cref{thm:conv_iterates}).

\subsection{Warmup: gradient flow}\label{sec:warmup}

We first review prior findings on gradient flow on diagonal linear neural networks. \citet{pmlr-v125-woodworth20a} show that the limit   $\beta_{\balpha}^*$  of the \emph{gradient flow} $\dd w_t = - \nabla F(w_t) \dd t$ initialised at $(u_0, v_0) = (\sqrt{2} \balpha, \mathbf{0})$  is the solution of the  minimal interpolation problem:
\begin{equation}\label{eq:implicit_opt_hypentropy}
    \beta_\balpha^* = \underset{ \beta^\star \in \cS }{\argmin} \  \psi_\balpha(\beta^\star)\,,\quad \text{where}\quad     \psi_\balpha(\beta) = \frac{1}{2} \sum_{i=1}^d \Big ( \beta_i \mathrm{arcsinh}( \frac{\beta_i}{\alpha_i^2} ) \! -\! \sqrt{\beta_i^2 + \alpha_i^4} +  \alpha_i^2 \Big )\,.
\end{equation}
The convex potential $\psi_\balpha$ is the \textbf{hyperbolic entropy function} (or \textbf{hypentropy}) \citep{pmlr-v117-ghai20a}.
Depending on the structure of the vector $\balpha$, the generalisation properties of $\beta^\star_\balpha$ highly vary.
We point out the two main characteristics of $\balpha$ that affect the behaviour of~$\psi_\balpha$ and therefore also the solution $\beta^\star_\balpha$.

\textbf{1.} The \textbf{Scale} of $\balpha$. For an initialisation vector $\balpha$ we call the $\ell_1$-norm $\Vert \balpha\Vert_1$ the \textbf{scale} of the initialisation. It is an important quantity affecting the properties of the recovered solution $\beta^\star_\balpha$. To see this let us consider a uniform initialisation of the form $\balpha=\alpha \bm{1}$ for a scalar value $\alpha > 0$. In this case the potential $\psi_{\balpha}$ has the property of resembling the $\ell_1$-norm as the scale $\alpha$ vanishes:  $\psi_{\balpha} \sim \ln(1/\alpha)\|.\|_1$ as $\alpha\to0$.
Hence, a small initialisation results in a low $\ell_1$-norm solution which is known to induce sparse recovery guarantees \citep{candestao}.
This setting is often referred to as the ``rich'' regime \citep{pmlr-v125-woodworth20a}.
%
In contrast, using a large initialisation scale leads to solutions with low $\ell_2$-norm: $\psi_{\balpha} \sim  \|.\|_2^2 / (2 \alpha^2)$ as $\alpha\to\infty$, a setting known as the ``kernel'' or ``lazy'' regime. 
%
Overall, to retrieve the minimum $\ell_1$-norm solution,
one should use a uniform initialisation with small scale $\alpha$, see \Cref{fig:weighted_l1_norm} in Appendix~\ref{sec:app:example} for an illustration and \cite[Theorem 2]{pmlr-v125-woodworth20a} for a precise characterisation.


\textbf{2.} The \textbf{Shape} of $\balpha$.
In addition to the scale of the initialisation $\balpha$, a lesser studied aspect is its ``shape'', which is a term we use to refer to the relative distribution of $\{\alpha_i\}_i$ along the $d$ coordinates \cite{azulay2021implicit}.  It is a crucial property because having $\balpha \to \mathbf{0}$  \textbf{does not} necessarily lead to the potential $\psi_\balpha$ being close to the $\ell_1$-norm. Indeed, we have that $\psi_\balpha(\beta) \overset{\balpha \to \bm{0}}{\sim} \sum_{i=1}^d \ln(\frac{1}{ \alpha_i}) | \beta_i |$ (see Appendix~\ref{sec:app:example}), therefore if the vector $\ln(1 / \balpha)$ has entries changing at different rates, then $\psi_\balpha(\beta)$ is a \textbf{weighted} $\ell_1$-norm. In words, if the entries of $\balpha$ \textit{do not go to zero ``uniformly"}, then the resulting implicit bias minimizes a weighed $\ell_1$-norm. This phenomenon can lead to solutions with vastly different sparsity structure than the minimum $\ell_1$-norm interpolator. See \Cref{fig:weighted_l1_norm} and \Cref{app:example} in Appendix~\ref{sec:app:example}.

\subsection{Implicit bias of (stochastic) gradient descent}

In \cref{thm:implicit_bias}, we prove that for an initialisation $\sqrt{2} \balpha\in\R^d$ and for \textbf{arbitrary} stepsize sequences $(\gamma_k)_{k\geq0}$ \textbf{if the iterates converge to an interpolator}, then this interpolator is the solution of a constrained minimisation problem which involves the hyperbolic entropy $\psi_{\balpha_\infty}$ defined in \eqref{eq:implicit_opt_hypentropy}, where $\balpha_\infty\in\R^d$  is an {effective} initialisation which depends on the trajectory and on the stepsize sequence. Later, \textbf{we prove the convergence of iterates for macroscopic step sizes} in \cref{thm:conv_iterates}.


\begin{theorem}[Implicit bias of (S)GD] \label{thm:implicit_bias}
    Let $(u_k,v_k)_{k\geq0}$ follow the mini-batch SGD recursion \eqref{eq:SGD_recursion} initialised at $(u_0, v_0) = (\sqrt{2} \balpha, \mathbf{0})$ and with stepsizes $(\gamma_k)_{k\geq0}$. Let $(\beta_k)_{k\geq0}=(u_k\odot v_k)_{k\geq0}$ and assume that they converge to some interpolator $\beta_\infty^\star\in\cS$.
    Then, $\beta^\star_\infty$ satisfies:
    \vspace*{0.5em}
    \begin{equation}\label{eq:implicit_bias}    
        \beta^\star_{\infty} =  \underset{ \beta^\star \in \cS }{\argmin} \  D_{\psi_{\balpha_\infty}}(\beta^\star,\tilde\beta_0) \,,
    \end{equation}
    where $D_{\psi_{\balpha_\infty}}$ is the Bregman divergence with hyperentropy potential $\psi_{\balpha_\infty}$ of the \textbf{effective initialisation} $\balpha_\infty$, and $\tilde{\beta}_0$ is a small \textbf{perturbation term}. 
    The \textbf{effective initialisation} ${\balpha_\infty}$ is given by, 
    \begin{equation}\label{eq:def_alpha_inf}
        \balpha_\infty^2=\balpha^2\odot\exp\left(-\sum_{k=0}^\infty q\big(\gamma_k \nabla \cL_{\cB_k}(\beta_k)\big)\right)\,,
    \end{equation}
    where $q(x)=-\frac{1}{2} \ln((1-x^2)^2)$ satisfies $q(x)\geq0$ for $|x|\leq \sqrt{2}$, with the convention $q(1) = + \infty$. 
    
    The \textbf{perturbation term}  
    $\Tilde{\beta}_0\in\R^d$ is explicitly  
    given by $\tilde{\beta}_0=\half\big(\balpha_+^2-\balpha_-^2\big)$, where $q_\pm(x)= \mp 2x - \ln((1\mp x)^2)$, and  $\balpha_{\pm}^2=\balpha^2\odot\exp\left(-\sum_{k=0}^\infty q_\pm(\gamma_k \nabla \cL_{\cB_k}(\beta_k))\right)$.
\end{theorem}

\myparagraph{Trajectory-dependent characterisation.} The characterisation of $\beta^\star_\infty$ in \Cref{thm:implicit_bias} holds for any stepsize schedule such that the iterates converge and goes beyond the continuous-time frameworks previously studied \cite{pmlr-v125-woodworth20a,pesme2021implicit}. The result even holds for adaptive stepsize schedules which keep the stepsize scalar such as AdaDelta \citep{zeiler2012adadelta}.
An important aspect of our result is that $\balpha_\infty$ and $\tilde{\beta}_0$ depend on the iterates' trajectory.
Nevertheless, we argue that our formulation provides  useful ingredients for understanding the implicit regularisation effects of (S)GD for this problem compared to trivial characterisations (such as \emph{e.g.}, $\min_\beta\NRM{\beta-\beta^\star_\infty}$).
Importantly, \textbf{the key parameters $\balpha_\infty,\tilde\beta_0$ depend on crucial parameters such as the stepsize and noise in a useful and analysable manner}: understanding how they affect $\balpha_\infty$ and $\tilde \beta_0$ coincides with understanding how they affect the recovered solution $\beta^\star_\infty$ and its generalisation properties.
This is precisely the object of \Cref{sec:EoS,sec:alphainfty} where we discuss the qualitative and quantitative insights from \Cref{thm:implicit_bias} in greater detail.


\myparagraph{The perturbation $\Tilde{\beta}_0$ can be ignored.}   
We show in Proposition~\ref{app:tilde_beta0}, under reasonable assumptions on the stepsizes, that $|\tilde\beta_0|\leq \balpha^2$ and $\balpha_\infty\leq \balpha$ (component-wise). The magnitude of $\Tilde{\beta}_0$ is therefore negligible in front of the magnitudes of $\beta^\star \in S$ and one can roughly ignore the term  $\Tilde{\beta}_0$. Hence, the implicit regularisation \cref{eq:implicit_bias} can be thought of as $\beta^\star_\infty \approx \argmin_{\beta^\star \in S} D_{\psi_{\balpha_\infty}}(\beta^\star,0)=\psi_{\balpha_\infty}(\beta^\star)$, and thus \emph{the solution $\beta^\star_\infty$ minimises the same potential function that the solution of gradient flow (see \Cref{eq:implicit_opt_hypentropy}), but with an {effective initialisation} $\balpha_\infty$.}
Also note that for $\gamma_k\equiv\gamma \to 0$ we have  $\balpha_\infty \to \balpha$ and $\tilde\beta_0\to \mathbf{0}$ (\cref{prop:conv_beta_alpha}), recovering the previously known result for gradient flow \eqref{eq:implicit_opt_hypentropy}.

\myparagraph{Deviation from gradient flow.} The difference with gradient flow is directly associated with the quantity $\sum_k q(\gamma_k \nabla \cL_{\cB_k}(\beta_k))$.
Also, as the (stochastic) gradients converge to 0 and $q(x) \overset{x \to 0}{\sim} x^2$, one should think of this sum as roughly being $\sum_k \nabla \cL_{\cB_k}(\beta_k)^2$:  the larger this sum, the more the recovered solution differs from that of  gradient flow.
The  full picture of how large stepsizes and stochasticity impact the generalisation properties of $\beta^\star_\infty$ and the recovery of minimum $\ell_1$-norm solution is nuanced  as clearly seen in \cref{fig:main_theorem}.

\subsection{Convergence of the iterates}
\Cref{thm:implicit_bias} provides the implicit minimisation problem but says nothing about the convergence of the iterates. Here we show under very reasonable assumptions on the stepsizes that the iterates indeed converge  towards a global optimum. Note that since the loss $F$ is non-convex, such a convergence result is non-trivial and requires an involved analysis. 

\begin{theorem}[Convergence of the iterates]\label{thm:conv_iterates}
    Let $(u_k,v_k)_{k\geq0}$ follow the mini-batch SGD recursion~\eqref{eq:SGD_recursion} initialised at $u_0=\sqrt{2} \balpha\in\R_{>0}^d$ and $v_0=\mathbf{0}$, and let $(\beta_k)_{k\geq0}=(u_k\odot v_k)_{k\geq0}$. Recall the ``smoothness'' parameter $L$ on the minibatch loss defined in the notations. 
There exist $B>0$ verifying $B=\tilde\cO(\min_{\beta^\star\in\cS}\NRM{\beta^\star}_\infty)$ and a numerical constant $c>0$ such that for stepsizes satisfying $\gamma_k\leq \frac{c}{LB}$, the iterates
$(\beta_k)_{k\geq0}$ converge almost surely to the interpolator $\beta_\infty^\star$ solution of \Cref{eq:implicit_bias}.
\end{theorem}
In fact, we can be more precise by showing an exponential rate of convergence of the losses as well as characterise the rate of convergence of the iterates as follows.

\begin{proposition}[Quantitative convergence rates]\label{prop:conv_quantitative}
For a uniform initialisation $\balpha = \alpha \mathbf{1}$ and under the assumptions of \Cref{thm:conv_iterates}, we have:
    \begin{equation*}
        \esp{\cL(\beta_k)}\leq \left(1-\frac{1}{2}\gamma\alpha^2\lambda_b\right)^k\cL(\beta_0)\quad\text{and}\quad         \esp{\NRM{\beta_k-\beta^\star_{\alpha_k}}^2}\leq C\left(1-\frac{1}{2}\gamma\alpha^2\lambda_b\right)^k\,,
    \end{equation*}
    where $\lambda_b>0$ is the largest value such that $\lambda_b H\preceq \E_\cB[H_\cB]$,
    $C=2B(\alpha^2\lambda^+_{\min})^{-1}\left(1+(4B\lambda_{\max})(\alpha^2\lambda_{\min}^+)^{-1}\right)\cL(\beta_0)$ and $\lambda_{\min}^+,\lambda_{\max}>0$ are respectively the smallest non-null and the largest eigenevalues of $H$, and $\beta^\star_{\alpha_k}$ is the interpolator that minimises the perturbed hypentropy $h_k$ of parameter $\alpha_k$, as defined in \Cref{eq:def_hk_main} in the next subsection.
\end{proposition}

The convergence of the losses is proved directly using the time-varying mirror structure that we exhibit in the next subsection, the convergence of the iterates is proved by studying the curvature of the mirror maps on a small neighborhood around the affine interpolation  space.

\subsection{Sketch of proof through a time varying mirror descent} 

As in the continuous-time framework, our results heavily rely on showing that the iterates $(\beta_k)_k$ follow a mirror descent recursion with time-varying potentials on the convex loss $\cL(\beta)$. To show this, we first define the following quantities:
\begin{align*}
&\balpha^2_k \coloneqq \balpha_{+, k} \odot \balpha_{-, k} \qquad  \text{and}\qquad\phi_k \coloneqq \frac{1}{2} \argsinh \left (  \frac{\balpha_{+, k}^2 - \balpha_{-, k}^2 }{2 \balpha_{k}^2 } \right ) \in \R^d
\,,
\end{align*}
where $\balpha_{\pm, k} \coloneqq \balpha  \exp\left( - \frac{1}{2} \sum_{i = 0}^{k-1} q_\pm\big( \gamma_\ell \nabla \cL_{\cB_\ell}(\beta_\ell)  \big)  \right) \in \R^d$.
Finally for $k \geq 0$, we define the potentials  $(h_k : \R^d \to \R)_{k\geq0}$ as:
\begin{align}
\label{eq:def_hk_main}
    h_k(\beta) = \psi_{\balpha_k}(\beta) - \langle \phi_k, \beta \rangle.
\end{align}
Where $\psi_{\balpha_k}$ is the hyperbolic entropy function defined \Cref{eq:implicit_opt_hypentropy}.  Now that all the relevant quantities are defined, we can state the following proposition which explicits the time-varying stochastic mirror descent.
\begin{proposition}\label{prop:tv_md_main}
The iterates $(\beta_k=u_k\odot v_k)_{k\geq0}$ from \Cref{eq:SGD_recursion} satisfy the Stochastic Mirror Descent recursion with varying potentials $(h_k)_k$:
\begin{align*}
   \nabla h_{k+1}(\beta_{k+1}) = \nabla h_{k}(\beta_{k}) - \gamma_k \nabla \cL_{\cB_k}(\beta_k)\,,
\end{align*}
where $h_k:\R^d\to\R$ for $k\geq 0$ are defined \Cref{eq:def_hk_main}. Since $\nabla h_0(\beta_0) = 0$ we have:
\begin{align}\label{eq:spanX}
    \nabla h_k(\beta_k) \in \mathrm{span}(x_1, \dots, x_n).
\end{align} 
\end{proposition}
\Cref{thm:implicit_bias,thm:conv_iterates,prop:conv_quantitative} follow from this key proposition: by suitably modifying classical convex optimization techniques to account for the time-varying potentials, we can prove the convergence of the iterates towards an interpolator $\beta^\star_\infty$ along with that of the relevant quantities $\balpha_{\pm, k}$, $\balpha_{k}$ and $\phi_k$.
The implicit regularisation problem then directly follows from: (1) the limit condition $\nabla h_\infty(\beta_\infty)\in\Span(x_1,\ldots,x_n)$ as seen from \cref{eq:spanX} and (2) the interpolation condition $X \beta^\star_\infty = y$. Indeed,  these two conditions exactly correspond to the KKT conditions of the convex problem \cref{eq:implicit_bias}.



\section{Analysis of the impact of the stepsize and stochasticity on  $\alpha_\infty$}\label{sec:alphainfty}

In this section, we analyse the effects of large stepsizes and stochasticity on the implicit bias of (S)GD. We focus on how these factors influence the effective initialisation $\balpha_\infty$, which plays a key role as shown in \Cref{thm:implicit_bias}.
From its definition in \cref{eq:def_alpha_inf}, we see that $\balpha_\infty$ is a function of the vector $\sum_k q(\gamma_k \nabla \cL_{\cB_k} (\beta_k))$.
We  henceforth call this quantity the \textit{gain vector}. 
For simplicity of the discussions, from now on, we consider constant stepsizes $\gamma_k = \gamma$ for all $k \geq 0$ and a uniform initialisation of the weights $\balpha = \alpha \mathbf{1}$ with $\alpha >0$. We can then write the gain vector as: 
\begin{align*}
    \Gain_{\gamma} \coloneqq \ln \left ( \frac{\balpha^2}{\balpha_\infty^2} \right ) = \sum_k q(\gamma \nabla \cL_{\cB_k} (\beta_k))  \in \R^d\,.
\end{align*}
Following our discussion in \cref{sec:warmup} on the scale and the shape of $\balpha_\infty$, we recall the link between the scale and shape of $\rm Gain_\gamma$ and the recovered solution:

\textbf{1.} The \textbf{scale} of $\rm Gain _\gamma$, i.e. the magnitude of $\Vert \rm Gain _\gamma \Vert_1$ indicates how much the implicit bias of (S)GD differs from that of gradient flow: $\Vert \rm Gain _\gamma \Vert_1 \sim 0$ implies that $\balpha_\infty \sim \balpha$ and therefore the recovered solution is close to that of gradient flow. On the contrary, $\Vert \rm Gain _\gamma \Vert_1 >\!\!> \ln(1 / \alpha)$ implies that $\balpha_\infty$ has effective scale much smaller than $\balpha$ thereby changing the implicit regularisation \cref{eq:implicit_bias}. 

\textbf{2.} The \textbf{shape} of $\rm Gain _\gamma$ indicates which coordinates of $\beta$ in the associated minimum weighted $\ell_1$ problem are most penalised. First recall from \Cref{sec:warmup} that a uniformly large $\rm Gain _\gamma$ leads to $\psi_{\balpha_\infty}$ being closer to the $\ell_1$-norm.
However, with small weight initialisation $\alpha\to0$, we have,
\begin{align} 
\psi_{\balpha_\infty}(\beta) \sim 
    \ln(\frac{1}{\alpha}) \Vert \beta \Vert_1 + \sum_{i=1}^d  \rm Gain _\gamma(i) \vert \beta_i \vert\,,
    \label{eq:shape}
\end{align}
In this case, having a heterogeneously large vector $\rm Gain_\gamma$ leads to a weighted $\ell_1$ norm as the effective implicit regularisation, where the coordinates of $\beta$ corresponding to the largest entries of $\rm Gain _\gamma$ are less likely to be recovered.



\subsection{The scale of $\rm Gain_\gamma$ is increasing with the stepsize}




The following proposition highlights the dependencies of the scale of the gain $\Vert \rm Gain_\gamma \Vert_1$ in terms of various problem constants. 
\begin{restatable}{proposition}{propmagnitudeIB}
\label{prop:magnitude_IB}
Let $\Lambda_b, \lambda_b > 0$
\footnote{$\Lambda_b, \lambda_b > 0$ are data-dependent constants; for $b = n$, we have $(\lambda_n, \Lambda_n) = (\lambda_{\rm min}^+(H), \lambda_{\rm max}(H))$ where $\lambda_{\rm min}^+(H)$ is the smallest non-null eigenvalue of $H$; for $b = 1$, we have $\min_i \Vert x_i \Vert_2^2 \leq \lambda_1 \leq \Lambda_1 \leq \max_i \Vert x_i \Vert_2^2$. }
be the largest and smallest values, respectively, such that $\lambda_b H \preceq \mathbb{E}_\cB \big[ H_\cB^2 \big] \preceq \Lambda_b H$. For any stepsize $\gamma > 0$ satisfying $\gamma\leq \frac{c}{BL}$ (as in \cref{thm:conv_iterates}), initialisation $\alpha \mathbf{1}$ and batch size $b \in [n]$, the magnitude of the gain satisfies:
\begin{equation} \label{eq:gain_bound}
\!\lambda_b \gamma^2 \sum_k \E\cL(\beta_k)  \leq \mathbb{E} \left[\Vert \Gain_\gamma \Vert_1\right] \leq 2\Lambda_b  \gamma^2 \sum_k\E\cL(\beta_k)\,,
\end{equation}
where the expectation is over a  uniform and independent sampling of the batches $(\cB_k)_{k\geq 0}$.
\end{restatable}

\myparagraph{The slower the training, the larger the gain.} \cref{eq:gain_bound} shows that the slower the training loss converges to $0$, the larger the sum of the loss and therefore the larger the scale of $\rm Gain _\gamma$. 
This means that the (S)GD trajectory deviates from that of gradient flow if the stepsize and/or noise slows down the training.
This supports observations previously made from stochastic gradient flow~\citep{pesme2021implicit} analysis. 

\myparagraph{The bigger the stepsize, the larger the gain.}
The effect of the stepsize on the magnitude of the gain is not directly visible in \cref{eq:gain_bound} because a larger stepsize tends to speed up the training.
For stepsize $0< \gamma \leq \gamma_{\max} =\frac{c}{BL}$ as in \Cref{thm:conv_iterates} we have that (see \Cref{sec:app:gain}): 
\begin{equation}\label{eq:sum_losses}
    \sum_k \gamma^2\cL(\beta_k) = \Theta\left(\gamma \ln\left(\frac{1}{\alpha}\right)\NRM{\beta^\star_{\ell_1}}_1\right)\,.
\end{equation}
\cref{eq:sum_losses} clearly shows that increasing the stepsize \textbf{boosts} the magnitude $\Vert \rm Gain _\gamma \Vert_1$ up until the limit of $\gamma_{\text{max}}$.
Therefore, the larger the stepsize the smaller is the effective scale of $\balpha_\infty$. In turn, larger gap between $\balpha_\infty$ and $\balpha$ leads to a larger deviation of (S)GD from the gradient flow.

\myparagraph{Large stepsizes and Edge of Stability.} 
The previous paragraph holds for stepsizes smaller than $\gamma_{\max}$ for which we can theoretically prove convergence. But what if we use even bigger stepsizes? Let $(\beta_k^{\gamma})_k$ denote the iterates generated with stepsize $\gamma$ and let us define $\tilde{\gamma}_{\max} \coloneqq \sup_{\gamma \geq 0} \{ \gamma \ \text{s.t.}\ \forall\gamma' \in (0, \gamma), \ \sum_k \cL(\beta_k^{\gamma'}) < \infty \} $, which corresponds to the largest stepsize such that the iterates still converge for a given problem (even if not provably so). From \Cref{prop:magnitude_IB} we have that $\gamma_{\max} \leq \tilde{\gamma}_{\max}$. As we approach this upper bound on convergence  $\gamma \to \Tilde{\gamma}_{\max}$, the sum $\sum_k \cL(\beta_k^\gamma)$ diverges.
For such large stepsizes, the iterates of gradient descent tend to ``bounce'' and this regime is commonly referred to as the \textit{Edge of Stability}. In this regime, the convergence of the loss can be made arbitrarily slow due to these bouncing effects. As a consequence, as seen through \Cref{eq:gain_bound}, the magnitude of $\rm \Gain _\gamma$ can be become arbitrarily big as observed in \cref{fig:Gain} (left). In this regime, the recovered solution tends to dramatically differ from the gradient flow solution, as seen in \cref{fig:main_theorem}.


\myparagraph{Impact of stochasticity and linear scaling rule.}
Assuming inputs $x_i$ sampled from $\mathcal{N}(0, \sigma^2 I_d)$ with $\sigma^2 > 0$,
we obtain $\mathbb{E} \left[\Vert \Gain_\gamma \Vert_1\right] = \Theta  \Big( \gamma \frac{\sigma^2 d}{b}  \ln \big ( \frac{1}{\alpha} \big) \Vert \beta^\star_{\ell_1} \Vert_1 \Big)\,$, w.h.p. over the dataset (see~\Cref{sec:app:linear_scaling_rule}, \cref{prop:lambda_b_sum}). 
The scale of $\rm Gain _\gamma$ decreases with  batch size and there exists a factor $n$ between that of SGD and that of GD. Additionally, the magnitude of $\rm Gain_\gamma$ depends on $\frac{\gamma}{b}$, resembling the \textbf{linear scaling rule} commonly used in deep learning~\cite{goyal2017accurate}.

By analysing the magnitude $\Vert \rm Gain _\gamma \Vert_1$, we have explained \textbf{the distinct behavior of (S)GD with large stepsizes compared to gradient flow}.  However, our current analysis does not qualitatively distinguish the behavior between SGD and GD beyond the linear stepsize scaling rules, in contrast with~\cref{fig:main_theorem}. A deeper understanding of the shape of $\Gain{\gamma}$ is needed to explain this disparity.

\subsection{The shape of $\rm Gain _\gamma$ explains the differences between GD and SGD}

In this section, we restrict our presentation to single batch SGD ($b=1$) and full batch GD ($b=n$).
When  visualising  the typical shape of $\rm Gain _\gamma$ for large stepsizes (see \Cref{fig:Gain} - right), we note that GD and SGD behave very differently.  
For GD,  the  magnitude of $\rm Gain _\gamma$ is higher for coordinates in the support of $\beta^\star_{\ell_1}$ and thus these coordinates are adversely weighted in the asymptotic limit of $\psi_{\balpha_\infty}$ (per \eqref{eq:shape}). This explains the distinction seed in \cref{fig:main_theorem}, where GD in this regime has poor sparse recovery despite having a small scale of $\balpha_\infty$, as opposed to SGD that behaves well. 

\begin{figure}[t]
\centering
\includegraphics[width=0.9\linewidth]{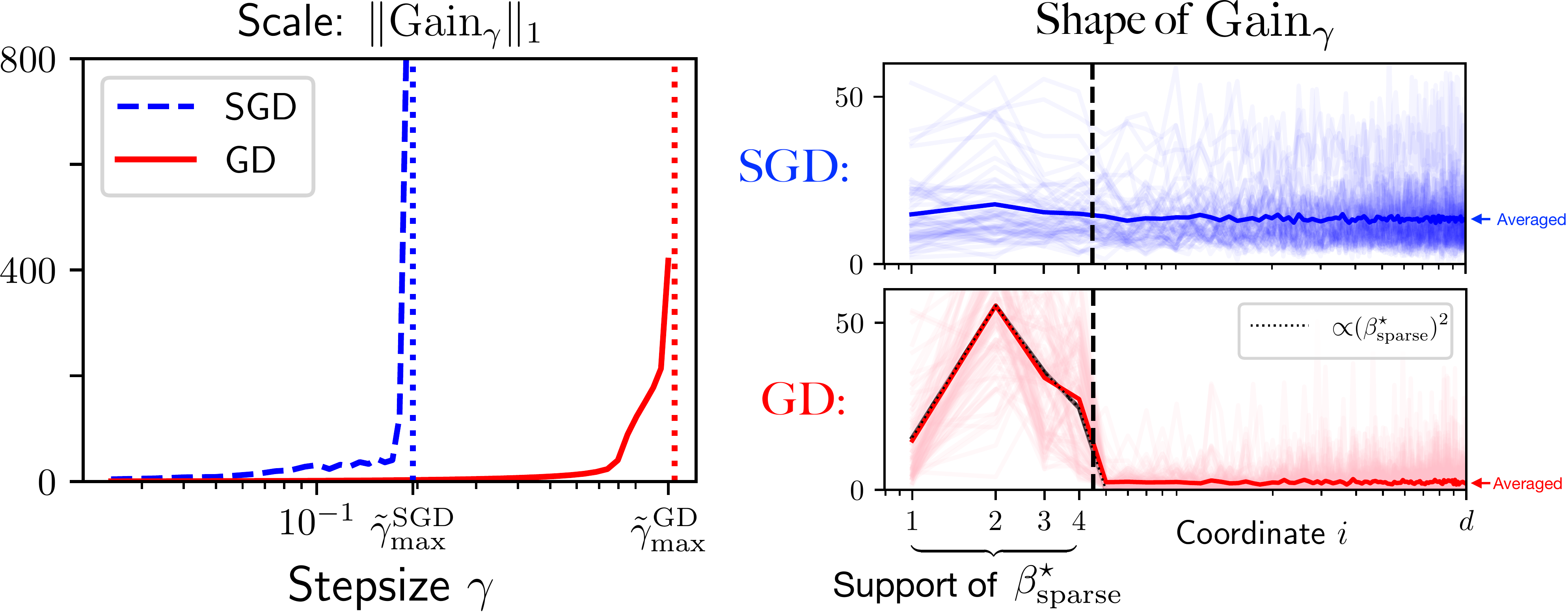}
\hspace*{-15pt}
\caption{ \textit{Left:} the scale of $\rm Gain_\gamma$ explodes as $\gamma \to \tilde{\gamma}_{\rm max}$ for both GD and SGD. 
\textit{Right:} $\beta^\star_{\rm sparse}$ is fixed, we perform $100$ runs of GD and SGD with different feature matrices, and we plot the $d$ coordinates of $\Gain_\gamma$ (for GD and SGD) on the $x$-axis (which is in log scale for better visualisation). The shape of $\rm Gain_\gamma^{\rm SGD}$ is homogeneous whereas that of GD is heterogeneous with much higher magnitude on the support of $\beta^\star_{\rm sparse}$. The shape of $\rm Gain_{\gamma}^{\rm GD}$ is proportional to the expected gradient at initialisation which is $(\beta^\star_{\rm sparse})^2$.\label{fig:Gain}} 
\end{figure}
%
%
%
%
%
The \textbf{shape} of $\rm Gain _\gamma$ is determined by the sum of the squared gradients $\sum_k \nabla \cL_{\cB_k}(\beta_k)^2$, and in particular by the degree of heterogeneity among the coordinates of this sum.
Precisely analysing the sum over the whole trajectory of the iterates $(\beta_k)_k$ is technically out of reach. However, we empirically observe for the trajectories shown in \Cref{fig:Gain} that the shape is largely determined within the first few iterates as formalized in the observation below. 
\begin{observation}
\label{claim:shady}
 $\sum_k \nabla \cL_{\cB_k}(\beta_k)^2 \appropto \mathbb{E}[ \nabla \cL_{\cB_k}(\beta_0)^2]$\,.
\end{observation}
In the simple case of a Gaussian  noiseless sparse recovery problem (where $y_i=\langle \beta^\star_{\rm sparse},x_i\rangle$ for some sparse vector $\beta^\star_{\rm sparse}$), we can control these gradients for GD and SGD (\Cref{sec:app:stoch_grad}) as:
\begin{align}
    &\nabla \cL (\beta_0  )^2 = (\beta^\star_{\rm sparse})^2  + \varepsilon \,, \text{ for some  } \eps \text{ verifying } \NRM{\eps}_\infty <\!< \NRM{\beta^\star_\sparse}_\infty^2 \label{eq:nabla0_GD}\,, \\
    &\mathbb{E}_{i_0} [\nabla \cL_{i_0}(\beta_0)^2] = \Theta \Big(  \Vert \beta^\star_{\rm sparse} \Vert_2^2 \mathbf{1} \Big)\,. \label{eq:nabla0_SGD}
\end{align}

\myparagraph{The gradient of GD is heterogeneous.}  Since $\beta^\star_{\rm sparse}$ is sparse by definition,  we deduce from \cref{eq:nabla0_GD} that $\nabla \cL(\beta_0)$ is heterogeneous with larger values corresponding to the support of $\beta^\star_{\rm sparse}$. Along with \cref{claim:shady}, this means that $\rm Gain _\gamma$ \textbf{has much larger values on the support of $\beta^\star_{\rm sparse}$}. The corresponding weighted $\ell_1$-norm therefore 
penalises the coordinates  belonging to the support of $\beta^\star_{\rm sparse}$, which hinders the recovery of $\beta^\star_{\rm sparse}$ (as explained in Example~\ref{app:example}, Appendix~\ref{sec:app:example}).

\myparagraph{The stochastic gradient of SGD is homogeneous.}  On the contrary, from \cref{eq:nabla0_SGD}, we have that the initial stochastic gradients are homogeneous, leading to a weighted $\ell_1$-norm where the weights are roughly balanced. 
The corresponding weighted $\ell_1$-norm is therefore close to the uniform $\ell_1$-norm and the classical $\ell_1$ recovery guarantees are expected.



\paragraph{Overall summary of the joint effects of the scale and shape.}
In summary we have the following trichotomy which fully explains \Cref{fig:main_theorem}:
\begin{enumerate}
    \item for small stepsizes, the scale is small,  and 
    (S)GD solutions are close to that of gradient flow;
    \item for large stepsizes the scale is significant and the recovered solutions differ from GF: 
    \begin{itemize}
        \item for SGD the shape of $\balpha_\infty$ is uniform, the associated norm is closer to the $\ell_1$-norm and the recovered solution is closer to the sparse solution;
        \item for GD, the shape is heterogeneous, the associated norm is weighted such that it hinders the recovery of the sparse solution. 
    \end{itemize}
\end{enumerate}


In this last section, we relate heuristically these findings to the \emph{Edge of Stability} phenomenon.


\section{Edge of Stability: the neural point of view}\label{sec:EoS}
%

In recent years it has been noticed that when training neural networks with `large' stepsizes at the limit of divergence, GD enters the \emph{Edge of Stability (EoS)} regime.
In this regime, as seen in \Cref{fig:EoS}, the iterates of GD ‘bounce' / 'oscillate'.
In this section, we come back to the point of view of the weights $w_k = (u_k, v_k) \in \R^{2d}$ and make the connection between our previous results and the common understanding of the \emph{EoS}  phenomenon.  The question we seek to answer is: in which case does GD enter the \textit{EoS} regime, and if so, what are the consequences on the trajectory? 
\emph{Keep in mind that this section aims to provide insights rather than formal statements.}
We study the GD trajectory starting from a small initialisation $\balpha = \alpha \mathbf{1}$ where $\alpha < \! \! <  1$ such that we can consider that gradient flow converges close to the sparse interpolator $\beta^\star_{\rm sparse} = \beta_{w^\star_{\rm sparse}}$ corresponding to the weights $w^\star_\sparse = (\sqrt{|\beta^\star_{\rm sparse}|}, \mathrm{sign}(\beta^\star_{\rm sparse}) \sqrt{|\beta^\star_{\rm sparse}|})$ (see Lemma 1 in \cite{pesme2023saddle} for the mapping from the predictors to weights for gradient flow).
%
%
The trajectory of GD as seen in \cref{fig:EoS} (left) can be decomposed into up to $3$ phases.

\myparagraph{First phase: gradient flow.}
The stepsize is appropriate for the local curvature (as seen in \Cref{fig:EoS}, lower right) around initialisation and the iterates of GD remain close to the trajectory of gradient flow (in black in \cref{fig:EoS}). If the stepsize is such that $\gamma < \frac{2}{\lambda_{\rm max} (\nabla^2 F(w^\star_{\rm sparse}))}$, then it is compatible with the local curvature and the iterates can converge: in this case GF and GD converge to the same point (as seen in \cref{fig:main_theorem} for small stepsizes). For larger  $\gamma >\frac{2}{\lambda_{\rm max} (\nabla^2 F(w^\star_{\rm sparse}))}$ (as is the case for $\gamma_{\mathrm{GD}}$ in \cref{fig:EoS}, lower right), the iterates cannot converge to $\beta^\star_{\rm sparse}$ and we enter the oscillating phase.

\myparagraph{Second phase: oscillations.}
The iterates start oscillating. The gradient of $F$ writes 
$\nabla_{(u, v)} F(w) \sim (\nabla \cL(\beta) \odot v, \nabla \cL(\beta) \odot u)$ and for $w$ in the vicinity of $w^\star_{\rm sparse}$ we have that $u_i \approx v_i \approx 0$ for $i \notin \rm supp (\beta^\star_{\rm sparse})$. Therefore for $w \sim w^\star_{\rm sparse}$ we have that $\nabla_u F(w)_{i} \approx \nabla_v F(w)_{i} \approx 0$ for $i \notin \rm supp (\beta^\star_{\rm sparse})$ and the gradients roughly belong to $\Span(e_i, e_{i+d})_{i \in {\rm supp} (\beta^\star_{\rm sparse})}$. This means that only the coordinates of the weights $(u_i, v_i)$ for $i \in \rm supp (\beta^\star_{\rm sparse})$ can oscillate  and similarly for $(\beta_i)_{i \in \rm supp (\beta^\star_{\rm sparse})}$ (as seen \Cref{fig:EoS} left).


\myparagraph{Last phase: convergence.} 
Due to the oscillations, the iterates gradually drift towards a region of lower curvature (\cref{fig:EoS}, lower right, the sharpness decreases) where they may (potentially) converge. 
\cref{thm:implicit_bias} enables us to understand where they converge: the coordinates of $\beta_k$ that have  oscillated significantly along the trajectory belong to the support of $\beta^\star_{\rm sparse}$, and therefore  $\rm Gain_\gamma(i)$ becomes much larger for $i \in \rm supp (\beta^\star_{\rm sparse})$ than for the other coordinates. 
Thus, the coordinates of the solution recovered in the \textit{EoS} regime are heavily penalised on the support of the sparse solution. This is observed in \Cref{fig:EoS} (left): the oscillations of $(\beta_i)_{i \in \rm supp (\beta^\star_{\rm sparse})}$ lead to a gradual shift of these coordinates towards $0$, hindering an accurate recovery of the solution $\beta^\star_{\rm sparse}$.

\begin{figure}[t]
\hspace*{-15pt}
\centering
\includegraphics[trim={0 0 0 0}, width=0.9\linewidth]{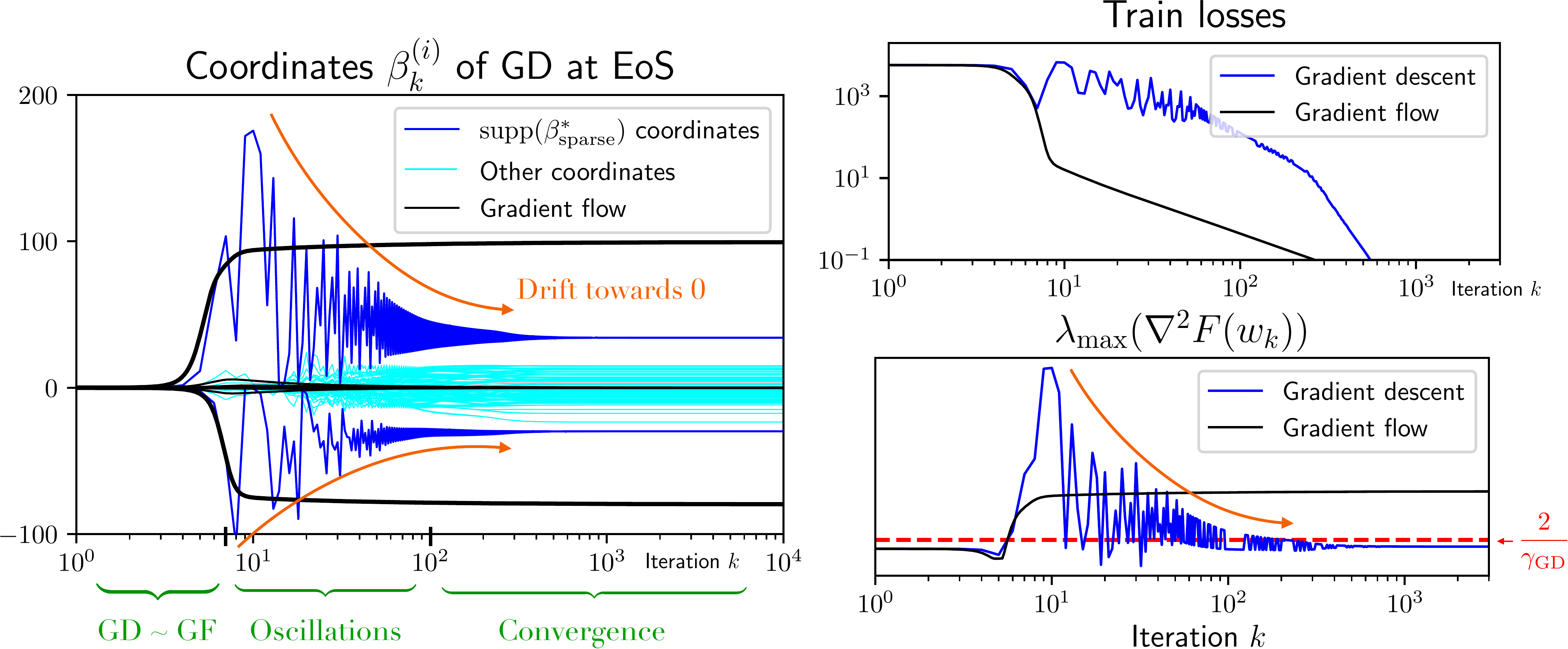}
   \caption{GD at the \textit{EoS}. \textit{Left:} For GD, the coordinates on the support of $\beta^\star_{\rm sparse}$ oscillate and drift towards $0$. \textit{Right, top:} The GD train losses saturate before eventually converging. \textit{Bottom:} GF converges towards a solution that has a high hessian maximum eigenvalue. GD cannot converge towards this solution because of its large stepsize: it therefore drifts towards a solution that has a curvature just below $2 / \gamma$. \label{fig:EoS}}.
     \vspace*{-1em}
\end{figure}
%

\myparagraph{SGD in the \emph{EoS} regime.} 
In contrast to the behavior of GD where the oscillations primarily occur on the non-sparse coordinates of ground truth sparse model, for SGD we see a different behavior in \Cref{fig:eos_sgd} (\Cref{app:sec:add_exp}). For stepsizes in the \textit{EoS} regime, just below the non-convergence threshold: the fluctuation of the coordinates occurs evenly over all coordinates, leading to a uniform $\balpha_\infty$. These fluctuations are reminiscent of label-noise SGD~\citep{andriushenko2022sgd}, that have been shown to recover the sparse interpolator in diagonal
linear networks \cite{pillaudvivien2022labelnoise}.
\section*{Conclusion}

We study the effect of stochasticity along with large stepsizes when training DLNs with (S)GD. We prove convergence of the iterates as well as explicitly characterise the recovered solution by exhibiting an implicit regularisation problem which depends on the iterates' trajectory. We show that large stepsizes highly changes the solution recovered by (S)GD with respect to gradient flow. However, surprisingly, the generalisation properties of GD with large stepsizes are highly different to those of SGD: without stochasticity, the use of large stepsizes can prevent the recovery of the sparse interpolator. We also provide insights on the link between the \textit{Edge of Stability } regime and our results.

\bibliographystyle{plainnat} 
\bibliography{refs}

\appendix
\onecolumn

\paragraph{Organisation of the Appendix.}
\begin{enumerate}
    \item In \Cref{app:sec:add_exp}, we provide additional experiments for uncentered data as well as on the behaviour of the sharpness and trace of the Hessian along the trajectory of the iterates. We finally provide an experiment highlighting  the EoS regime for SGD.
    \item In \Cref{sec:app:ingredients}, we prove that $(\beta_k)$ follows a Mirror descent recursion with varying potentials. We explicit these potentials and discuss some consequences.
    \item In \Cref{sec:app:param} we prove that (S)GD on the $\frac{1}{2}(w_+^2-w_-^2)$ and $u\odot v$ parametrisations with suitable initialisations lead to the same sequence $(\beta_k)$.
    \item In \Cref{sec:app:example}, we show that the hypentropy $\psi_\balpha$ converges to a \textbf{weighted}-$\ell_1$-norm when $\balpha$ converges to  $0$ non-uniformly. We then discuss the effects of this \textbf{weighted} $\ell_1$-norm for sparse recovery.
    \item In \Cref{app:sec:descentlemmas}, we provide our descent lemmas for mirror descent with varying potentials and prove the boundedness of the iterates.
    \item In \Cref{app:sec:proof_main}, we prove our main results: \cref{thm:implicit_bias} and \cref{thm:conv_iterates}, as well as quantitative convergence (\Cref{prop:conv_quantitative}).
    \item In \Cref{sec:app:misc}, we prove the lemmas and propositions given in the main text.
    \item In \Cref{sec:app:technical}, we provide technical lemmas used throughout the proof of \cref{thm:implicit_bias} and \cref{thm:conv_iterates}.
    \item In \Cref{app:concentration}, we provide concentration results for random matrices and random vectors, used to estimate with high probability (w.r.t. the dataset) quantities related to the data.
\end{enumerate}

\newpage

\section{Additional experiments and results}\label{app:sec:add_exp}

\subsection{Uncentered data \label{app:uncentered}}

When the data is uncentered, the discussion and the conclusion for GD are somewhat different. This paragraph is motivated by the observation of \citet{nacson2022implicitbiasstepsize} who notice that GD with large stepsizes helps to recover low $\ell_1$ solutions for uncentered data (\cref{fig:uncentered}). We make the following assumptions on the uncentered inputs.
\begin{assumption}\label{ass:RIP2}
There exist $\mu \in \R^d$ and $\delta, c_0, c_1, c_2 > 0 $ such that for all $s$-sparse vectors $\beta$ verifying $\langle\mu,\beta\rangle\geq  c_0\NRM{\beta}_\infty\NRM{\mu}_\infty$, there exists $\varepsilon \in \R^d$ such that $(X^\top X) \beta = \langle \beta, \mu \rangle \mu + \varepsilon$  where $\Vert \varepsilon \Vert_2 \leq \delta \Vert \beta \Vert_2$ and 
$ c_1 \langle \beta, \mu \rangle^2 \mu^2 \leq \frac{1}{n} \sum_i  x_i^2 \langle x_i, \beta \rangle^2 \leq c_2  \langle \beta, \mu \rangle^2 \mu^2$.
\end{assumption}
\cref{ass:RIP2} is not restrictive and holds with high probability for  $\mathcal{N}(\mu \one, \sigma^2 I _d)$ inputs when  $\mu > \!\!> \sigma \mathbf{1}$ (see  \cref{lemma:RIP_uncentered} in Appendix). The following lemma characterises the initial shape of SGD and GD gradients for uncentered data.
\begin{proposition}[Shape of the (stochastic) gradient at initialisation]\label{lemma:shape_SG_init_uncentered}
Under \cref{ass:RIP2} and if $\langle\mu,\beta_\sparse^\star\rangle\geq c_0\NRM{\beta}_\infty\NRM{\mu}_\infty$, the squared full batch gradient and the expected stochastic gradient descent at initialisation satisfy, for some $\eps$ satisfying $\NRM{\eps}_\infty<\!\!<\NRM{\beta_\sparse}_2$:
\begin{align}
    \nabla \cL & (\beta_0) = \langle \beta^\star_{\rm sparse},\mu\rangle^2 \mu^2 + \varepsilon \,,\label{eq:nabla0_GD_un}\\
\mathbb{E}_{i \sim \rm{Unif}([n])} & [\nabla \cL_{i}(\beta_0)^2] 
= \Theta \Big( \langle \beta^\star_{\rm sparse},\mu\rangle^2 \mu ^2  \Big)\,.\label{eq:nabla0_SGD_un}
\end{align}
\end{proposition}
In this case the initial gradients of SGD and of GD \textbf{are both homogeneous}, explaining the behaviours of gradient descent in \cref{fig:uncentered} (App.~\ref{app:sec:add_exp}): large stepsizes help in the recovery of the sparse solution in the presence of uncentered data, as opposed to centered data.
Note that for decentered data with a $\mu\in\R^d$ orthogonal to $\beta^\star_\sparse$, there is no effect of decentering on the recovered solution.
If the support of $\mu$ is the same as that of $\beta_\sparse^\star$, the effect is detrimental and the same discussion as in the centered data case applies.

\Cref{fig:uncentered}: for uncentered data the solutions of GD and SGD have similar behaviours, corroborating \Cref{lemma:shape_SG_init_uncentered}.

\begin{figure}[h]
\centering
\vspace*{-12.5pt}
\begin{minipage}[c]{.6\linewidth}
\hspace*{-25pt}
\vspace*{-10.5pt}
\includegraphics[trim={0cm 0 0 0cm}, clip, width=\linewidth]{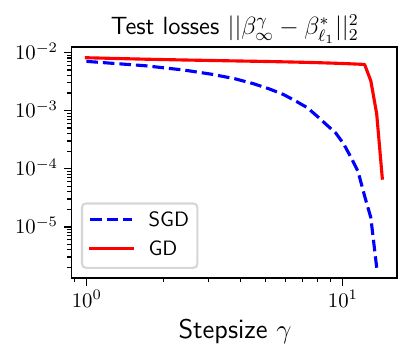}
  \end{minipage}
    \hspace*{-25pt}
  \caption{Noiseless sparse regression with a $2$-layer DLN with uncentered data $x_i\sim\cN(\mu \mathbf{1}, I_d)$ where $\mu = 5$. 
  All the stepsizes lead to convergence to a global solution and the solutions of SGD and GD have similar behaviours, corroborating \cref{lemma:shape_SG_init_uncentered}. The setup corresponds to $(n, d, s, \alpha) = (20, 30, 3, 0.1)$. \label{fig:uncentered}} 
\end{figure}

\subsection{Behaviour of the maximal value and trace of the hessian}
\label{app:hessian}

Here in \cref{fig:flatness}, we provide some additional experiments on the behaviour of: (1) the maximum eigenvalue of the hessian $\nabla^2 F(w_\infty^\gamma)$ at the convergence of the iterates of SGD and GD (2) the trace of hessian at the convergence of the iterates. As is clearly observed, increasing the stepsize for GD leads to a `flatter' minimum in terms of the maximum eigenvalue of the hessian, while increasing the stepsize for SGD leads to a `flatter' minimum in terms of its trace. These two solutions have very different structures. Indeed from the value of the hessian \cref{eq:hessian_F} at a global solution, and (very) roughly assuming that `$X^\top X = I_d$' and that `$\balpha \sim 0 $' (pushing the EoS phenomenon), one can see that minimising $\lambda_{\rm max}(\nabla^2 F(w))$ under the constraints $X(w_+^2 - w_-^2) = y$ and $w_+ \odot w_- = 0$ is equivalent to minimising $\Vert \beta \Vert_\infty$ under the constaint $X \beta = y$. On the other hand minimising the trace of the hessian is equivalent to minimising the $\ell_1$-norm.

\begin{figure}[ht]
\centering
\begin{minipage}[c]{.5\linewidth}
\hspace*{-15pt}
\includegraphics[trim={0 0 0 0}, width=\linewidth]{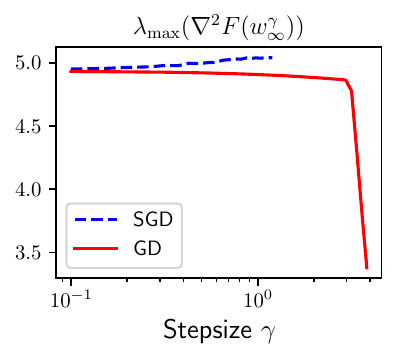}
   \end{minipage}
   \hspace*{-15pt}
   \begin{minipage}[c]{.5\linewidth}
    \vspace*{-6.5pt}
\includegraphics[width=0.99\linewidth]{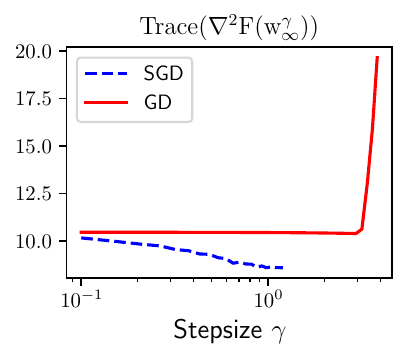}
   \end{minipage}
    \hspace*{-15pt}
   \caption{ Noiseless sparse regression setting. Diagonal linear network. Centered data. Behaviour of $2$ different types of flatness of the recovered solution by SGD and GD depending on the stepsize. The setup corresponds to $(n, d, s, \alpha) = (20, 30, 3, 0.1)$.\label{fig:flatness}}
\end{figure}

\subsection{Edge of Stability for SGD}

\begin{figure}[ht]
\centering
\includegraphics[trim={0 0 0 0}, width=0.5\linewidth]{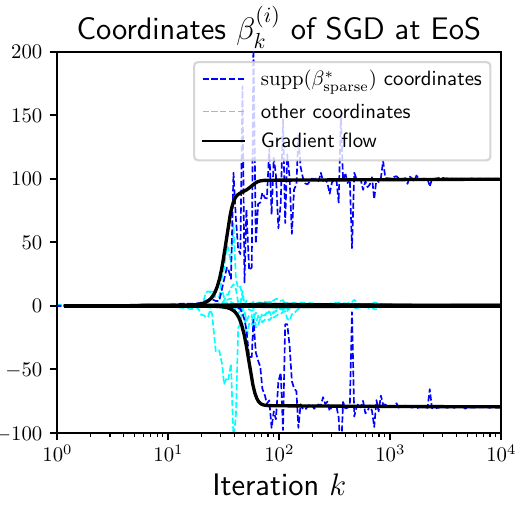}
   \caption{SGD at the edge of stability: all coordinates fluctuate, and the sparse solution is recovered.
   As opposed to GD at the EoS, since all coordinates fluctuate, the coordinates to recover are not more penalised than the others.
   \label{fig:eos_sgd}}
\end{figure}

\section{Main ingredients behind the proof of \cref{thm:implicit_bias} and \cref{thm:conv_iterates}}\label{sec:app:ingredients}

In this section, we show that the iterates $(\beta_k)_{k\geq0}$ follow a \emph{stochastic mirror descent} with \emph{varying potentials}. At the core of our analysis, this result enables us to \emph{(i)} prove convergence of the iterates to an interpolator and \emph{(ii)} completely characterise the inductive bias of the algorithm (SGD or GD).
Unveiling a mirror-descent like structure to characterise the implicit bias of a gradient method is classical. For gradient flow over diagonal linear networks \citep{pmlr-v125-woodworth20a}, the iterates follow a mirror flow with respect to the hypentropy \eqref{eq:implicit_opt_hypentropy} with parameter $\alpha$ the initialisation scale, while for stochastic gradient flow \citep{pesme2021implicit} the mirror flow has a continuously evolving potential. 

\subsection{Mirror descent and varying potentials}

We recall that for a strictly convex reference function $h:\R^d\to\R$, the (stochastic) mirror descent iterates algorithm write as \citep{bauschke2017descent,pmlr-v139-dragomir21a}, where the minimum is assumed to be attained over $\R^d$ and unique:
\begin{equation}\label{eq:md}
    \beta_{k+1}=\argmin_{\beta\in\R^d}\set{\eta_k\langle g_k,\beta\rangle + D_h(\beta,\beta_k)}\,,
\end{equation}
for stochastic gradients $g_k$, stepsize $\gamma_k\geq0$, and $D_h(\beta,\beta')=h(\beta)-h(\beta')-\langle\nabla h(\beta'),\beta-\beta'\rangle$ is the Bregman divergence associated to $h$. Iteration \eqref{eq:md} can also be cast as
\begin{equation}\label{eq:md2}
    \nabla h(\beta_{k+1})=\nabla h(\beta_k)-\gamma_k g_k\,.
\end{equation}
Now, let $(h_k)$ be strictly convex reference functions $\R^d\to \R$. Whilst in continuous time, there is only one natural way to extend mirror flow to varying potentials, in discrete time the varying potentials can be incorporated in \eqref{eq:md} (replacing $h$ by $h_k$ and leading to $\nabla h_k(\beta_{k+1})=\nabla h_k(\beta_k)-\gamma_k g_k$), the mirror descent with varying potentials we study in this paper incorporates $h_{k+1}$ and $h_k$ in \eqref{eq:md2}. 
The iterates are thus defined as through:
\begin{equation*}
    \beta_{k+1}=\argmin_{\beta\in\R^d}\set{\eta_k\langle g_k,\beta\rangle + D_{h_{k+1},h_k}(\beta,\beta_k)}\,,
\end{equation*}
where $D_{h_{k+1},h_k}(\beta,\beta')=h_{k+1}(\beta)-h_k(\beta')-\langle\nabla h_k(\beta'),\beta-\beta'\rangle$,
a recursion that can also be cast as:
\begin{equation*}
       \nabla h_{k+1}(\beta_{k+1}) = \nabla h_{k}(\beta_{k}) - \gamma_k g_k\,.
\end{equation*}
To derive convergence of the iterates, we prove analogs to classical mirror descent lemmas, generalised to time-varying potentials.

\subsection{The iterates $(\beta_k)$ follow a stochastic mirror descent with varying potential recursion}\label{app:sec:tvMD}


In this section we show and prove that the iterates $(\beta_k)_k$ follow a stochastic mirror descent with varying potentials. Before stating the proposition, we recall the definition of the potentials. To do so we introduce several quantities. 


Let $q,q_\pm:\R\to\R\cup\set{\infty}$ be defined as:
\begin{align*}
& q_\pm(x)= \mp 2 x - \ln \big ( (1 \mp x)^2 )\, ,\\
& q(x)=  \frac{1}{2} (q_+(x)+ q_-(x)) = - \frac{1}{2} \ln\big((1-x^2)^2\big)\,,
\end{align*}
with  the convention that $q(1) = \infty$.
Notice that $q(x) \geq 0$ for $|x| \leq \sqrt{2}$ and $q(x)<0$ otherwise.
For the iterates $\beta_k = u_k \odot v_k \in \R^d$, we recall the definition of the following quantities:
\begin{align*}
&\balpha_{\pm, k} = \balpha  \exp( - \frac{1}{2} \sum_{i = 0}^{k-1} q_\pm( \gamma_\ell \nabla \cL_{\cB_\ell}(\beta_\ell)  )  ) \in \R_{>0}^d \,,\\
&\balpha^2_k = \balpha_{+, k} \odot \balpha_{-, k}
\,,\\
&\phi_k = \frac{1}{2} \argsinh \big (  \frac{\balpha_{+, k}^2 - \balpha_{-, k}^2 }{2 \balpha_{k}^2 } \big ) \in \R^d\,.
\end{align*}
Finally for $k \geq 0$, we define the potentials  $(h_k : \R^d \to \R)_{k\geq0}$ as:
\begin{equation}
\label{eq:def_hk}
    h_k(\beta) = \psi_{\balpha_k}(\beta) - \langle \phi_k, \beta \rangle\,,
\end{equation}
where $\psi_{\balpha_k}$ is the hyperbolic entropy defined in \eqref{eq:implicit_opt_hypentropy} of scale $\balpha_k$:
\[\psi_{\balpha_k}(\beta) = \frac{1}{2} \sum_{i=1}^d \Big ( \beta_i \mathrm{arcsinh}( \frac{\beta_i}{\alpha_{k, i}^2} ) \! -\! \sqrt{\beta_i^2 + \alpha_{k, i}^4} +  \alpha_{k, i}^2 \Big )\]
where $\alpha_{k, i}$ corresponds to the $i^{th}$ coordinate of the vector $\balpha_k$.

Now that all the relevant quantities are define, we can state the following proposition which explicits the time-varying stochastic mirror descent followed by $(\beta_k)_k$

\begin{restatable}{proposition}{tvSMD}\label{prop:tv_md}
The iterates $(\beta_k=u_k\odot v_k)_{k\geq0}$ from \Cref{eq:SGD_recursion} satisfy the Stochastic Mirror Descent recursion with varying potentials $(h_k)_k$:
\begin{align}
\label{eq:time_varying_MD}
   \nabla h_{k+1}(\beta_{k+1}) = \nabla h_{k}(\beta_{k}) - \gamma_k \nabla \cL_{\cB_k}(\beta_k)\,,
\end{align}
where $h_k:\R^d\to\R$ for $k\geq 0$ are defined \Cref{eq:def_hk}. Since $\nabla h_0(\beta_0) = 0$ we have:
\[ \nabla h_k(\beta_k) \in \mathrm{span}(x_1, \dots, x_n)\] 
\end{restatable}




\begin{proof}
Using Proposition~\ref{prop:equivalence_param}, we study the $\half(w_+^2-w_-^2)$ parametrisation instead of the $u \odot v$, indeed this is the natural parametrisation to consider when doing the calculations as it ``separates" the recursions on $w_+$ and $w_-$.

Let us focus on the recursion of $w_+$:
\begin{align*}
    w_{+, k+1} &=  (1 - \gamma_k \nabla \cL_{\cB_k}(\beta_k)) \cdot w_{+, k}\,.
\end{align*}
We have:
\begin{align*}
    w_{+, k+1}^2 &=  (1 - \gamma_k \nabla \cL_{\cB_k}(\beta_k))^2 \cdot w_{+, k}^2 \\
    &= \exp{ (\ln (( 1 - \gamma_k \nabla \cL_{\cB_k}(\beta_k))^2 ) )} \cdot  w_{+, k}^2\,,
\end{align*}
with the convention that $\exp( \ln(0)) = 0$. This leads to:
\begin{align*}
    w_{+, k+1}^2 
    &= \exp \big ( - 2\gamma_k \nabla \cL_{\cB_k}(w_k) +  2\gamma_k \nabla \cL_{\cB_k}(\beta_k) + \ln ( (1 - \gamma_k \nabla \cL_{\cB_k}(\beta_k))^2 )  \big) \cdot  w_{+, k}^2 \\
    &= \exp \big ( - 2\gamma_k \nabla \cL_{\cB_k}(\beta_k) - q_+(\gamma_k \nabla \cL_{\cB_k}(\beta_k) ) \big ) \cdot  w_{+, k}^2 \,,
\end{align*}
since $q_+(x)= - 2x - \ln((1-x)^2)$.
Expanding the recursion and using that $ w_{+, k=0}$ is initialised at $ w_{+, k=0}=\balpha$, we thus obtain:
\begin{align*}
    w_{+, k}^2  &= \balpha^2  \exp( - \sum_{\ell = 0}^{k-1} q_+( \gamma_\ell \nabla \cL_{\cB_\ell}(\beta_\ell)  )  )  \exp{( - 2 \sum_{\ell = 0}^{k-1}  \gamma_\ell \nabla \cL_{\cB_\ell}(\beta_\ell) ) } \\
    &= \balpha_{+, k}^2  \exp{( - 2 \sum_{\ell = 0}^{k-1}  \gamma_\ell \nabla \cL_{\cB_\ell}(\beta_\ell) ) } \,,
\end{align*}
where we recall that $\balpha_{\pm, k}^2 = \balpha^2  \exp( - \sum_{\ell = 0}^{k-1} q_\pm( \gamma_\ell g_\ell  )  )$.
One can easily check that we similarly get:
\begin{align*}
    w_{-, k}^2  &= \balpha_{-, k}^2  \exp{( + 2 \sum_{\ell = 0}^{k-1}  \gamma_\ell \nabla \cL_{\cB_\ell}(\beta_\ell) ) } \,,
\end{align*}
leading to:
\begin{align*}
    \beta_k  &= \frac{1}{2} (w_{+, k}^2 - w_{-, k}^2) \\
    &= \frac{1}{2}  \balpha_{+, k}^2  \exp{( - 2 \sum_{\ell = 0}^{k-1}  \gamma_\ell \nabla \cL_{\cB_\ell}(\beta_\ell) ) } - \frac{1}{2}  \balpha_{-, k}^2  \exp{( + 2 \sum_{\ell = 0}^{k-1}  \gamma_\ell \nabla \cL_{\cB_\ell}(\beta_\ell) ) } \,.
\end{align*}
Using Lemma~\ref{lem:argsh}, the previous equation can be simplified into: 
\begin{align*}
    \beta_k   &= \balpha_{+, k}  \balpha_{-, k}  \sinh{ \Big ( - 2 \sum_{\ell = 0}^{k-1}  \gamma_\ell \nabla \cL_{\cB_\ell}(\beta_\ell) + \argsinh \big (  \frac{\balpha_{+, k}^2 - \balpha_{-, k}^2 }{2 \balpha_{+, k}  \balpha_{-, k}  } \big ) } \Big ) \,,
\end{align*}
which writes as:
\begin{align*}
    \frac{1}{2} \argsinh \big ( \frac{\beta_k}{ \balpha_k^2} \big) - \phi_k =  - \sum_{\ell = 0}^{k-1}  \gamma_\ell \nabla \cL_{\cB_\ell}(\beta_\ell) \in \mathrm{span}(x_1, \dots, x_n)\,,
\end{align*}
where $\phi_k = \frac{1}{2} \argsinh \big (  \frac{\balpha_{+, k}^2 - \balpha_{-, k}^2 }{2 \balpha_{k}^2 } \big )$, $\balpha^2_k = \balpha_{+, k} \odot \balpha_{-, k}$ and
since the potentials $h_k$ are defined in \cref{eq:def_hk} as $h_k=\psi_{\balpha_k}  - \langle \phi_k,\cdot\rangle$ with
\begin{equation}
    \psi_\balpha(\beta) = \frac{1}{2} \sum_{i=1}^d \Big ( \beta_i \mathrm{arcsinh}( \frac{\beta_i}{\balpha_i^2} ) \ - \sqrt{\beta_i^2 + \balpha_i^4} +  \balpha_i^2 \Big )
\end{equation}
specifically such that $\nabla h_k(\beta_k) = \frac{1}{2} \argsinh \big ( \frac{\beta_k}{ \balpha_k^2} \big) - \phi_k$. 
Hence,
\begin{equation*}
   \nabla h_k(\beta_k)=\sum_{\ell<k} \gamma_\ell \nabla \cL_{\cB_\ell}(\beta_\ell)  \,,
\end{equation*}
so that:
\begin{align*}
    \nabla h_{k+1}(\beta_{k+1}) =  \nabla h_k(\beta_k)  - \gamma_k \nabla \cL_{\cB_k}(\beta_k)\,,
\end{align*}
which corresponds to a Mirror Descent with varying potentials $(h_k)_k$.
\end{proof}


\section{Equivalence of the $u \odot v$ and $\frac{1}{2} (w_+^2-w_-^2)$ parametrisations}\label{sec:app:param}

We here prove the equivalence between the $\frac{1}{2} (w_+^2-w_-^2)$ and $u\odot v$ parametrisations, \textbf{that we use throughout the proofs in the Appendix.}

\begin{restatable}{proposition}{propequivparam}
\label{prop:equivalence_param}
Let $(\beta_k)_{k\geq 0}$ and $(\beta'_k)_{k\geq0}$ be respectively generated by stochastic gradient descent on the $u\odot v$ and $\half(w_+^2-w_-^2)$ parametrisations:
\begin{equation*}
        (u_{k+1},v_{k+1})=(u_k,v_k)-\gamma_k \nabla_{u,v}\big( \cL_{\cB_k}(u\odot v)\big)(u_k,v_k)\,,
\end{equation*}
and 
\begin{equation*}
        w_{\pm,k+1}=w_{\pm,k}-\gamma_k \nabla_{w_\pm}\big( \cL_{\cB_k}(\half(w_+^2-w_-^2))\big)(w_{+,k},w_{-,k})\,,
\end{equation*}
initialised as $u_0=\sqrt{2} \balpha,v_0=0$ and $w_{+,0}=w_{-,0}=\balpha$. Then for all $k\geq 0$, we have $\beta_k=\beta'_k$.
\end{restatable}


\begin{proof}
We have:
\begin{equation*}
      w_{\pm,0}=\balpha\,,\quad w_{\pm, {k+1}} =(1 \mp \gamma_k \nabla \cL_{\cB_k}(\beta'_k)) w_{\pm, k}\,,
\end{equation*}
and 
\begin{equation*}
   u_0=\sqrt{2} \balpha\,,\quad v_0=0\,,\quad u_{k+1}= u_k - \gamma_k \nabla \cL_{\cB_k}(\beta_k) v_k\,,\quad v_{k+1}= v_k - \gamma_k \nabla \cL(\beta_k) u_k\,.
\end{equation*}
Hence,
\begin{equation*}
    \beta_{k+1}= (1+\gamma_k^2\nabla\cL(\beta_k)^2)\beta_k - \gamma_k (u_k^2+v_k^2)\nabla \cL_{\cB_k}(\beta_k)\,,
\end{equation*}
and
\begin{equation*}
    \beta'_{k+1}= (1+\gamma_k^2\nabla\cL_{\cB_k}(\beta'_k)^2)\beta'_k - \gamma_k (w_{+,k}^2+w_{-,k}^2)\nabla \cL_{\cB_k}(\beta'_k)\,.
\end{equation*}
Then, let $z_k=\frac{1}{2}(u_k^2-v_k^2)$ and $z'_k=w_{+,k}w_{-k}$.
We have $z_0=\balpha^2$, $z'_0=\balpha^2$ and:
\begin{equation*}
    z_{k+1}=(1-\gamma_k^2\nabla\cL_{\cB_k}(\beta_k)^2)z_k\,,\quad z'_{k+1}=(1-\gamma_k^2\nabla\cL_{\cB_k}(\beta'_k)^2)z'_k\,.
\end{equation*}
Using $a^2+b^2=\sqrt{(2ab)^2+(a^2-b^2)^2}$ for $a,b\in\R$, we finally obtain that:
\begin{equation*}
    u_k^2+v_k^2 = \sqrt{(2\beta_k)^2+(2z_k)^2}\,,\quad w_{+,k}^2+w_{-,k}^2 = \sqrt{(2\beta'_k)^2+(2z'_k)^2}\,.
\end{equation*}
We conclude by observing that $(\beta_k,z_k)$ and $(\beta'_k,z'_k)$ follow the exact same recursions, initialised at the same value $(0,\balpha^2)$\,.

\end{proof}

\section{Convergence of $\psi_\alpha$ to a weighted $\ell_1$ norm and harmful behaviour}\label{sec:app:example}

We show that when taking the scale of the initialisation to $0$, one must be careful in the characterisation of the limiting norm, indeed if each entry does not go to zero "at the same speed", then the limit norm is a \textbf{weighted} $\ell_1$-norm rather than the classical $\ell_1$ norm. 
\begin{proposition}
    \label{app:prop:weightedl1norm}
    For $\alpha \geq 0$ and a vector $h \in \R^d$, let $\tilde{\alpha} = \alpha \exp(- h \ln(1 / \alpha)) \in \R^d$. Then we have that for all $\beta \in \R^d$
    \begin{align*}
        \psi_{\tilde{\alpha}}(\beta) \underset{\alpha \to 0 }{\sim} \ln( \frac{1}{\alpha}) \cdot \sum_{i=1}^d (1 + h_i) \vert \beta_i \vert.
    \end{align*}
\end{proposition}

\begin{proof}
    Recall that 
    \begin{align*}
        \psi_{\tilde{\alpha}}(\beta) &=  \frac{1}{2} \sum_{i=1}^d \Big ( \beta_i \mathrm{arcsinh}( \frac{\beta_i}{\tilde{\alpha}_i^2} ) \ - \sqrt{\beta_i^2 + \tilde{\alpha_i}^4} +  \tilde{\alpha}_i^2 \Big ) 
    \end{align*}
Using that $\mathrm{arcsinh}(x) \underset{|x| \to \infty}{\sim} \mathrm{sgn}(x) \ln(|x|)$, and that $\ln(\frac{1}{\tilde{\alpha}_i^2}) = (1 + h_i) \ln(\frac{1}{\alpha^2})$ we obtain that 
\begin{align*}
        \psi_{\tilde{\alpha}}(\beta) &\underset{\alpha \to 0}{\sim} \frac{1}{2} \sum_{i=1}^d \mathrm{sgn}(\beta_i) \beta_i (1 + h_i) \ln( \frac{1}{\alpha^2}) \\
        &= \frac{1}{2} \ln( \frac{1}{\alpha^2}) \sum_{i=1}^d (1 + h_i)  \vert \beta_i \vert.
    \end{align*}
\end{proof}

The following \Cref{fig:weighted_l1_norm} illustrates the effect of the non-uniform shape $\balpha$ on the corresponding potential $\psi_\balpha$.

 \begin{figure}[h!]
\centering
\begin{minipage}[c]{.5\linewidth}
\hspace*{-15pt}
\includegraphics[trim={1 0 0 0}, width=\linewidth]{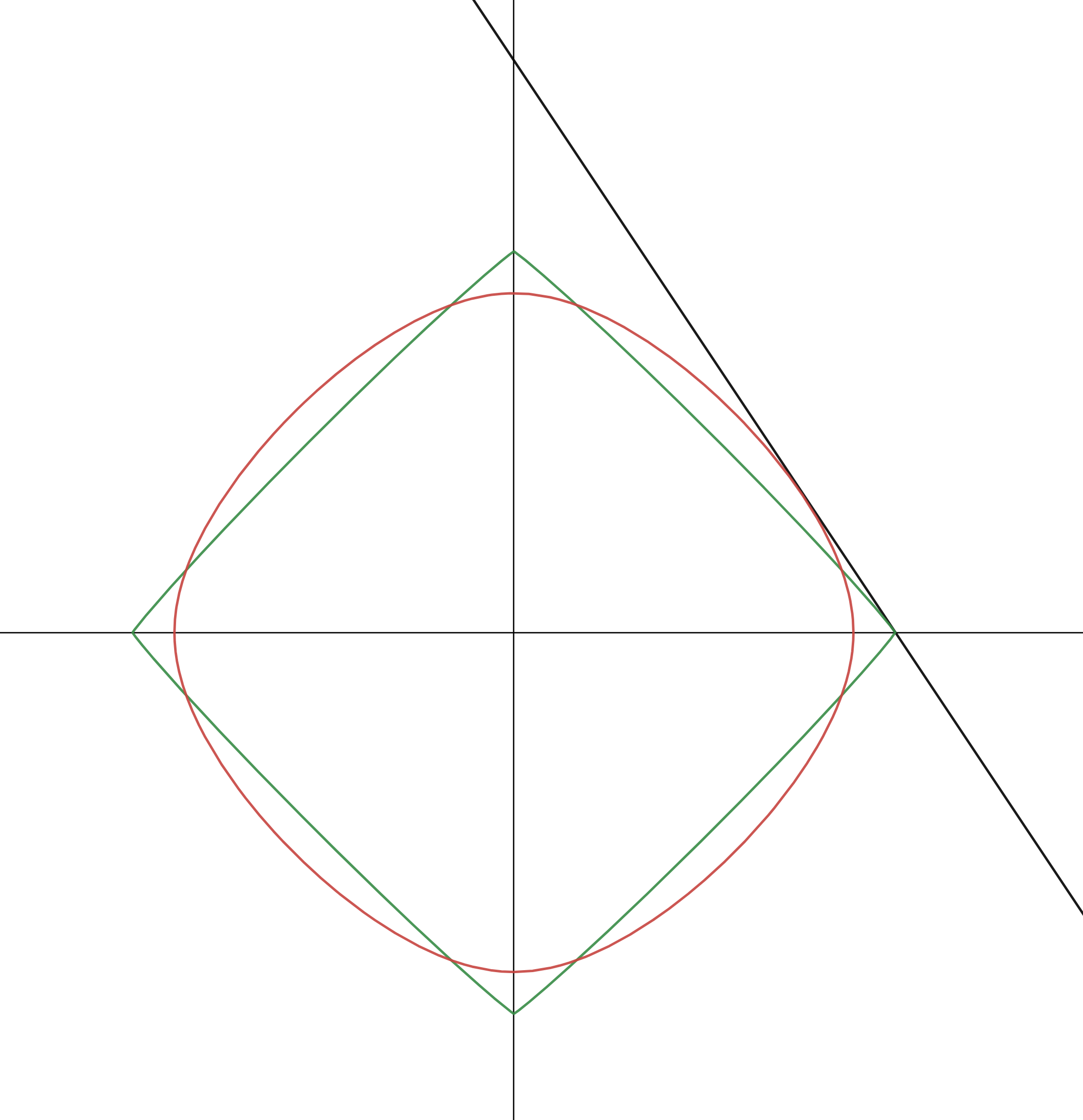}
    \end{minipage}
    \hspace*{-15pt}
    \begin{minipage}[c]{.5\linewidth}
    \vspace*{-6.5pt}
\includegraphics[width=0.913\linewidth]{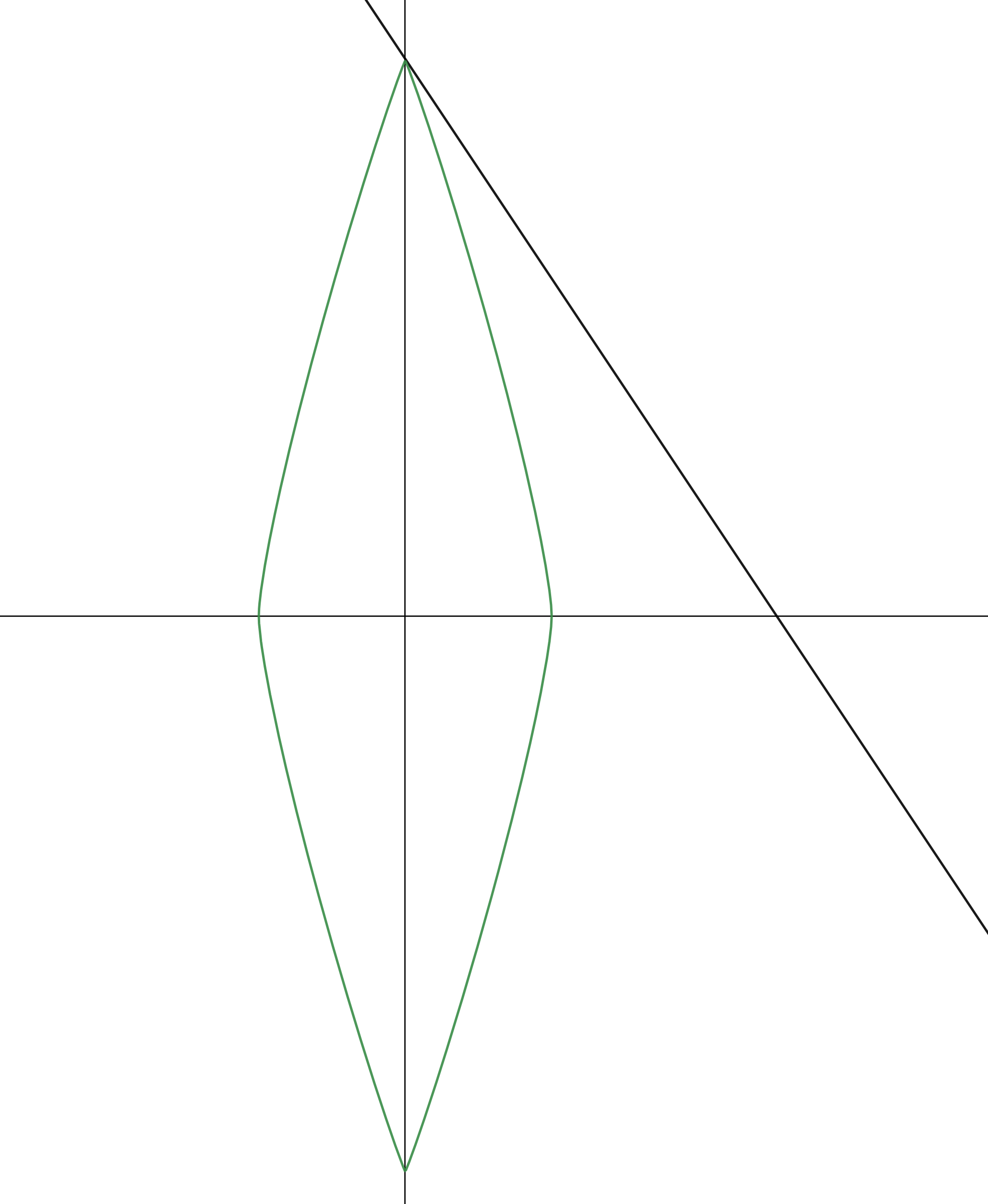}
    \end{minipage}
    \hspace*{-15pt}
    \caption{ \textit{Left}: Uniform $\balpha = \alpha \mathbf{1}$: a smaller scale $\alpha$ leads to the potential $\psi_\alpha$ being closer to the $\ell_1$-norm. \textit{Right}: A non uniform $\balpha$ can lead to the recovery of a solution which is very far from the minimum $\ell_1$-norm solution.  The affine line corresponds to the set of interpolators when $n=1$, $d=2$ and $s = 1$. \label{fig:weighted_l1_norm}} 
\end{figure}
 
More generally, for $\alpha$ such that $\alpha_i\to 0$ for all $i\in[d]$ at rates such that $\ln(1/\alpha_i)\sim q_i \ln(1/\max_i\alpha_i)$, we retrieve a weighted $\ell_1$ norm:
\begin{equation*}
    \frac{\psi_\alpha(\beta)}{\ln(1/\alpha^2)}\to \sum_{i=1}^d q_i |\beta_i|\,.
\end{equation*}
Hence, even for arbitrary small $\max_i\alpha_i$, if the \emph{shape} of $\alpha$ is ‘bad', the interpolator $\beta_\alpha$ that minimizes $\psi_\alpha$ can be arbitrary far away from $\beta_{\ell^1}^\star$ the interpolator of minimal $\ell_1$ norm.


We illustrate the importance of the previous proposition in the following example.


\begin{example}
\label{app:example}
We illustrate how, even for arbitrary small $\max_i \alpha_i$, the interpolator $\beta^\star_\alpha$ that minimizes $\psi_\alpha$ can be far from the minimum $\ell_1$ norm solution, due to the shape of $\balpha$ that is not uniform. The message of this example is that for $\balpha\to 0$ non-uniformly across coordinates, if the coordinates of $\alpha$ that go slowly to $0$ coincide with the non-null coordinates of the sparse interpolator we want to retrieve, then $\beta^\star_\alpha$ will be far from the sparse solution.

A simple counterexample can be built: let $\beta^\star_{\rm sparse}=(1,\ldots,1,0,\ldots,0)$ (with only the $s=o(d)$ first coordinates that are non-null), and let $(x_i)$, $(y_i)$ be generated as $y_i=\langle \beta^\star_{\rm sparse},  x_i \rangle$ with $x_i\sim\cN(0,1)$. For $n$ large enough ($n$ of order $s\ln(d)$ where $s$ is the sparsity), the design matrix $X$ is RIP~\citep{candestao}, so that the minimum $\ell_1$ norm interpolator $\beta^\star_{\ell^1}$ is exactly equal to $\beta^\star_{\rm sparse}$.

However, if $\alpha$ is such that $\max_i\alpha_i\to 0$ with $h_i>>1$ for $j\leq s$ and $h_i=1$ for $i\geq s+1$ ($h_i$ as in Proposition~\ref{app:prop:weightedl1norm}), $\beta^\star_\alpha$ will be forced to verify $\beta^\star_{\alpha,i}=0$ for $i\leq s$ and hence $\Vert \beta^\star_{\alpha,1}-\beta^\star_{\ell^1} \Vert _1\geq s$.
\end{example}

\section{Main descent lemma and boundedness of the iterates}\label{app:sec:descentlemmas}

The goal of this section is to prove the following proposition, our main descent lemma: for well-chosen stepsizes, the Bregman divergences $(D_{h_k}(\beta^\star, \beta_k))_{k\geq0}$ decrease.
We then use this proposition to bound the iterates for both SGD and GD.

\begin{restatable}{proposition}{propbregmandescentbound}
\label{prop:bregman_descent_and_bound}
There exist a constant $c>0$ and $B>0$ such that $B=\cO(\inf_{\beta^\star\in\cS}\NRM{\beta^\star}_\infty)$ for GD and $B=\cO(\ln(1/\alpha)\inf_{\beta^\star\in\cS}\NRM{\beta^\star}_\infty)$ for SGD, such that if $\gamma_k\leq\frac{c}{L B}$ for all $k$, then
we have, for all $k\geq0$ and any interpolator $\beta^\star \in \cS$:
\begin{equation*}
    D_{h_{k+1}}(\ww^\star,\ww_{k+1})\leq D_{h_k}(\ww^\star,\ww_k)-\gamma_k\cL_{\cB_k}(\ww_k)\,.
\end{equation*}
\end{restatable}

To prove this result, we first provide a general descent lemma for time-varying mirror descent (\cref{prop:MD_varying}, \cref{app:MD_varying}), before proving the proposition for fixed iteration $k$ and bound $B>0$ on the iterates infinity norm in Appendix~\ref{app:descent_lemma_fixed_k} (\cref{prop:bregman_descent}). We finally use this to prove a bound on the iterates infinity norm in \cref{app:bound_iterates}.

\subsection{Descent lemma for (stochastic) mirror descent with varying potentials} \label{app:MD_varying}

In the following we adapt a classical mirror descent equality but for time varying potentials, that differentiates from \citet{orabona2013varyingmirror} in that it enables us to prove the decrease of the Bregman divergences of the iterates. Moreover, as for classical MD, it is an equality.

\begin{proposition}\label{prop:MD_varying}
    For $h,g:\R^d\to\R$ functions, let $D_{h,g}(\ww,\ww')=h(\ww)-g(\ww')-\langle \nabla g(\ww'),\ww-\ww'\rangle$\footnote{for $h=g$, we recover the classical Bregman divergence that we denote $D_h=D_{h,h}$} for $\beta,\beta'\in\R^d$.
    Let $(h_k)$ strictly convex functions defined $\R^d$ $\cL$ a convex function defined on $\R^d$.
    Let $(\ww_k)$ defined recursively through $\ww_0\in \R^d$, and
    \begin{equation*}
        \ww_{k+1}\in\argmin_{\ww\in \R^d} \set{\gamma_k\langle\nabla \cL(\ww_k),\ww-\ww_k\rangle + D_{h_{k+1},h_k}(\ww,\ww_k)}\,,
    \end{equation*}
    where we assume that the minimum is unique and attained in $\R^d$. Then, $(\ww_k)$ satisfies
    \begin{equation*}
        \nabla h_{k+1}(\ww_{k+1})=\nabla h_k(\ww_k)-\gamma_k \nabla \cL(\ww_k)\,,
    \end{equation*}
    and for any $\beta\in \R^d$,
    \begin{align*}
        D_{h_{k+1}}(\ww,\ww_{k+1})&=D_{h_k}(\ww,\ww_k)-\gamma_k \langle \nabla \cL(\beta_k), \beta_k - \beta \rangle + D_{h_{k+1}}(\ww_k,\ww_{k+1})\\
        &\quad - \big(h_{k+1}-h_{k}\big)(\ww_k) + \big(h_{k+1}-h_{k}\big)(\ww)\,.
    \end{align*}
\end{proposition}

\begin{proof}
    Let $\beta\in \R^d$.
    Since we assume that the minimum through which $\ww_{k+1}$ is computed is attained in $\R^d$, the gradient of the function $V_k(\ww)=\gamma_k\langle\nabla \cL(\ww_k),\ww-\ww_k\rangle + D_{h_{k+1},h_k}(\ww,\ww_k)$ evaluated at $\ww_{k+1}$ is null, leading to $\nabla h_{k+1}(\ww_{k+1})=\nabla h_k(\ww_k)-\gamma_k \nabla \cL(\ww_k)$.

    Then, since $\nabla V_{k}(\ww_{k+1})=0$, we have $D_{V_k}(\ww,\ww_{k+1})=V_k(\ww)-V_k(\ww_{k+1})$.
    Using $\nabla^2 V_k=\nabla^2h_{k+1}$, we also have $D_{V_k}=D_{h_{k+1}}$. Hence:
    \begin{equation*}
        D_{h_{k+1}}(\ww,\ww_{k+1})=\gamma_k\langle\nabla\cL(\ww_k),\ww-\ww_{k+1}\rangle  + D_{h_{k+1},h_k}(\ww,\ww_k)-D_{h_{k+1},h_k}(\ww_{k+1},\ww_k)\,.
    \end{equation*}
    We write $\gamma_k\langle\nabla\cL(\ww_k),\ww-\ww_{k+1}\rangle=\gamma_k\langle\nabla\cL(\ww_k),\ww -\ww^{k}\rangle + \gamma_k\langle\nabla\cL(\ww_k),\ww_k-\ww_{k+1}\rangle$.
    We also have $\gamma_k\langle\nabla\cL(\ww_k),\ww_k-\ww_{k+1}\rangle= \langle\nabla h_k(\ww_k)-\nabla h_{k+1}(\ww_{k+1}),\ww_k-\ww_{k+1}\rangle=D_{h_k,h_{k+1}}(\ww_k,\ww_{k+1})+D_{h_{k+1},h_{k}}(\ww_{k+1},\ww^{k})$, so that $\gamma_k\langle\nabla\cL(\ww_k),\ww_k-\ww_{k+1}\rangle-D_{h_{k+1},h_{k}}(\ww_{k+1},\ww^{k})=D_{h_k,h_{k+1}}(\ww_k,\ww_{k+1})$.
    Thus,
    \begin{equation*}
        D_{h_{k+1}}(\ww ,\ww_{k+1})=D_{h_{k+1},h_k}(\ww ,\ww_k)-\gamma_k\big(D_f(\ww ,\ww_k)+D_f(\ww_k,\ww )\big) + D_{h_k,h_{k+1}}(\ww_k,\ww_{k+1})\,,
    \end{equation*}
    and writing $D_{h,g}(\ww,\ww')=D_g(\ww,\ww')+h(\ww)-g(\ww)$ concludes the proof.
\end{proof}

\subsection{Proof of Proposition~\ref{prop:bregman_descent}}\label{app:descent_lemma_fixed_k}

In next proposition, we use Proposition~\ref{prop:MD_varying} to prove our main descent lemma.
To that end, we bound the error terms that appear in Proposition~\ref{prop:MD_varying} as functions of $\cL_{\cB_k}(\beta_k)$ and norms of $\beta_k,\beta_{k+1}$, so that for explicit stepsizes, the error terms can be cancelled by half of the negative quantity $-2\cL_{\cB_k}(\beta_k)$.

\textbf{Additional notation:} let $L_2,L_\infty>0$ such that $\forall\beta$, $\Vert H_\cB \beta \Vert_2 \leq L \Vert \beta \Vert_2$, $\NRM{H_\cB \beta}_\infty\leq L\NRM{\beta}_\infty$ for all batches $\cB \subset [n]$ of size $b$.

\begin{restatable}{proposition}{propbregmandescent}
\label{prop:bregman_descent}
Let $k\geq0$ and $B>0$.
Provided that $\NRM{\beta_{k}}_\infty,\NRM{\beta_{k+1}}_\infty,\NRM{\beta^\star}_\infty\leq B$ and $\gamma_k\leq\frac{c}{L B}$ where $c>0$ is some numerical constant,
we have:
\begin{equation}\label{eq:2_descent_bis}
    D_{h_{k+1}}(\ww^\star,\ww_{k+1})\leq D_{h_k}(\ww^\star,\ww_k)-\gamma_k\cL_{\cB_k}(\ww_k)\,.
\end{equation}
\end{restatable}

\begin{proof}
Let $\beta^\star\in\cS$ be any interpolator.
From Proposition~\ref{prop:MD_varying}:
\begin{align*}
   D_{h_{k+1}}(\beta^\star, \beta_{k+1}) = D_{h_k}(\beta^\star, \beta_k) - 2 \gamma_k \cL_{\cB_k}(\beta_{k}) +  D_{h_{k+1}}(\beta_{k+1}, \beta_k) - (h_{k+1}- h_{k})(\beta_k)  + (h_{k+1} - h_k)(\beta^\star).
\end{align*}
We want to bound the last three terms of this equality. First, to bound the last two we apply \cref{tech_lemma:bound_hk+1-hk} assuming that $\Vert \beta^\star \Vert_\infty, \Vert \beta_{k+1} \Vert_\infty \leq B$:
\begin{align*}
    - (h_{k+1}- h_{k})(\beta_k)  + (h_{k+1} - h_k)(\beta^\star) \leq 24 B L_2\gamma_k^2\cL_{\cB_k}(\ww_k)
\end{align*}

We now bound $D_{h_{k+1}}(\ww_k,\ww_{k+1})$.
Classical Bregman manipulations provide that 
\begin{align*}
  D_{h_{k+1}}(\ww_k,\ww_{k+1}) &=D_{h_{k+1}^*}(\nabla h_{k+1}(\ww_{k+1}),\nabla h_{k+1}(\ww_k)) \\
  &=D_{h_{k+1}^*}(\nabla h_{k}(\ww^{k})-\gamma_k\nabla \cL_{\cB_k}(\ww_k),\nabla h_{k+1}(\ww_k))\, .
\end{align*}
From Lemma~\ref{lemma:relative_smoothness} we have that $h_{k+1}$ is $\min(1/(4\alpha_{k+1}^2),1/(4B))$ strongly convex on the $\ell^\infty$-centered ball of radius $B$ therefore $h_{k+1}^*$ is $\max(4\alpha_{k+1}^2,4B)=4B$ (for $\alpha$ small enough or $B$ big enough) smooth on this ball, leading to:
\begin{align*}
    D_{h_{k+1}}(\ww_k,\ww_{k+1}) 
    &\leq 2B \NRM{\nabla h_{k}(\ww_{k})-\gamma_k\nabla \cL_{\cB_k}(\ww_k)-\nabla h_{k+1}(\ww_k)}_2^2\\
    &\leq 4B\big(\NRM{\nabla h_{k}(\ww_{k})-\nabla h_{k+1}(\ww_k)}_2^2+\NRM{\gamma_k\nabla \cL_{\cB_k}(\ww_k)}_2^2\big)\,.
\end{align*}
Using $|\nabla h_k(\ww)-\nabla h_{k+1}(\ww)|\leq 2\delta_k$ where $ \delta_k = q(\gamma_k \nabla \cL_{\cB_k}(\beta_k))$, we get that:
\begin{equation*}
    D_{h_{k+1}}(\ww_k,\ww_{k+1})\leq 8B\NRM{\delta_k}_2^2 + 4B L \gamma_k^2\cL_{\cB_k}(\ww_k)\,.
\end{equation*}
Now, $\NRM{\delta_k}_2^2 \leq \NRM{\delta_k}_1\NRM{\delta_k}_\infty$ and using Lemma~\ref{tech_lemma:bound_q},  $\NRM{\delta_k}_1\NRM{\delta_k}_\infty\leq 4 \NRM{\gamma_k\nabla \cL_{\cB_k}(\ww_k)}_2^2\NRM{\gamma_k\nabla \cL_{\cB_k}(\ww_k)}_\infty^2\leq 2 \NRM{\gamma_k\nabla \cL_{\cB_k}(\ww_k)}_2^2 $ since $\NRM{\gamma_k\nabla \cL_{\cB_k}(\ww_k)}_\infty\leq\gamma_k L_\infty\NRM{\beta_k-\beta_\infty}\leq \gamma_k\times 2LB\leq 1/2$ is verified for $\gamma_k\leq 1/(4LB)$.
Thus, 
\begin{equation*}
    D_{h_{k+1}}(\ww_k,\ww_{k+1})\leq 40 BL_2\gamma_k^2\cL_{\cB_k}(\ww_k)\,.
\end{equation*}

Hence, provided that $\NRM{\beta_k}_\infty\leq B$, $\NRM{\beta_{k+1}}_\infty\leq B$ and $\gamma_k\leq 1/(4LB)$, we have:
    \begin{align*}
        D_{h_{k+1}}(\ww^\star,\ww_{k+1})\leq D_{h_k}(\ww^\star,\ww_k)-2\gamma_k\cL_{\cB_k}(\ww_k) +  64 L_2\gamma_k^2 B \cL_{\cB_k}(\ww_k)\,,
    \end{align*}
    and thus 
    \begin{equation*}
                D_{h_{k+1}}(\ww^\star,\ww_{k+1})\leq D_{h_k}(\ww^\star,\ww_k)-\gamma_k\cL_{\cB_k}(\ww_k)\,.
    \end{equation*}
if $\gamma_k\leq \frac{c}{BL}$, where $c = \frac{1}{64}$.



\end{proof}

\subsection{Bound on the iterates}\label{app:bound_iterates}

We now bound the iterates $(\beta_k)$ by an explicit constant $B$ that depends on $\NRM{\beta^\star}_1$ (for any fixed $\beta^\star\in\cS$). 

The first bound we prove holds for both SGD and GD, and is of the form $\cO(\NRM{\beta^\star}_1\ln(1/\alpha^2)$ while the second bound, that holds only for GD ($b=n$) is of order $\cO(\NRM{\beta^\star}_1)$ (independent of $\alpha$). While a bound independent of $\alpha$ is only proved for GD, we believe that such a result also holds for SGD, and in both cases $B$ should be thought of order $\cO(\NRM{\beta^\star}_1)$.

\subsubsection{Bound that depends on $\alpha$ for GD and SGD}

A consequence of Proposition~\ref{prop:bregman_descent} is the boundedness of the iterates, as shown in next corollary. Hence, Proposition~\ref{prop:bregman_descent} can be applied using $B$ a uniform bound on the iterates $\ell^\infty$ norm.

\begin{corollary}
\label{cor:bound_iterates_sgd}
Let $B=3\NRM{\beta^\star}_1\ln\big(1+\frac{\NRM{\beta^\star}_1}{\alpha^2}\big)$. For stepsizes $\gamma_k\leq \frac{c}{BL}$, we have $\NRM{\beta_k}_\infty\leq B$ for all $k\geq0$.
\end{corollary}

\begin{proof}
We proceed by induction.
Let $k\geq0$ such that $\NRM{\beta_k}_\infty\leq B$ for some $B>0$ and $D_{h_{k}}(\ww^\star,\ww_{k})\leq D_{h_0}(\ww^\star,\ww_0)$ (note that these two properties are verified for $k=0$, since $\beta_0=0$). For $\gamma_k$ sufficiently small (\emph{i.e.}, that satisfies $\gamma_k\leq \frac{c}{B'L}$ where $B'\geq \NRM{\beta_{k+1}}_\infty,\NRM{\beta_{k}}_\infty,\NRM{\beta^\star}_\infty$), using Proposition~\ref{prop:bregman_descent}, we have $D_{h_{k+1}}(\ww^\star,\ww_{k+1})\leq D_{h_k}(\ww^\star,\ww_k)$ so that $D_{h_{k+1}}(\ww^\star,\ww_{k+1})\leq D_{h_0}(\ww^\star,\ww_0)$, which can be rewritten as:
\begin{equation*}
         \sum_{i=1}^d \alpha_{k+1,i}^2(\sqrt{1+(\frac{\ww_{k+1,i}}{\alpha_{k+1,i}^2})^2}-1)\leq \sum_{i=1}^d \ww_i^\star \argsinh(\frac{\ww_{k+1,i}}{\alpha^2})\,.
    \end{equation*}
Hence, $\NRM{\ww_{k+1}}_1\leq \NRM{\ww^\star}_1\ln(1+\frac{\NRM{\ww_{k+1}}_1}{\alpha^2})$.
We then notice that for $x,y>0$, $x\leq y\ln(1+x)\implies x\leq 3y\ln(1+y)$: if $x> y\ln(1+y)$ and $x>y$, we have that $y\ln(1+y)<y\ln(1+x)$, so that $1+y<1+x$, which contradicts our assumption. Hence, $x\leq \max(y,y\ln(1+y))$. In our case, $x=\NRM{\beta^{k+1}}_1/\alpha^2$, $y=\NRM{\beta^\star}_1/\alpha^2$ so that for small alpha, $\ln(1+y)\geq1$.

Hence, we deduce that $\NRM{\ww_{k+1}}_1\leq B$, where $B=\NRM{\ww^\star}_1\ln(1+\frac{\NRM{\ww^\star}_1}{\alpha^2})$. 

This is true as long as $\gamma_k$ is tuned using $B'$ a bound on $\max(\NRM{\beta_k}_\infty,\NRM{\beta_{k+1}}_\infty)$.
Using the continuity of $\beta_{k+1}$ as a function of $\gamma_k$ ($\beta_k$ being fixed), we show that $\gamma_k\leq \half\times\frac{c}{BL}$ can be used using this $B$.
Indeed, let $\phi:\R^+\to\R^d$ be the function that takes as entry $\gamma_k\geq0$ and outputs the corresponding $\NRM{\beta_{k+1}}_\infty$: $\phi$ is continuous. 
Let $\gamma_r=\half\times\frac{c}{rL}$ for $r>0$ and 
$\bar r=\sup\set{r\geq0: B < \phi(\gamma_r)}$ (the set is upper-bounded; if is is empty, we do not need what follows since it means that any stepsize leads to $\NRM{\beta_{k+1}}_\infty\leq B$).
By continuity of $\phi$, $ \phi(\gamma_{\bar r})=B$. 
Furthermore, for all $r$ that satisfies $r\geq \max(\phi(\gamma_r),B)\geq \max(\phi(\gamma_r),\NRM{\beta_k}_\infty,\NRM{\beta^\star}_\infty) $, we have, using what is proved just above, that $\NRM{\beta_{k+1}}_\infty\leq B$ and thus $\phi(\gamma_r)\leq B$ for such a $r$:

\begin{lemma}\label{lemma:r}
For $r>0$ such that $r\geq \max(\phi(\gamma_r),B)$, we have $\phi(\gamma_r)\leq B$.
\end{lemma}

Now, if $\bar r> B$, by definition of $\bar r$ and by continuity of $\phi$, since $\phi(\bar r)=B$, there exists some $B<r<\bar r$ such that $\phi(\gamma_r)>B$ (definition of the supremum) and $\phi(\gamma_r)\leq 2B$ (continuity of $\phi$).
This particular choice of $r$ thus satisfies $r>B$ and and $\phi(\gamma_r)\leq 2B \leq 2r$, leading to $\phi(\gamma_r)\leq B$, using Lemma~\ref{lemma:r}, hence a contradiction: we thus have $\bar r \leq B$.

This concludes the induction: for all $r\geq B$, we have $r\geq \bar r$ so that $\phi(\gamma_r)\leq B$ and thus for all stepsizes $\gamma\leq \frac{c}{2LB}$, we have $\NRM{\beta_{k+1}}_\infty\leq B$.

\end{proof}

\subsubsection{Bound independent of $\alpha$}

We here assume in this subsection that $b=n$. We prove that for gradient descent, the iterates are bounded by a constant that does not depend on $\alpha$.
\begin{restatable}{proposition}{propboundediterates} 
\label{prop:bound_indep}
Assume that $b=n$ (full batch setting).
There exists some $B=\cO(\NRM{\beta^\star}_1)$ such that for stepsizes $\gamma_k\leq \frac{c}{BL}$, we have $\NRM{\beta_k}_\infty\leq B$ for all $k\geq0$.
\end{restatable}

\begin{proof}

We first begin by proving the following proposition: for sufficiently small stepsizes, the loss values decrease.
In the following lemma we provide a bound on the gradient descent iterates $(w_{+, k}, w_{-, k})$ which will be useful to show that the loss is decreasing.

\begin{proposition}
    For $\gamma_k \leq \frac{c}{LB}$ where $B\geq\max(\NRM{\beta_k}_\infty, \NRM{\beta_{k+1}}_\infty)$, we have $\cL(\beta_{k+1})\leq \cL(\beta_k)$    
\end{proposition}
\begin{proof}
Oddly, using the time-varying mirror descent recursion is not the easiest way to show the decrease of the loss, due to the error terms which come up. Therefore to show that the loss is decreasing we use the gradient descent recursion.
Recall that the iterates $w_k = (w_{+, k}, w_{-, k}) \in \R^{2d}$ follow a gradient descent on the non convex loss $F(w) = \frac{1}{2} \Vert y - \half X (w_+^2 - w_-^2) \Vert_2 $.

For $k\geq 0$, using the Taylor formula we have that $F(w_{k+1})\leq F(w_k)-\gamma_k(1-\frac{\gamma_kL_k}{2})\NRM{\nabla F(w_k)}^2$ with the local smoothness $L_k = \sup_{w\in[w_k,w_{k+1}]}\lambda_{\max}(\nabla^2F(w))$.
Hence if $\gamma_k\leq 1/L_k$ for all $k$ we get that the loss is non-increasing.
We now bound $L_k$. Computing the hessian ot $F$, we obtain that:
\begin{equation}\label{eq:hessian_F}
    \begin{aligned}
\nabla^2F(w_k)&=
    \begin{pmatrix}
    \diag(\nabla \cL(\beta_k)) & 0 \\
    0 & - \diag(\nabla \cL(\beta_k))
    \end{pmatrix}
    \\
    &\quad+
    \begin{pmatrix}
   \ \ \  \diag(w_{+,k}) H \diag(w_{+,k}) & - \diag(w_{-,k}) H \diag(w_{+,k})\\
    -\diag(w_{+,k}) H \diag(w_{-,k}) & \ \ \  \diag(w_{-,k}) H \diag(w_{-,k})
    \end{pmatrix}\,.
\end{aligned}
\end{equation}
Let us denote by $M = \begin{pmatrix}
         M_+ & M_{+, -} \\
         M_{+, -} & M_-
     \end{pmatrix} \in \R^{2d \times 2d}$ the second matrix in the previous equality.
With this notation $\Vert \nabla^2 F(w_k) \Vert \leq  \Vert \nabla \cL(\beta_k) \Vert_\infty + 2 \Vert M \Vert $ (where the norm corresponds to the Schatten $2$-norm which is the largest eigenvalue for symmetric matrices). 
Now, notice that:
\begin{align*}
    \Vert M \Vert^2 &= \underset{u \in \R^{2 d}, \Vert u \Vert = 1}{ \sup} \Vert M u \Vert^2 \\
    &= \underset{ \substack{ u_+ \in \R^{d} , \Vert u_+ \Vert = 1 \\
    u_- \in \R^{d} , \Vert u_- \Vert = 1 \\
     (a, b) \in \R^2, a^2 + b^2 = 1 } }{ \sup} \Big \Vert M \begin{pmatrix}
         a \cdot u_+ \\ b \cdot u_-
     \end{pmatrix} \Big \Vert^2\,.
\end{align*}
We have:
\begin{align*}
   \Big \Vert M \begin{pmatrix}
         a \cdot u_+ \\ b \cdot u_-
     \end{pmatrix} \Big \Vert^2 
     &= \Big \Vert  \begin{pmatrix}
         a M_+ u_+ + b M_{+-} u_- \\ a M_{+-} u_+ + b M_- u_-
     \end{pmatrix} \Big \Vert^2 \\
     &= \Vert a M_+ u_+ + b M_{+-} u_- \Vert^2 + \Vert  a M_{+-} u_+ + b M_- u_- \Vert^2 \\
     &\leq 2 \Big ( a^2  \Vert M_+ u_+ \Vert^2 + b^2  \Vert M_{+-} u_- \Vert^2 + a^2 \Vert M_{+-} u_+  \Vert^2 + b^2 \Vert M_- u_-  \Vert^2 \Big )\\
     &\leq 2 \Big ( \Vert M_+ \Vert^2 +  \Vert M_{+-} \Vert^2+ \Vert M_- \Vert^2 \Big )\,.
\end{align*}
Since $\Vert M_{\pm} \Vert \leq \lambda_{max} \cdot \Vert w_\pm \Vert_\infty^2$ and $ \Vert M_{+-} \Vert \leq \lambda_{max} \Vert w_+ \Vert_\infty \Vert w_- \Vert_\infty$ we finally get that 
\begin{align*}
    \Vert M \Vert^2 &\leq 6 \lambda_{max}^2 \cdot \max( \Vert w_+ \Vert_\infty^2, \Vert w_-\Vert_\infty^2)^2 \\
    &\leq 6 \lambda_{max}^2 ( \Vert w_+^2 \Vert_\infty +  \Vert w_-^2 \Vert_\infty)^2 \\
    &\leq 12 \lambda_{max}^2  \Vert w_+^2 + w_-^2 \Vert_\infty^2\,.
\end{align*}

We now upper bound this quantity in the following lemma.

\begin{lemma}
\label{lemma:bounded_sum_w++w-}
    For all $k \geq 0$, the following inequality holds component-wise:
    \begin{align*}
        w_{+, k}^2 + w_{-,k}^2 &= \sqrt{4 \balpha_k^4 + \beta_k^2} \,.
    \end{align*} 
\end{lemma}

\begin{proof}
Notice from the definition of $w_{+, k}$ and $w_{-, k}$ given in the proof of \Cref{prop:tv_md} that:
\begin{align}
\label{eq:w+w-}
|w_{+, k}| |w_{-, k}| = \balpha_{-, k} \balpha_{+, k} = \balpha_k^2.
\end{align}

And $\balpha_0=\balpha^2$. Now since $\balpha_k$ is decreasing coordinate-wise (under our assumptions on the stepsizes, $\gamma_k^2\nabla\cL(\beta_k)^2\leq (1/2)^2<1$), we get that.:
\begin{equation*}
    w_{+,k}^2+w_{-,k}^2 = 2\sqrt{\balpha_k^4 + \beta_k^2} \leq 2\sqrt{ \balpha^4 + \beta_k^2}
\end{equation*}
leading to $w_{+, k}^2 + w_{-,k}^2 \leq \sqrt{4\balpha^4+B^2}$.
\end{proof}

From \cref{lemma:bounded_sum_w++w-}, $w_{+, k}^2+w_{-,k}^2$ is bounded by $2\sqrt{ \balpha^4 +B^2}$. Putting things together we finally get that $\Vert \nabla^2 F(w) \Vert \leq  \Vert \nabla \cL(\beta) \Vert_\infty + 8 \lambda_{max} \sqrt{4 \Vert \balpha \Vert_\infty^4 + B^2}$. Hence,
\begin{equation*}
    L_k\leq \sup_{\NRM{\beta}_\infty\leq B}\NRM{\nabla\cL(\beta)}_\infty + 8\lambda_{\max}\sqrt{\Vert \balpha \Vert_\infty^4+B^2}\leq LB+ 8\lambda_{\max}\sqrt{\Vert \balpha \Vert_\infty^4+B^2}\leq 10LB\,,
\end{equation*}
for $B\geq \Vert \balpha \Vert_\infty^2$.
\end{proof}

We finally prove the bound on $\NRM{\beta_k}_\infty$ independent of $\alpha$ for a uniform initialisation $\balpha = \alpha \mathbf{1}$, using the monotonic property of $\cL$.

\propboundediterates
\begin{proof}
In this proof, we first let $B$ be a bound on the iterates. Tuning stepsizes using this bound, we prove that the iterates are bounded by a some $B'=\cO(\NRM{\beta^\star}_1)$. Finally, we conclude by using the continuity of the iterates (at a finite horizon) that this explicit bound can be used to tune the stepsizes.

 Writing the mirror descent with varying potentials, we have, since $\nabla h_0(\beta_0)=0$,
\begin{equation*}
    \nabla h_k(\beta_k)=-\sum_{\ell<k}\gamma_\ell\nabla\cL(\beta_\ell)\,,
\end{equation*}
leading to, by convexity of $h_k$:
\begin{equation*}
    h_k(\beta_k)-h_k(\beta^\star)\leq \langle \nabla h_k(\beta_k),\beta_k-\beta^\star\rangle = -\sum_{\ell<k}\langle\gamma_\ell\nabla\cL(\beta_\ell),\beta_k-\beta^\star\rangle \,.
\end{equation*}
We then write, using $\nabla\cL(\beta)=H(\beta-\beta^\star)$ for $H=XX^\top$, that $-\sum_{\ell<k}\langle\gamma_\ell\nabla\cL(\beta_\ell),\beta_k-\beta^\star\rangle= -\sum_{\ell<k}\gamma_\ell \langle X^\top(\bar\beta_k-\beta^\star),X^\top(\beta_k-\beta^\star)\rangle \leq  \sum_{\ell<k}\gamma_\ell \sqrt{\cL(\bar\beta_k)\cL(\beta_k)} $, leading to:
\begin{equation*}
    h_k(\beta_k)-h_k(\beta^\star)\leq 2\sqrt{\sum_{\ell<k}\gamma_\ell \cL(\bar\beta_k)\sum_{\ell<k}\gamma_\ell \cL(\beta_k)} \leq 2\sum_{\ell<k}\gamma_\ell \cL(\bar\beta_k) \leq 2D_{h_0}(\ww^\star,\ww^0)\,,
\end{equation*}
where the last inequality holds provided that $\gamma_k\leq \frac{1}{ CLB}$. Thus,
\begin{equation*}
    \psi_{\balpha_k}(\beta_k)\leq \psi_{\balpha_k}(\beta^\star) + 2 \psi_{\balpha_0}(\beta^\star) + \langle \phi_k,\beta_k-\beta^\star\rangle\,.
\end{equation*}
Then, $\langle \phi_k,\beta_k-\beta^\star\rangle\leq \NRM{\phi_k}_1\NRM{\beta_k-\beta^\star}_\infty$ and $\NRM{\phi_k}_1\leq C\lambda_{\max}\sum_{k<K}\gamma_k^2\cL(\beta^k) \leq C\lambda_{\max}\gamma_{\max}h_0(\beta^\star)$.
Then, using
$$ \Vert \beta \Vert_\infty - \frac{1}{\ln(1 / \alpha^2)} \leq \frac{\psi_\alpha(\beta)}{\ln(1 / \alpha^2)} \leq \Vert \beta \Vert_1 \big (1 + \frac{\ln(\Vert \beta \Vert_1 + \alpha^2) }{\ln(1 / \alpha^2)} \big)\,, $$
we have:
\begin{align*}
    \NRM{\beta_k}_\infty &\leq \frac{1}{\ln(1 / \alpha^2)} + \Vert \beta^\star \Vert_1 \big (1 + \frac{\ln(\Vert \beta^\star \Vert_1 + \alpha^2) }{\ln(1 / \alpha^2)} \big) + \Vert \beta^\star \Vert_1 \big (1 + \frac{\ln(\Vert \beta^\star \Vert_1 + \alpha^2) }{\ln(1 / \alpha^2)} \big)\\
    &\quad +  B_0C\lambda_{\max}\gamma_{\max}h_0(\beta^\star)/\ln(1/\alpha^2)\\
    & \leq R + B_0C\lambda_{\max}\gamma_{\max}h_0(\beta^\star)/\ln(1/\alpha^2)\,,
\end{align*}
where $R=\cO(\NRM{\beta^\star}_1)$ is independent of $\alpha$.
Hence, since $B_0=\sup_{k<\infty} \NRM{\beta_k}_\infty<\infty$, we have:
\begin{equation*}
    B_0(1-C\lambda_{\max}\gamma_{\max}h_0(\beta^\star)/\ln(1/\alpha^2))\leq R\implies B_0\leq 2R\,,
\end{equation*}
provided that $\gamma_{\max}\leq 1/(2C\lambda_{\max}h_0(\beta^\star)/\ln(1/\alpha^2))$ (note that $h_0(\beta^\star)/\ln(1/\alpha^2)$ is independent of $\alpha^2$).

Hence, if for all $k$ we have $\gamma_k\leq \frac{1}{C'LB}$ where $B$ bounds all $\NRM{\beta_k}_\infty$, we have $\NRM{\beta_k}_\infty\leq 2R$ for all $k$, where $R=\cO(\NRM{\beta^\star}_1)$ is independent of $\alpha$ and stepsizes $\gamma_k$.

Let $K>0$ be fixed, and $$\bar \gamma=\inf\set{\gamma>0\quad \text{s.t.}\quad \sup_{k\leq K}\NRM{\beta_k}_\infty>2R}\,.$$ 
For $\gamma\geq0$ a constant stepsize, let $$\varphi(\gamma)=\sup_{k\leq K}\NRM{\beta_k}_\infty\,,$$ which is a continuous function of $\gamma$.
For $r>0$, let $\gamma_r=\frac{1}{C'Lr}$. 

An important feature to notice is that if $\gamma<\gamma_r$ and $r$ bounds all $\NRM{\beta_k}_{\infty},k\leq K$, then $\varphi(\gamma)\leq R$, as shown above.
We will show that we have $\bar\gamma\geq\gamma_{2R}$.
Reasoning by contradiction, if $\bar\gamma<\gamma_{2R}$: by continuity of $\varphi$, we have $\varphi(\bar\gamma)\leq R$ and thus, there exists some small $0<\eps<\gamma_{2R}-\bar\gamma$ such that for all $\gamma\in[\bar\gamma,\bar\gamma+\eps]$, we have $\varphi(\bar\gamma)\leq 2R$. 

However, such $\gamma$'s verify both $\varphi(\gamma)\leq 2R$ (since $\gamma\in[\bar\gamma,\bar\gamma+\eps]$ and by definition of $\eps$) and $\gamma\leq \gamma_{2R}$ (by definition of $\eps$), and hence $\varphi(\gamma)\leq R$. This contradicts the infimum of $\bar\gamma$, and hence $\bar\gamma\geq \gamma_{2R}$.
Thus, for $\gamma\leq \gamma_{2R}=\frac{1}{2C'LR}$, we have $\NRM{\beta_k}_\infty\leq R$.
\end{proof}

\end{proof}

\section{Proof of \cref{thm:implicit_bias,thm:conv_iterates}, and of \Cref{prop:conv_quantitative}}\label{app:sec:proof_main}

\subsection{Proof of \cref{thm:implicit_bias,thm:conv_iterates}}

We are now equipped to prove \cref{thm:implicit_bias} and \cref{thm:conv_iterates}, condensed in the following Theorem.

\begin{restatable}{theorem}{maintheorem}
\label{thm:conv_gd_appendix}
Let $(u_k,v_k)_{k\geq0}$ follow the mini-batch SGD recursion \eqref{eq:SGD_recursion} initialised at $u_0=\sqrt{2} \balpha\in\R_{>0}^d$ and $v_0=\mathbf{0}$, and let $(\beta_k)_{k\geq0}=(u_k\odot v_k)_{k\geq0}$.
There exists and explicit $B>0$ and a numerical constant $c>0$ such that: 
\begin{enumerate}
\item For stepsizes satisfying $\gamma_k\leq \frac{c}{LB}$, the iterates satisfy $\NRM{\gamma_k\nabla\cL_{\cB_k}(\beta_k)}_\infty\leq 1$ and $\NRM{\beta_k}_\infty\leq B$ for all $k$;
    \item For stepsizes satisfying $\gamma_k\leq \frac{c}{LB}$, $(\beta_k)_{k\geq0}$ converges almost surely to some $\beta_\infty^\star\in\cS$,
    \item If $(\beta_k)_k$ and the neurons $(u_k,v_k)_k$ respectively converge to a model $\beta^\star_\infty$ and neurons $(u_\infty,v_\infty)$ satisfying $\beta^\star_\infty\in\cS$ (and $\beta^\star_\infty=u_\infty\odot v_\infty$), then for almost all stepsizes (with respect to the Lebesgue measure), the limit $\beta^\star_\infty$ satisfies:
    \begin{equation*}
        \beta^\star_{\infty} =  \underset{ \beta^\star \in \cS }{\argmin} \  D_{\psi_{\balpha_\infty}}(\beta^\star,\tilde\beta_0) \,,
    \end{equation*}
    for $\balpha_\infty\in\R^d_{>0}$ and $\tilde\beta_0\in\R^d$ satisfying
    \begin{equation*}
        \balpha_\infty^2=\balpha^2\odot\exp\left(-\sum_{k=0}^\infty q\big(\gamma_k \nabla \cL_{\cB_k}(\beta_k)\big)\right)\,,
    \end{equation*}
    where $q(x)=-\frac{1}{2} \ln((1-x^2)^2)\ge0$  for  $|x|\le \sqrt{2}$, 
and $\Tilde{\beta}_0$ is a perturbation term equal to:   
    \begin{align*}
         \tilde{\beta}_0=\half\big(\balpha_+^2-\balpha_-^2\big),
    \end{align*}
    where, $q_\pm(x)= \mp 2x - \ln((1\mp x)^2)$, and  $\balpha_{\pm}^2=\balpha^2\odot\exp\left(-\sum_{k=0}^\infty q_\pm(\gamma_k \nabla \cL_{\cB_k}(\beta_k))\right)$.
\end{enumerate}

\end{restatable}

\begin{proof}
\myparagraph{Point 1.}
The first point of the Theorem is a direct consequence of Corollary~\ref{cor:bound_iterates_sgd} and the bounds proved in \cref{app:bound_iterates}.

\myparagraph{Point 2.}
Then, for stepsizes $\gamma_k\leq \frac{c}{LB}$, using Proposition~\ref{prop:bregman_descent_and_bound} for any interpolator $\beta^\star \in \cS$:
\begin{equation}
\label{app:eq:sumlosses}
    D_{h_{k+1}}(\ww^\star,\ww_{k+1})\leq D_{h_k}(\ww^\star,\ww_k)-\gamma_k\cL_{\cB_k}(\ww_k)\,.
\end{equation}
Hence, summing:
\begin{equation*}
    \sum_k \gamma_k\cL_{\cB_k}(\beta_k)\leq D_{h_0}(\beta^\star,\beta_0)\,,
\end{equation*}
so that the series converges.

Under our stepsize rule, $\NRM{\gamma_k\nabla\cL_{\cB_k}(\beta_k)}_\infty\leq \half$, leading to $\NRM{q(\gamma_k\nabla\cL_{\cB_k}(\beta_k)}_\infty\leq 3 \NRM{\gamma_k\nabla\cL_{\cB_k}(\beta_k)}_\infty^2$ by \cref{tech_lemma:bound_q}. Using $\NRM{\nabla \cL_{\cB_k}(\beta_k)}^2\leq 2L_2\cL_{\cB_k}(\beta_k)$, we have that $\ln(\balpha_{\pm,k})$, $\ln(\balpha_k)$ all converge.

We now show that $\sum_k \gamma_k\cL(\beta_k)<\infty$. We have:
\begin{equation*}
    \sum_{\ell<k}\cL(\beta_k)=\sum_{\ell<k}\gamma_k\cL_{\cB_k}(\beta_k) + M_k\,,
\end{equation*}
where $M_k=\sum_{\ell<k}\gamma_k(\cL(\beta_k)-\cL_{\cB_k}(\beta_k))$. We have that $(M_k)$ is a martingale with respect to the filtration $(\cF_k)$ defined as $\cF_k=\sigma(\beta_\ell,\ell\leq k)$. Using our upper-bound on $\sum_{\ell<k}\gamma_k\cL_{\cB_k}(\beta_k)$, we have:
\begin{equation*}
    M_k \geq \sum_{\ell<k}\gamma_k\cL(\beta_k)-\sum_{\ell<k}\gamma_k\cL_{\cB_k}(\beta_k)\geq - D_{h_0}(\beta^\star,\beta_0)\,,
\end{equation*}
and hence $(M_k)$ is a lower bounded martingale. Using Doob's first martingale convergence theorem (a lower bounded super-martingale converges almost surely, \citet{Doob1990-jw}), $(M_k)$ converges almost surely. Consequently, since $\sum_{\ell<k}\gamma_k\cL(\beta_k)=\sum_{\ell<k}\gamma_k\cL_{\cB_k}(\beta_k) + M_k$, we have that $\sum_{\ell<k}\gamma_k\cL(\beta_k)$ converges almost surely (the first term is upper bounded, the second converges almost surely). 

We now prove the convergence of $(\ww_k)$. Since it is a bounded sequence, let $\ww_{\sigma(k)}$ be a convergent sub-sequence and let $\ww^\star_\infty$ denote its limit: $\ww_{\sigma(k)}\to \beta^\star_\infty$.

Almost surely, $\sum_k \gamma_k\cL(\beta_k)<\infty$ and so $\gamma_k\cL(\beta_k)\to0$, leading to $\cL(\beta_k)\to 0$ since the stepsizes are lower bounded, so that $\cL(\beta_{\sigma(k)})\to0$,
and hence $\cL(\beta^\star_\infty)=0$: this means that $\beta^\star_\infty$ is an interpolator.

Since the quantities $(\balpha_k)_k$, $(\balpha_{\pm,k})_k$ and $(\phi_k)_k$ converge almost surely to $\balpha_\infty$, $\balpha_{\pm}$ and $\phi_\infty$, we get that the potentials $h_k$ uniformly converge to $h_\infty = \psi_{\balpha_\infty} - \langle \phi_\infty, \cdot \rangle$ on all compact sets. Now notice that we can decompose $\nabla h_\infty(\beta^\star_\infty)$ as:
\[\nabla h_\infty(\beta^\star_\infty)=\big(\nabla h_\infty(\beta^\star_\infty)-\nabla h_\infty(\ww_{\sigma(k)})\big) + \big(\nabla h_\infty(\ww_{\sigma(k)})-\nabla h_{\sigma(k)}(\ww_{\sigma(k)})\big) +\nabla h_{\sigma(k)}(\ww_{\sigma(k)}).\]

The first two terms converge to 0: the first is a direct consequence of the convergence of the extracted subsequence, the second is a consequence of the uniform convergence of $h_{\sigma(k)}$ to $h_\infty$ on compact sets. Finally the last term is always in $\Span(x_1, \dots, x_n)$ due to \Cref{prop:tv_md}, leading to $\nabla h_\infty(\beta^\star_\infty)\in\Span(x_1, \dots, x_n)$.
Consequently, $\nabla h_\infty(\beta^\star_\infty)\in\Span(x_1, \dots, x_n)$. Notice that from the definition of $h_\infty$, we have that 
$\nabla h_\infty(\beta^\star_\infty) = \nabla \psi_{\balpha_\infty}(\beta^\star_\infty) - \phi_\infty$.  Now since $\phi_\infty = \half\argsinh(\frac{\balpha_{+}^2-\alpha_{-}^2}{2\balpha_\infty^2})$, one can notice that $\tilde{\beta}_0$ is precisely defined such that $\nabla\psi_{\alpha_\infty}(\tilde\beta_0) = \phi_\infty$. Therefore $\nabla \psi_{\balpha_\infty}(\beta^\star_\infty) - \nabla\psi_{\balpha_\infty}(\tilde\beta_0) \in\Span(x_1, \dots, x_n)$. This condition along with the fact that  $\ww^\star_\infty$ is an interpolator are exactly the optimality conditions of the convex minimisation problem:
\[\underset{ \beta^\star \in \cS }{\min} \  D_{\psi_{\balpha_\infty}}(\beta^\star,\tilde\beta_0) \]
Therefore $\beta^\star_\infty$ must be equal to the unique minimiser of this problem.
Since this is true for any sub-sequence we get that  $\ww_k$ converges almost surely to:
\[ \beta^\star_ \infty = \underset{\beta \in \cS}{\argmin} \ \ D_{\psi_{\balpha_\infty}}(\beta^\star,\tilde\beta_0).\]

\myparagraph{Point 3.}
From what we just proved, note that it is sufficient to prove that $\balpha_k,\balpha_{\pm,k},\phi_k$ converge to limits $\balpha_\infty,\balpha_{\pm,\infty},\phi_\infty$ satisfying $\balpha_\infty,\balpha_{\pm,\infty}\in\R^d_{>0}$ (with positive and non-null coordinates) and $\phi_\infty\in\R^d$. Indeed, if this holds and since we assume that the iterates converge to some interpolator, we proved just above that this interpolator is uniquely defined through the desired implicit regularization problem.
We thus prove the convergence of $\balpha_k,\balpha_{\pm,k},\phi_k$.

Note that the convergence of $u_k,v_k$ is equivalent to the convergence of $w_{\pm,k}$ in the $w_+^2-w_-^2$ parameterisation used in our proofs, that we use there too.
We have:
\begin{equation*}
    w_{\pm, k+1} =  (1 \mp \gamma_k \nabla \cL_{\cB_k}(\beta_k)) \odot w_{\pm, k}\,,
\end{equation*}
so that 
\begin{equation*}
    \ln(w_{\pm, k}^2) = \sum_{\ell<k} \ln((1 \mp \gamma_\ell \nabla \cL_{\cB_\ell}(\beta_\ell))^2)\,.
\end{equation*}

We now assume that stepsizes are such that for all $\ell\geq 0$ and $i\in[d]$, stepsizes are such that we have $|\gamma_\ell \nabla_i \cL_{\cB_\ell}(\beta_\ell)|\ne 1$: this is true for all stepsizes except a countable number of stepsizes, and so this is true for almost all stepsizes.
Since we assume that the iterates $\beta_k$ converge to some interpolator, this leads to $\gamma_\ell \nabla \cL_{\cB_\ell}(\beta_\ell) \to 0$ if we assume that stepsizes do not diverge.

Taking the limit, we have 
\begin{equation*}
    \ln(w_{\pm, \infty}^2) = \sum_{\ell<\infty} \ln((1 \mp \gamma_\ell \nabla \cL_{\cB_\ell}(\beta_\ell))^2)\,.
\end{equation*}
This limit is in $(\set{-\infty}\cup\R)^d$ (since $w_{\pm,\infty}\in\R^d$), and a coordinate of the limit is equal to $-\infty$ if and only if the sum on the RHS diverges to $-\infty$ (note that from our assumption just above, no term of the sum can be equal to $-\infty$).

We have $\ln((1 \mp \gamma_\ell \nabla \cL_{\cB_\ell}(\beta_\ell))^2)\sim \mp 2\gamma_\ell \nabla \cL_{\cB_\ell}(\beta_\ell)$ as $\ell\to \infty$, so that if for some coordinate $i$ we have $\sum_{\ell}\gamma_\ell \nabla_i \cL_{\cB_\ell}(\beta_\ell)=\mp\infty$, then the coordinate $i$ of the limit satisfies $\ln(w_{i,\pm, \infty}^2)=+\infty$, which is impossible.
Hence, the sum $\sum_{\ell}\gamma_\ell \nabla \cL_{\cB_\ell}(\beta_\ell)$ is in $\R^d$ (and is thus converging); consequently, $\sum_{\ell}\gamma_\ell^2 \nabla \cL_{\cB_\ell}(\beta_\ell)^2$ converges and thus $\sum_{\ell}q(\gamma_\ell \nabla \cL_{\cB_\ell}(\beta_\ell))$ and $\sum_{\ell}q_\pm(\gamma_\ell \nabla \cL_{\cB_\ell}(\beta_\ell))$ all converge: the sequences $\balpha_k,\balpha_{\pm,k}$ thus converge to limits in $\R^d_{>0}$, and $\phi_k$ converges, concluding our proof.

\end{proof}

\subsection{Proof of \Cref{prop:conv_quantitative}}

We begin with the following Lemma, that explicits the curvature of $D_h$ around the set of interpolators.

\begin{lemma}\label{lem:curvature}
    For all $k\geq 0$, if $\cL(\beta_k)\leq \frac{1}{2\lambda_{\max}}(\alpha^2\lambda^+_{\min})^2$, we have $\NRM{\beta_k-\beta^\star_{\alpha_k}}^2\leq 2B(\alpha^2\lambda^+_{\min})^{-1}\cL(\beta_k)$.
\end{lemma}

\begin{proof}
     Recall that the sequence $\zz^k=\nabla h_k(\ww^k)$ satisfies $\zz^0=0$ and $\zz^{k+1}=\zz^k-\gamma_k\cL(\ww^k)$, so that we have that $\zz^k\in V=\Ima(\XX\XX^\top)$ for all $k\geq0$.
    Then, let $\ww^\alpha_k$ be the unique minimizer of $h_k$ over $\cS$ the space of interpolators: $\ww^\alpha_k$ is exactly characterized by $\XX^\top\ww^\alpha_k=\YY$ and $\nabla h_k(\ww^\alpha_k)\in V$.
    We define $\zz^\alpha_k\in V$ as $\zz^\alpha_k=\nabla h_k(\ww^\alpha_k)$.
    
    
    Now, fix $\zz^\alpha=\zz^\alpha_k$ and $h=h_k$, and let us define $\psi:\zz\in V\to D_{h^*}(\zz,\zz^\alpha)$ and $\phi:\zz\in V\to\cL(\nabla h^*(\zz))$.
    We next show that for all $\zz\in V$, there exists $\mu_z$ such that $\nabla^2\phi(\zz)\geq\mu_z\nabla^2\psi(\zz)$, and that $\mu_z\geq \mu$ for $\zz$ in an open convex set of $V$ around $\zz^\alpha$, for some $\mu>0$.
    For $A\in\R^{d\times d}$ an operator/matrix on $\R^d$, let us denote $A_V$ its restriction/co-restriction to $V$.
    
    First, for $\zz\in V$, we have $\nabla^2\psi(\zz)=\nabla^2(h^*(\zz)-h^*(\zz)-\langle\nabla h^*(\zz^\alpha),z-z^\alpha\rangle)(\zz)=\nabla^2h^*(\zz)_V$.
    Then, $\nabla \phi(\zz)=\nabla^2h^*(\zz)\nabla\cL(\nabla h^*(\zz))$, so that $\nabla^2\phi(\zz)= \big(\nabla^2h^*(\zz)\nabla^2\cL(\nabla h^*(\zz))\nabla^2h^*(\zz)\big)_V + \nabla^3h^*(\zz)(\nabla\cL(\nabla h^*(\zz)),\cdot,\cdot)_V$.

    Since $h$ is $1/(2\alpha^2)$ smooth (on $\R^d$ and thus on $V$), $h^*$ is $2\alpha^2$ strongly convex (on $V$ and on $\R^d$).
    Using $V=\Ima(\XX\XX^\top)$ and $\nabla^2\cL\equiv \XX\XX^\top$, we have $ \big(\nabla^2h^*(\zz)\nabla^2\cL(\nabla h^*(\zz))\nabla^2h^*(\zz)\big)_V=\nabla^2h^*(\zz)_V\nabla^2\cL(\nabla h^*(\zz))_V\nabla^2h^*(\zz)_V $, and thus $\big(\nabla^2h^*(\zz)\nabla^2\cL(\nabla h^*(\zz))\nabla^2h^*(\zz)\big)_V \succeq 2\alpha^2\lambda_{\min}^+ \nabla^2h^*(\zz)_V$.

    For the other term of $\nabla^2\phi$, namely $\nabla^3h^*(\zz)(\nabla\cL(\nabla h^*(\zz)),\cdot,\cdot)_V$, we compute $\nabla^3_{ijk}h^*(\zz)=\one_{i=j=k}2\alpha^2_{i,k}\sinh(\zz_i)$, leading to:
    $\nabla^3h^*(\zz)(\nabla\cL(\nabla h^*(\zz)),\cdot,\cdot)_V=\diag(2\alpha^2\sinh(\zz)\odot(\XX\XX^\top(2\alpha^2\sinh(\zz)-\ww^\alpha)))_V$. Thus, writing $\ww_\zz=2\alpha^2_{i,k}\sinh(\zz)=\nabla h^*(\zz)$ the primal surrogate of $\zz$, we have:
    \begin{align*}
        \nabla^3h^*(\zz)(\nabla\cL(\nabla h^*(\zz)),\cdot,\cdot)_V&=\diag(2\alpha^2_{i,k}\sinh(\zz)\odot(\XX\XX^\top(\ww_\zz-\ww^\alpha_k)))_V\\
        &\succeq -\NRM{\XX\XX^\top(\ww_\zz-\ww^\alpha_k)}_{\infty}\diag(2\alpha^2_k\odot|\sinh(\zz)|)_V\\
        &\succeq -\NRM{\XX\XX^\top(\ww_\zz-\ww^\alpha_k)}_\infty\diag(2\alpha^2_k\odot\cosh(\zz))_V\\
        &= -\NRM{\XX\XX^\top(\ww_\zz-\ww^\alpha_k)}_\infty \nabla^2 \psi(\zz)\,.
    \end{align*}
    Wrapping things together, 
    \begin{align*}
        \nabla^2\phi(\zz)\succeq \big(2\alpha^2\lambda^+_{\min}-\NRM{\XX\XX^\top(\ww_\zz-\ww^\alpha)}_\infty)\nabla^2\psi(\zz)\,.
    \end{align*}
    Let $\cZ=\set{\zz\in V:\NRM{\XX\XX^\top(\ww_\zz-\ww^\alpha_k)}_\infty < \alpha^2\lambda^+_{\min}}$ that satisfies $ \set{\ww\in V: \cL(\ww_\zz)< \frac{1}{2\lambda_{\max}}(\alpha^2\lambda^+_{\min})^2}\subset\cZ$.
    $\cZ$ is an open convex set of $V$ containing $\zz^\alpha$. On $\cZ$, $\nabla^2\phi\succeq\alpha^2\lambda^+_{\min}\nabla^2\psi$, and $\psi(\zz^\alpha)=\phi(\zz^\alpha)=0$, so that for all $\zz\in\cZ$, we have $\phi(\zz)\geq \alpha^2\lambda^+_{\min}\psi(\zz)$.
    Hence, for all $\zz\in\cZ$, we have $D_{h_k}(\beta_k^\alpha,\beta_\zz)\leq D_{h^\star}(\zz,\zz^\alpha)\leq (\alpha^2\lambda^+_{\min})^{-1} \cL(\beta_\zz)$, and using the fact that $D_{h_k}$ is $\frac{1}{4B}$ strongly convex, we obtain, for $\beta_\zz=\beta_k$ (since $\zz^k\in V$): if $\cL(\beta_k)\leq \frac{1}{2\lambda_{\max}}(\alpha^2\lambda^+_{\min})^2$, we have $\NRM{\beta^\alpha_k-\beta_k}_2^2\leq (\alpha^2\lambda^+_{\min})^{-1} \cL(\beta_k)$.
\end{proof}

\begin{proposition}
\label{prop:key_identity}
As assume $\cL$ is $L_r$-relatively smooth with respect to all the $h_k$'s. Then for all $\beta$ we have the following inequality.
\begin{align*}
    \gamma_k ( \cL(\beta_{k+1}) - \cL(\beta) ) &\leq  D_{h_k}(\beta, \beta_k) - D_{h_{k+1}}(\beta, \beta_{k+1}) - (1 - \gamma_k L_r)  D_{h_k}(\beta_{k+1}, \beta_k)\\
    &\quad+ (h_{k+1}- h_{k})(\beta)  - (h_{k+1} - h_k)(\beta_{k+1})\,.
\end{align*}
\end{proposition}

\begin{proof}
For any $\beta, \beta_k, \beta_{k+1}$, the following holds (three points identity for time varying potentials, \Cref{prop:MD_varying}):
\begin{align*}
    D_{h_k}(\beta, \beta_k) - D_{h_{k+1}}(\beta, \beta_{k+1}) &= \big [ h_k(\beta) - (h_k(\beta_k) + \langle \nabla h_k(\beta_k), \beta - \beta_k \rangle) \big ]  \\
     &\qquad - \big [ h_{k+1}(\beta) - (h_{k+1}(\beta_{k+1}) + \langle \nabla h_{k+1}(\beta_{k+1}), \beta - \beta_{k+1} \rangle) \big ] \\
     &= h_k(\beta) - h_{k+1}(\beta) +  \langle \nabla h_{k+1}(\beta_{k+1}) - \nabla h_k(\beta_k), \beta - \beta_{k+1} \rangle \\
     &\qquad +  h_{k+1}(\beta_{k+1}) - \big [ h_k(\beta_k) + \langle \nabla h_k(\beta_k), \beta_{k+1} - \beta_k \rangle  \big ] \\
     &= h_k(\beta) - h_{k+1}(\beta) +  \langle \nabla h_{k+1}(\beta_{k+1}) - \nabla h_k(\beta_k), \beta - \beta_{k+1} \rangle \\
     &\qquad +  h_{k+1}(\beta_{k+1}) - h_k(\beta_{k+1}) + D_{h_k}(\beta_{k+1}, \beta_k).
\end{align*}
Rearranging and plugging in our mirror update we obtain that for all $\beta$:
\begin{align*}
    \gamma_k  \langle \nabla \cL(\beta_k), \beta_{k+1} - \beta \rangle & =  D_{h_k}(\beta, \beta_k) - D_{h_{k+1}}(\beta, \beta_{k+1})\\
    &\quad- D_{h_k}(\beta_{k+1}, \beta_k) - (h_{k+1} - h_k)(\beta_{k+1}) + (h_{k+1}- h_{k})(\beta).
\end{align*}
From the convexity of $\cL$ and its $L_r$-relative smoothness we also have that:
\begin{align*}
    \cL(\beta_{k+1}) \leq \cL(\beta) + \langle \nabla \cL(\beta_k), \beta_{k+1} - \beta \rangle + L_r D_{h_k}(\beta_{k+1}, \beta_k),
\end{align*}
Finally: 
\begin{align*}
    \gamma_k ( \cL(\beta_{k+1}) - \cL(\beta) ) &\leq  D_{h_k}(\beta, \beta_k) -  D_{h_{k+1}}(\beta, \beta_{k+1}) - (1 - \gamma_k L_r) D_{h_k}(\beta_{k+1}, \beta_k)  \\
    &\quad+ (h_{k+1}- h_{k})(\beta) - (h_{k+1} - h_k)(\beta_{k+1}).
\end{align*}
Note that in our setting, for any $\beta$, $k \mapsto h_k(\beta)$ is \textbf{increasing}. We can therefore write that:
\begin{align*}
    \gamma_k ( \cL(\beta_{k+1}) - \cL(\beta) ) \leq  D_{h_k}(\beta, \beta_k) -  D_{h_{k+1}}(\beta, \beta_{k+1}) - (1 - \gamma_k L_r) D_{h_k}(\beta_{k+1}, \beta_k) + (h_{k+1}- h_{k})(\beta).
\end{align*}
In particular, for $\beta = \beta^*$:
\begin{align*}
    \gamma_k \cL(\beta_{k+1}) &\leq  D_{h_k}(\beta^*, \beta_k) -  D_{h_{k+1}}(\beta^*, \beta_{k+1})   - (1 - \gamma_k L) D_{h_k}(\beta_{k+1}, \beta_k)  + (h_{k+1}- h_{k})(\beta^*)\\
    &\quad- (h_{k+1} - h_k)(\beta_{k+1}) \\ 
    &\leq  D_{h_k}(\beta^*, \beta_k) -  D_{h_{k+1}}(\beta^*, \beta_{k+1})   - (1 - \gamma_k L_r) D_{h_k}(\beta_{k+1}, \beta_k) + (h_{k+1}- h_{k})(\beta^*) 
\end{align*}
and in $\beta = \beta_k$:
\begin{align*}
    \gamma_k  \cL(\beta_{k+1}) &\leq  \gamma_k \cL(\beta_k)   -  D_{h_{k+1}}(\beta_k, \beta_{k+1}) - (1 - \gamma_k L_r) D_{h_k}(\beta_{k+1}, \beta_k)  + (h_{k+1}- h_{k})(\beta_k) \\
    &\quad- (h_{k+1} - h_k)(\beta_{k+1}) \\ 
     &\leq  \gamma_k \cL(\beta_k)   -  D_{h_{k+1}}(\beta_k, \beta_{k+1}) - (1 - \gamma_k L_r) D_{h_k}(\beta_{k+1}, \beta_k) + (h_{k+1}- h_{k})(\beta_k)
\end{align*}
\end{proof}

\begin{proof}[Proof of \Cref{prop:conv_quantitative}]
    We apply \Cref{prop:key_identity} for $\beta=\beta_k$, with $L_r=4BL$ (using \Cref{lemma:relative_smoothness}) and replacing $\cL$ by $\cL_{\cB_k}$, to obtain:
    \begin{align*}
    \gamma_k ( \cL_{\cB_k}(\beta_{k+1}) - \cL_{\cB_k}(\beta_k) ) &\leq   - D_{h_{k+1}}(\beta_k, \beta_{k+1}) - (1 - \gamma_k L_r)  D_{h_k}(\beta_{k+1}, \beta_k)\\
    &\quad+ (h_{k+1}- h_{k})(\beta_k)  - (h_{k+1} - h_k)(\beta_{k+1})\,, 
\end{align*}
and thus, taking the mean wrt $\cB_k$,
\begin{align*}
    \gamma_k ( \E_{\cB_k}\cL(\beta_{k+1}) - \cL(\beta_k) ) &\leq  - \E_{\cB_k}D_{h_{k+1}}(\beta_k, \beta_{k+1}) - (1 - \gamma_k L_r) \E_{\cB_k} D_{h_k}(\beta_{k+1}, \beta_k)\\
    &\quad+ \E_{\cB_k}(h_{k+1}- h_{k})(\beta_k)  - \E_{\cB_k}(h_{k+1} - h_k)(\beta_{k+1})\\
     &\leq   - (1 - \gamma_k L_r) \E_{\cB_k} D_{h_k}(\beta_{k+1}, \beta_k)\\
    &\quad+ \E_{\cB_k}(h_{k+1}- h_{k})(\beta_k)  - \E_{\cB_k}(h_{k+1} - h_k)(\beta_{k+1})\,.
\end{align*}
First, as in the proof of \Cref{prop:bregman_descent}, using the fact that $h_k$ is $\ln(1/\alpha_k)$ smooth, 
\begin{align*}
    D_{h_{k}}(\ww_{k+1},\beta_k) 
    &\geq \frac{1}{2\ln(1/\alpha_k)} \NRM{\nabla h_{k}(\ww_{k})-\gamma_k\nabla \cL_{\cB_k}(\ww_k)-\nabla h_{k}(\ww_k) +\nabla h_{k+1}(\ww_{k+1})-\nabla h_{k}(\ww_{k+1})}_2^2\\
    &\geq -\frac{1}{2\ln(1/\alpha_k)}\NRM{\nabla h_{k}(\ww_{k})-\nabla h_{k+1}(\ww_k)}_2^2+ \frac{1}{4\ln(1/\alpha_k)}\NRM{\gamma_k\nabla \cL_{\cB_k}(\ww_k)}_2^2\,,
\end{align*}
and thus
\begin{align*}
       \E D_{h_{k}}(\ww_{k+1},\beta_k)      &\geq \esp{-\frac{1}{2\ln(1/\alpha_k)}\NRM{\nabla h_{k}(\ww_{k})-\nabla h_{k+1}(\ww_k)}_2^2+ \frac{\lambda_b}{2\ln(1/\alpha_k)}\gamma_k^2 \cL_{\cB}(\ww_k)}\,.
\end{align*}
Now, we apply \cref{tech_lemma:bound_hk+1-hk} assuming that $\Vert \beta^\star \Vert_\infty, \Vert \beta_{k+1} \Vert_\infty \leq B$ (which is satisfied since we are under the assumption of \Cref{thm:conv_iterates}):
\begin{align*}
    (h_{k+1}- h_{k})(\beta_k)  - (h_{k+1} - h_k)(\beta^\star) \leq 24 B L\gamma_k^2\cL_{\cB_k}(\beta_k)\,.
\end{align*}
Using $|\nabla h_k(\ww)-\nabla h_{k+1}(\ww)|\leq 2\delta_k$ where $ \delta_k = q(\gamma_k \nabla \cL_{\cB_k}(\beta_k))$ as in \Cref{prop:bregman_descent}, we have:
\begin{align*}
    \E\NRM{\nabla h_{k}(\ww_{k})-\nabla h_{k+1}(\ww_k)}_2^2\leq 16B \gamma_k^2 \E\NRM{\nabla\cL_{\cB_k}(\beta_k)}^2\leq 32BL \gamma_k^2 \E\cL(\beta_k) \,.
\end{align*}
Wrapping everything together,
\begin{align*}
    \esp{\cL(\beta_{k+1})-\cL(\beta_k)}&\leq -(1-\gamma_k 4BL) \frac{\lambda_b}{2\ln(1/\alpha_k)}\gamma_k\E\cL(\beta_k)\\
    &\quad +\big(\gamma_k^2(1-4\gamma_kBL)24BL +\frac{32BL}{\ln(1/\alpha_k)}\big)\gamma_k^2\E\cL(\beta_k)\,.
\end{align*}
Thus, for $\gamma_k\leq \frac{c'}{LB\ln(1/(\min_i \alpha_{k,i}))}$, we have the first part of \Cref{prop:conv_quantitative}.

Using \Cref{lem:curvature}, we then have:
\begin{align*}
    \esp{\NRM{\beta_k-\beta^\star_{\alpha_k}}^2}&=\esp{\one_\set{\cL(\beta_k)\leq \frac{1}{2\lambda_{\max}}(\alpha^2\lambda^+_{\min})^2}\NRM{\beta_k-\beta^\star_{\alpha_k}}^2}\\
    &\quad+\esp{\one_\set{\cL(\beta_k)> \frac{1}{2\lambda_{\max}}(\alpha^2\lambda^+_{\min})^2}\NRM{\beta_k-\beta^\star_{\alpha_k}}^2}\\
    &\leq \esp{\one_\set{\cL(\beta_k)\leq \frac{1}{2\lambda_{\max}}(\alpha^2\lambda^+_{\min})^2}2B(\alpha^2\lambda^+_{\min})^{-1}\cL(\beta_k)}\\
    &\quad+\proba{\cL(\beta_k)> \frac{1}{2\lambda_{\max}}(\alpha^2\lambda^+_{\min})^2}\times 4B^2\\
    &\leq 2B(\alpha^2\lambda^+_{\min})^{-1}\esp{\cL(\beta_k)}\\
    &\quad+ \frac{\esp{\cL(\beta_k)}}{ \frac{1}{2\lambda_{\max}}(\alpha^2\lambda^+_{\min})^2}\times 4B^2\\
    &=2B(\alpha^2\lambda^+_{\min})^{-1}\Big(1+\frac{4B\lambda_{\max}}{\alpha^2\lambda_{\min}^+}\Big)\esp{\cL(\beta_k)}\,.
\end{align*}
\end{proof}

\section{Proof of miscellaneous results mentioned in the main text}\label{sec:app:misc}

In this section, we provide proofs for results mentioned in the main text and that are not directly directed to the proof of \cref{thm:conv_gd_appendix}.

\subsection{Proof of Proposition~\ref{prop:magnitude_IB} and the sum of the losses}\label{sec:app:gain}

    

We start by proving the following proposition, present as is in the first 9 pages of this paper.
We then continue with upper and lower bounds (of similar magnitude) on the sum of the losses.

\propmagnitudeIB*

\begin{proof}
From \cref{tech_lemma:bound_q}, for all $-1/2\leq x\leq 1/2$, it holds that $x^2\leq q(x)\leq 2x^2$.
We have, using $\NRM{\gamma_k\nabla\cL_{\cB_k}(\beta_k)}_\infty\leq 1/2$ (which holds under the stepsize assumption):
\begin{align*}
    \E\NRM{\Gain_\gamma}_1&= -\E\sum_i\ln\left(\frac{\alpha_{\infty,i}}{\alpha}\right)\\
    &= \sum_{\ell<\infty}\sum_i \E q\big(\gamma_\ell\nabla_i \cL_{\cB_\ell}(\beta_\ell)\big)\\
    & \leq  2\sum_{\ell<\infty}\sum_i \E\big(\gamma_\ell\nabla_i \cL_{\cB_\ell}(\beta_\ell)\big)^2\\
    &= \sum_{\ell<\infty} \gamma_\ell^2 \E\NRM{\nabla \cL_{\cB_\ell}(\beta_\ell)}^2_2\\
    &\leq 4\Lambda_b\sum_{\ell<\infty} \gamma_\ell^2  \E\cL_{\cB_\ell}(\beta_\ell)\,,
\end{align*}
since $\E\NRM{\nabla \cL_{\cB_\ell}(\beta_\ell)}^2_2\leq 2 \Lambda_b\cL_{\cB_\ell}(\beta_\ell)$.
For the left handside we use $q(x)\geq x^2$ for $|x|\leq 1/2$ and $\E\NRM{\nabla \cL_{\cB_k}(\beta_\ell)}^2_2\geq 2 \lambda_b\cL_{\cB_k}(\beta_\ell)$. Finally, since $\cB_\ell$ independent freom $\beta_\ell$, we have $\E\cL_{\cB_\ell}(\beta_\ell)=\E\cL(\beta_\ell)$.
\end{proof}

\begin{proposition}\label{prop:sum_losses}
    For stepsizes $\gamma_k\equiv \gamma\leq \frac{c}{LB}$ (as in \cref{thm:conv_iterates}), we have:
    \begin{equation*}
        \sum_{k\geq 0} \gamma^2 \E\cL(\beta_k) = \Theta\left( \gamma \NRM{\beta^\star}_1\ln(1/\alpha)\right)\,.
    \end{equation*}
\end{proposition}

\begin{proof}
We  first lower bound $\sum_{k<\infty}\gamma_k^2\cL_{\cB_k}(\beta_k)$. We have the following equality, that holds for any $k$:
\begin{align*}
        D_{h_{k+1}}(\beta^\star,\beta_{k+1})&=D_{h_k}(\beta^\star,\beta_k)-2\gamma \cL_{\cB_k}(\beta_k)+ D_{h_{k+1}}(\beta_k,\beta_{k+1})\\
        &\quad + \big(h_k-h_{k+1}\big)(\beta_k) - \big(h_k-h_{k+1}\big)(\beta^\star)\,,
\end{align*}
leading to, by summing for $k\in\N$:
\begin{equation*}
    \sum_{k<\infty}2\gamma\cL_{\cB_k}(\beta_k)= D_{h_0}(\beta^\star,\beta_0)-\lim_{k\to\infty}D_{h_k}(\beta^\star,\beta_k) + \sum_{k<\infty}D_{h_{k+1}}(\beta_k,\beta_{k+1}) + \sum_{k<\infty}\big(h_k-h_{k+1}\big)(\beta_k)-\big(h_k-h_{k+1}\big)(\beta^\star)\,.
\end{equation*}
First, since $h_k\to h_\infty$, $\beta_k\to\beta_\infty$, we have $\lim_{k\to\infty}D_{h_k}(\beta^\star,\beta_k)=0$.
Then, $D_{h_{k+1}}(\beta_k,\beta_{k+1})\geq0$.
Finally, $|\big(h_k-h_{k+1}\big)(\beta_k)-\big(h_k-h_{k+1}\big)(\beta^\star)| \leq 16 B L_2\gamma^2\cL_{\cB_k}(\beta_k) $.
Hence :
\begin{equation*}
        \sum_{k<\infty}2\gamma(1+16\gamma BL_2)\cL_{\cB_k}(\beta_k)\geq D_{h_0}(\beta^\star,\beta_0)\,,
\end{equation*}
and thus $\sum_{k<\infty}\gamma\cL_{\cB_k}(\beta_k)\geq D_{h_0}(\beta^\star,\beta_0)/4$ for $\gamma\leq c/(BL)$ (with $c\geq 16$). This gives the RHS inequality. The LHS is a direct consequence of bounds proved in previous subsections.

Hence, we have that 
\begin{equation*}
    \gamma^2\sum_k\cL(\beta_k)=\Theta\left(\gamma D_{h_0}(\beta^\star,\beta_0)\right)\,.
\end{equation*}
Noting that $D_{h_0}(\beta^\star,\beta_0)=h_0(\beta^\star) =\Theta\big(\ln(1/\alpha)\NRM{\beta^\star}_1\big)$ concludes the proof.
\end{proof}

\subsection{$\tilde \beta_0$ is negligible}

In the following proposition we show that $\Tilde{\beta}_0$ is close to $\mathbf{0}$ and therefore one should think of the implicit regularization problem as $\beta^\star_\infty = \argmin_{\beta^\star \in S} \psi_{\alpha_\infty}(\beta^\star)$

\begin{proposition}
    \label{app:tilde_beta0}
    Under the assumptions of \Cref{thm:conv_iterates},
    \begin{align*}
        \vert \tilde{\beta}_0 \vert \leq \alpha^2,
    \end{align*}
    where the inequality must be understood coordinate-wise.
\end{proposition}

\begin{proof}
\begin{align*}
    \vert \tilde{\beta}_0 \vert &= \frac{1}{2}\vert \alpha_+^2 - \alpha_-^2 \vert \\
    &= \frac{1}{2} \alpha^2 \big \vert \exp( - \sum_k q_+(\gamma_k \nabla \cL(\beta_k)) -  \exp( - \sum_k q_-(\gamma_k \nabla \cL(\beta_k)) \big \vert \\
    &\leq \alpha^2,
\end{align*}
where the inequality is because $q_+(\gamma_k \nabla \cL(\beta_k) )\geq 0$, $q_-(\gamma_k \nabla \cL(\beta_k) ) \geq 0$ for all $k$.


\end{proof}


\subsection{Impact of stochasticity and linear scaling rule}\label{sec:app:linear_scaling_rule}

\begin{proposition}\label{prop:lambda_b_sum}
    With probability $1-2ne^{-d/16}-3/n^2$ over the $x_i\sim_{\rm iid}\cN(0,\sigma^2I_d)$, $ c_1 \frac{d\sigma^2}{b}(1+o(1))\leq \lambda_b\leq \Lambda_b\leq c_2 \frac{d\sigma^2}{b}(1+o(1))$\,,
\end{proposition}
so that under these assumptions, 
\begin{equation*}
    \sum_k \gamma_k\E\cL(\beta_k)=\Theta\left(\frac{\gamma}{b}\sigma^2\NRM{\beta^\star}_1\ln(1/\alpha)\right)\,.
\end{equation*}
\begin{proof}
The bound on $\lambda_b,\Lambda_b$ is a direct consequence of the concentration bound provided in \cref{lem:concentration3bis}.
\end{proof}

\subsection{(Stochastic) gradients at the initialisation}\label{sec:app:stoch_grad}
 \label{sec:app:stoch_grad}

To  understand the behaviour  and the effects of the stochasticity and the stepsize on the shape  of $\rm Gain _\gamma$, we analyse a noiseless sparse recovery problem under the following standard assumption~\ref{ass:sparse_reg_setup}  \cite{candestao} and as common in the sparse recovery literature, we make the following assumption~\ref{ass:RIP1} on the inputs.
 
\begin{assumption}\label{ass:sparse_reg_setup}
There exists an $s$-sparse ground truth vector $\beta^\star_{\rm sparse}$ where $s$ verifies $n=\Omega(s\ln(d))$, such that $y_i=\langle\beta^\star_{\rm sparse}, x_i\rangle$ for all $i\in[n]$.
\end{assumption}
\begin{assumption}\label{ass:RIP1}
There exists $\delta, c_1, c_2 > 0 $ such that for all $s$-sparse vectors $\beta$, there exists $\varepsilon \in \R^d$ such that $(X^\top X) \beta = \beta + \varepsilon$  where $\Vert \varepsilon \Vert_\infty \leq \delta \Vert \beta \Vert_2$ and 
$ c_1 \Vert \beta \Vert_2^2  \mathbf{1} \leq \frac{1}{n} \sum_i  x_i^2 \langle x_i, \beta \rangle^2 \leq c_2 \Vert \beta \Vert_2^2 \mathbf{1}$.
\end{assumption}
The first part of \cref{ass:RIP1} closely resembles the classical restricted isometry property (RIP) and is relevant for GD while the second part is relevant for SGD.
Such an assumption is not restrictive and holds with high probability for Gaussian inputs $\mathcal{N}(0, \sigma^2 I _d)$ (see \cref{lem:concentration_ripsgd} in Appendix).

Based on the claim above, we analyse the shape of the (stochastic) gradient at initialisation. 
For GD and SGD, it respectly writes, where $g_0=\nabla \cL_{i_0}(\beta_0)^2$, $i_0\sim \rm{Unif}([n])$:
\begin{align*}
    \nabla \cL(\beta_0)^2 = [X^\top X \beta^\star]^2\,,\ \ \mathbb{E}_{i_0} [g_0] = \frac{1}{n} \sum_i  x_i^2 \langle x_i, \beta^\star \rangle^2.
\end{align*}
The following lemma then shows that while the initial stochastic gradients of SGD are homogeneous, it is not the case for that of GD. 

\begin{proposition}
\label{lemma:shape_SG_init}
Under \cref{ass:RIP1}, the squared full batch gradient and the expected stochastic gradient at initialisation satisfy, for some $\eps$ verifying $\NRM{\eps}_\infty <\!< \NRM{\beta^\star_\sparse}_\infty^2$:
\begin{align}
    \nabla \cL &(\beta_0  )^2 = (\beta^\star_{\rm sparse})^2  + \varepsilon \,, \label{eq:nabla0_GD} \\
    \mathbb{E}_{i_0} [\nabla &\cL_{i_0}(\beta_0)^2] = \Theta \Big(  \Vert \beta^\star \Vert_2^2 \mathbf{1} \Big)\,. \label{eq:nabla0_SGD}
\end{align}
\end{proposition}

\begin{proof}[Proof of \cref{lemma:shape_SG_init}]
Under \cref{ass:RIP1}, we have using:
\begin{align*}
    \nabla\cL(\beta_0)^2&=(X^\top X \beta^\star_\sparse)\\
    &= (\beta^\star_\sparse + \eps)^2\\
    &={\beta_\sparse^\star}^2 + \eps^2 +2\eps \beta_\sparse^\star\,.
\end{align*}
We have $\NRM{\eps^2 +2\eps \beta_\sparse^\star}_\infty\leq \NRM{\eps}_\infty^2 +2\NRM{\eps}_\infty \NRM{\beta_\sparse^\star}_\infty$, and we conclude by using $\NRM{\eps}_\infty\leq \delta\NRM{\beta^\star_\sparse}_2$.

Then,
\begin{align*}
    \mathbb{E}_{i \sim \rm{Unif}([n])} [\nabla \cL_{i}(\beta_0)^2]&=\frac{1}{n}x_i^2\langle x_i,\beta^\star_\sparse\rangle \,,
\end{align*}
and we conclude using \cref{ass:RIP1}.
    
\end{proof}

\begin{proof}[Proof of \cref{lemma:shape_SG_init_uncentered}]
The proof proceeds as that of \cref{lemma:shape_SG_init}.
\end{proof}

\subsection{Convergence of $\alpha_\infty$ and $\tilde\beta_0$ for $\gamma\to0$}

\begin{proposition}\label{prop:conv_beta_alpha}
Let $\tilde\beta_0(\gamma),\alpha_\infty(\gamma)$ be as defined in Theorem~\ref{thm:implicit_bias}, for constant stepsizes $\gamma_k\equiv\gamma$.
We have:
\begin{equation*}
    \tilde\beta_0(\gamma)\to 0\,,\quad \balpha_\infty\to \alpha\one\,,
\end{equation*}
when $\gamma\to 0$.
\end{proposition}
\begin{proof}
We have, as proved previoulsy, that
\begin{align*}
    \NRM{\sum_{k}\gamma^2\nabla\cL_{\cB_k}(\beta_k)^2}_1&\leq\sum_{k}\gamma^2\NRM{\nabla\cL_{\cB_k}(\beta_k)^2}_1\\
    &=\sum_{k}\gamma^2\NRM{\nabla\cL_{\cB_k}(\beta_k)}_2^2\\
    &\leq 2L\gamma^2\sum_k\cL_{\cB_k}(\beta_k)\\
    &\leq 2L\gamma D_{h_0}(\beta^\star,\beta_0)\,,
\end{align*}
for $\gamma\leq \frac{c}{BL}$. Thus, $\sum_{k}\gamma^2\nabla\cL_{\cB_k}(\beta_k)^2\to0$ as $\gamma\to 0$ (note that $\beta_k$ implicitly depends on $\gamma$, so that this result is not immediate).

Then, for $\gamma\leq \frac{c}{LB}$,
\begin{equation*}
    \NRM{\ln(\balpha_\infty^2/\alpha^2)}_1\leq \sum_k \NRM{q(\gamma\cL(\beta_k)}_1\leq 2\sum_{k}\gamma^2\NRM{\nabla\cL_{\cB_k}(\beta_k)^2}_1 \,,
\end{equation*}
which tends to $0$ as $\gamma\to 0$.
Similarly, $\NRM{\ln(\balpha_{+,\infty}^2/\alpha^2)}_1\to0$ and $\NRM{\ln(\balpha_{-,\infty}^2/\alpha^2)}_1\to 0$ as $\gamma\to0$, leading to $\tilde\beta_0(\gamma)\to 0$ as $\gamma\to 0$.

\end{proof}

\section{Technical lemmas}\label{sec:app:technical}

In this section we present a few technical lemmas, used and referred to throughout the proof of \cref{thm:conv_gd}.

\begin{lemma}\label{lem:argsh}
Let $\alpha_+,\alpha_->0$ and $x\in\R$, and $\beta=\alpha_+^2e^x-\alpha_-^2e^{-x}$. We have:
\begin{equation*}
    \argsinh\Big(\frac{\beta}{2\alpha_+\alpha_-}\Big)=x+\ln\Big(\frac{\alpha_+}{\alpha_-}\Big)=x+\argsinh\Big(\frac{\alpha_+^2-\alpha_-^2}{2\alpha_+\alpha_-}\Big)\,.
\end{equation*}
\end{lemma}
\begin{proof}First,
\begin{align*}
        \frac{\beta}{2\alpha_+\alpha_-}&= \half\left(\frac{\alpha_+}{\alpha^-}e^x-\big(\frac{\alpha_+}{\alpha^-}\big)^{-1}e^{-x}\right)\\
        &= \frac{e^{x+\ln(\alpha_+/\alpha_-)}-e^{-x-\ln(\alpha_+/\alpha_-)}}{2}\\
        &=\sinh(x+\ln(\alpha_+/\alpha_-))\,,
\end{align*}
hence the result by taking the $\argsinh$ of both sides. Note also that we have $\ln(\alpha_+/\alpha_-)=\argsinh(\frac{\alpha_+^2-\alpha_-^2}{2\alpha_+\alpha_-})$.
\end{proof}

\begin{lemma}
\label{tech_lemma:bound_q}
If 
 $|x|\leq 1/2$ then $x^2 \leq q(x)\leq 2 x^2$
\end{lemma}

\begin{lemma}\label{lemma:relative_smoothness}
    On the $\ell_\infty$ ball of radius $B$, the quadratic loss function $\beta \mapsto \cL(\beta)$ is $4\lambda_{\max}\max(B,\alpha^2)$-relatively smooth w.r.t all the $h_k$'s.
\end{lemma}
\begin{proof}
    We have:
    \begin{equation*}
        \nabla^2 h_k(\beta)=\diag\Big(\frac{1}{2\sqrt{\alpha_k^4+\beta^2}}\Big)\succeq \diag\Big(\frac{1}{2\sqrt{\alpha^4+\beta^2}}\Big) \,,
    \end{equation*}
    since $\alpha_k\leq \alpha$ component-wise. Thus, $\nabla^2h_k(\beta)\succeq\half\min\big(\min_{1\leq i\leq d}\frac{1}{2|\beta_i|}, \frac{1}{2\alpha^2}\big)I_d= \frac{1}{\max(4\NRM{\beta}_\infty,4\alpha^2)}I_d$, and $h_k$ is $\frac{1}{\max(4B,4\alpha^2)}$-strongly convex on the $\ell^\infty$ norm of radius $B$. Since $\cL$ is $\lambda_{\max}$-smooth over $\R^d$, we have our result. 
\end{proof}

\begin{lemma}
\label{tech_lemma:bound_hk+1-hk}
For $k \geq 0$ and for all $\beta \in \R^d$:
\begin{align*}
    |h_{k+1}(\ww)-h_{k}(\ww)| \leq 8 L_2 \gamma_k^2 \cL_{\cB_k}(\beta_k) \NRM{\ww}_\infty.
\end{align*}
\end{lemma}

\begin{proof}
    We have $\alph_{+,k+1}^2=\alph_{+,k}^2e^{-\ddelta_{+,k}}$ and $\alph_{-,k+1}^2=\alph_{-,k}^2 e^{-\ddelta_{-,k}}$, for $\ddelta_{+,k}= \tilde{q}(\gamma_k\nabla\cL_{\cB_k}(\ww_k))$ and $\ddelta_{-,k}= \tilde{q}(-\gamma_k\nabla\cL_{\cB_k}(\ww_k))$. And $\alpha_{k+1} = \alpha_k \exp(- \delta_k)$ where $\delta_k \coloneqq \delta_{+, k} + \delta_{-, k} = q(\gamma_k \nabla \cL_{\cB_k}(\beta_k) )$.

To prove the result we will use that for $\ww\in\R^d$, we have $|(h_{k+1}-h_{k})(\ww)|\leq \sum_{i=1}^d\int_0^{|\ww_i|}|\nabla_i h_{k+1}(x)-\nabla_i h_{k}(x)|\dd x$. 

First, using that$|\argsinh(a)-\argsinh(b)|\leq |\ln(a/b)|$ for $ab>0$. We have that 
\begin{align*}
    \Big|\argsinh\Big(\frac{x}{\alph_{k+1}^2}\Big)-\argsinh\Big(\frac{x}{\alph_{k}^2}\Big)\Big|&\leq\ln\left(\frac{\alph_{k}^2}{\alph_{k+1}^2}\right)\\
    &= \ddelta_k\,,
\end{align*}
since $\delta_k \geq 0$ due to our stepsize condition.

We now prove that $|\ph_{k+1}-\ph_k|\leq \frac{|\ddelta_{+,k}-\ddelta_{-,k}|}{2}$.
We have $\ph_k = \argsinh \big (  \frac{\alph_{+, k}^2 - \alph_{-, k}^2 }{2 \alph_{+, k}  \alph_{-, k}  } \big )$ and hence,
\begin{align*}
    |\ph_{k+1}-\ph_k|&=\left|\argsinh \Big (  \frac{\alph_{+, k}^2 - \alph_{-, k}^2 }{2 \alph_{+, k}  \alph_{-, k}  } \Big ) - \argsinh \Big (  \frac{\alph_{+, k+1}^2 - \alph_{-, k+1}^2 }{2 \alph_{+, k+1}  \alph_{-, k+1}  } \Big ) \right|\,.
\end{align*}
Then, assuming that $\alph_{+,k,i}\geq\alph_{-,k,i}$, we have:
\begin{align*}
    \frac{\alph_{+, k+1,i}^2 - \alph_{-, k+1,i}^2 }{2 \alph_{+, k+1,i}  \alph_{-, k+1,i}  }& = e^{\ddelta_{k,i}/2}\frac{\alph_{+, k,i}^2e^{-\ddelta_{+,k,i}} - \alph_{-, k,i}^2e^{-\ddelta_{-,k,i}} }{2 \alph_{+, k,i}  \alph_{-, k,i}  }\\
    &\left\{\begin{aligned}
        &\leq \left\{\begin{aligned}
            & e^{\frac{\ddelta_{+,k,i}-\ddelta_{-,k,i}}{2}}  \frac{\alph_{+, k,i}^2 - \alph_{-, k,i}^2 }{2 \alph_{+, k,i}  \alph_{-, k,i}  }\quad \text{if}\quad \ddelta_{+,k,i}\geq \ddelta_{-,k,i}\\
            & e^{\frac{\ddelta_{-,k,i}-\ddelta_{+,k,i}}{2}}  \frac{\alph_{+, k,i}^2 - \alph_{-, k,i}^2 }{2 \alph_{+, k,i}  \alph_{-, k,i}  }\quad \text{if}\quad \ddelta_{-,k,i}\geq \ddelta_{+,k,i}
        \end{aligned}\right.\\
        &\geq \left\{\begin{aligned}
            & e^{-\frac{\ddelta_{+,k,i}-\ddelta_{-,k,i}}{2}}  \frac{\alph_{+, k,i}^2 - \alph_{-, k,i}^2 }{2 \alph_{+, k,i}  \alph_{-, k,i}  }\quad \text{if}\quad \ddelta_{+,k,i}\geq \ddelta_{-,k,i}\\
            & e^{-\frac{\ddelta_{-,k,i}-\ddelta_{+,k,i}}{2}}  \frac{\alph_{+, k,i}^2 - \alph_{-, k,i}^2 }{2 \alph_{+, k,i}  \alph_{-, k,i}  }\quad \text{if}\quad \ddelta_{-,k,i}\geq \ddelta_{+,k,i}
        \end{aligned}\right.
    \end{aligned}\right.\,.
\end{align*}
We thus have $\frac{\alph_{+, k+1,i}^2 - \alph_{-, k+1,i}^2 }{2 \alph_{+, k+1,i}  \alph_{-, k+1,i}  } \in \left[e^{-\frac{\left|\ddelta_{+,k,i}-\ddelta_{-,k,i}\right|}{2}},e^{\frac{\left|\ddelta_{+,k,i}-\ddelta_{-,k,i}\right|}{2}}\right] \times \frac{\alph_{+, k,i}^2 - \alph_{-, k,i}^2 }{2 \alph_{+, k,i}  \alph_{-, k,i}  }$, and this holds similarly if $\alph_{+,k,i}\leq\alph_{-,k,i}$.
Then, using $|\argsinh(a)-\argsinh(b)|\leq |\ln(a/b)|$ we obtain that:
\begin{align*}
    |\ph_{k+1}-\ph_k|&=\left|\argsinh \Big (  \frac{\alph_{+, k}^2 - \alph_{-, k}^2 }{2 \alph_{+, k}  \alph_{-, k}  } \Big ) - \argsinh \Big (  \frac{\alph_{+, k+1}^2 - \alph_{-, k+1}^2 }{2 \alph_{+, k+1}  \alph_{-, k+1}  } \Big ) \right|\\
    &\leq \frac{|\ddelta_{+,k}-\ddelta_{-,k}|}{2}\,.
\end{align*}
Wrapping things up, we have:

\begin{equation*}
    |\nabla h_k(\ww)-\nabla h_{k+1}(\ww)|\leq \delta_k + \frac{|\ddelta_{+,k}-\ddelta_{-,k}|}{2} \leq 2\ddelta_k\,,
\end{equation*}
This leads to the following bound:
\begin{align*}
    |h_{k+1}(\ww)-h_{k}(\ww)| &\leq \langle|2\ddelta_k|,|\ww|\rangle \\
    &\leq 2\NRM{\ddelta_k}_1\NRM{\ww}_\infty.
\end{align*}

Recall that $\delta_k = q(\gamma_k \nabla \cL_{\cB_k}(\beta_k)$, hence from \cref{tech_lemma:bound_q} if $\gamma_k \Vert \nabla \cL_{\cB_k}(\beta_k) \Vert_\infty \leq 1 /2$, we get that 
$$\Vert \delta_k \Vert_1 \leq 2 \gamma_k^2 \Vert \nabla \cL_{\cB_k}(\beta_k) \Vert_2^2 \leq  4 L_2 \gamma_k^2 \cL_{\cB_k}(\beta_k).$$
Putting things together we obtain that 
\begin{align*}
    |h_{k+1}(\ww)-h_{k}(\ww)| &\leq \langle|2\ddelta_k|,|\ww|\rangle \\
    &\leq 8 L_2\gamma_k^2 \cL_{\cB_k}(\beta_k) \NRM{\ww}_\infty.
\end{align*}

\end{proof}

\section{Concentration inequalities for matrices}\label{app:concentration}

In this last section of the appendix, we provide and prove several concentration bounds for random vectors and matrices, with (possibly uncentered) isotropic gaussian inputs.
These inequalities can easily be generalized to subgaussian random variables via more refined concentration bounds, and to non-isotropic subgaussian random variables \cite{pmlr-v134-even21a}, leading to a dependence on an effective dimension and on the subgaussian matrix $\Sigma$.
We present these lemmas before proving them in a row.

The next two lemmas closely ressemble the RIP assumption, for centered and then for uncentered gaussians.

\begin{lemma}\label{lem:concentration4}
Let $x_1,\ldots,x_n\in\R^d$ be \emph{i.i.d.}~random variables of law $\cN(0,I_d)$ and  $H=\frac{1}{n}\sum_{i=1}^nx_ix_i^\top$.
Then, denoting by $\cC$ the set of all $s$-sparse vector $\beta\in\R^d$ satisfying $\NRM{\beta}_2\leq 1$, there exist $C_4,C_5>0$ such that for any $\eps>0$, if $n\geq C_4s\ln(d)\eps^{-2}$,
    \begin{equation*}
        \P\left(\sup_{\beta\in\cS}\NRM{H\beta-\beta}_\infty \geq \eps\right)\leq e^{-C_5 n}\,.
    \end{equation*}
\end{lemma}

\begin{lemma}\label{lemma:RIP_uncentered}
Let $x_1,\ldots,x_n\in\R^d$ be \emph{i.i.d.}~random variables of law $\cN(\mu,\sigma^2I_d)$ and  $H=\frac{1}{n}\sum_{i=1}^nx_ix_i^\top$.
Then, denoting by $\cC$ the set of all $s$-sparse vector $\beta\in\R^d$ satisfying $\NRM{\beta}_2\leq 1$, there exist $C_4,C_5>0$ such that for any $\eps>0$, if $n\geq C_4s\ln(d)\eps^{-2}$,
    \begin{equation*}
        \P\left(\sup_{\beta\in\cS}\NRM{H\beta-\mu\langle\mu,\beta\rangle-\sigma^2\beta}_\infty \geq \eps\right)\leq e^{-C_5 n}\,.
    \end{equation*}
\end{lemma}

We then provide two lemmas that estimate the mean Hessian of SGD.

\begin{lemma}\label{lem:concentration_ripsgd}
    Let $x_1,\ldots,x_n$ be i.i.d. random variables of law $\cN(0,I_d)$.
    Then, there exist $c_1,c_2>0$ such that with probability $1- \frac{1}{d^2}$ and if $n=\Omega(s^{5/4}\ln(d))$, we have for all $s$-sparse vectors $\beta$:
    \begin{equation*}
        c_1\NRM{\beta}_2^2\one\leq\frac{1}{n}\sum_{i=1}^nx_i^2\langle x_i,\beta\rangle^2 \leq c_2\NRM{\beta}_2^2\one\,,
    \end{equation*}
    where the inequality is meant component-wise.
\end{lemma}

\begin{lemma}\label{lem:concentration_ripsgd_uncentered}
    Let $x_1,\ldots,x_n$ be i.i.d. random variables of law $\cN(\mu,\sigma^2I_d)$.
    Then, there exist $c_0,c_1,c_2>0$ such that with probability $1- \frac{c_0}{d^2}-\frac{1}{nd}$ and if $n=\Omega(s^{5/4}\ln(d))$ and $\mu\geq 4\sigma\sqrt{\ln(d)}\one$, we have for all $s$-sparse vectors $\beta$:
\begin{equation*}
    \frac{\mu^2}{2} \left(\langle\mu,\beta\rangle^2+\frac{1}{2}\sigma^2\NRM{\beta}_2^2 \right)\leq\frac{1}{n}\sum_i x_i^2\langle x_i,\beta\rangle^2\leq4\mu^2\left(\langle\mu,\beta\rangle^2+2\sigma^2\NRM{\beta}_2^2 \right)\,.
\end{equation*}    
where the inequality is meant component-wise.
\end{lemma}

Finally, next two lemmas are used to estimate $\lambda_b,\Lambda_b$ in our paper.

\begin{lemma}\label{lem:concentration2}
Let $x_1,\ldots,x_n\in\R^d$ be \emph{i.i.d.}~random variables of law $\cN(\mu\one,\sigma^2I_d)$. Let  $H=\frac{1}{n}\sum_{i=1}^nx_ix_i^\top$ and $\tilde H=\frac{1}{n}\sum_{i=1}^n\NRM{x_i}^2 x_ix_i^\top$.
There exist numerical constants $C_2,C_3>0$ such that
    \begin{equation*}
        \P \Big( C_2\big(\mu^2+\sigma^2)d H \preceq  \tilde H \preceq  C_3\big(\mu^2+\sigma^2)d H\Big)\geq 1- 2ne^{-d/16}\,.
    \end{equation*}
\end{lemma}

\begin{lemma}\label{lem:concentration3bis}
    Let $x_1,\ldots,x_n\in\R^d$ be \emph{i.i.d.}~random variables of law $\cN(\mu \one, \sigma^2 I_d)$ for some $\mu\in\R$. Let  $H=\frac{1}{n}\sum_{i=1}^nx_ix_i^\top$ and for $1\leq b\leq n$ let $\tilde H_b=\E_\cB\left[\left(\frac{1}{b}\sum_{i\in\cB}x_ix_i^\top\right)^2\right]$ where $\cB\subset[n]$ is sampled uniformly at random in $\set{\cB\subset[n]\,\text{s.t.}\,|\cB|=b}$.
    With probability $1-2ne^{-d/16}-3/n^2$, we have, for some numerical constants $c_1,c_2,c_3,C>0$:
    \begin{equation*}
        \left( c_1\frac{d(\mu^2+\sigma^2)}{b} - c_2\frac{(\sigma^2+\mu^2)\ln(n)}{\sqrt{d}} - c_3\frac{\mu^2 d}{n}\right)H \preceq \tilde H_b \preceq C\left( \frac{d(\mu^2+\sigma^2)}{b} + \frac{(\sigma^2+\mu^2)\ln(n)}{\sqrt{d}} + \mu^2 d\right)
    \end{equation*}
    \end{lemma}

\begin{proof}[Proof of Lemma~\ref{lem:concentration4}]
For $j\in[d]$, we have:
\begin{align*}
    (H\beta)_j &= \frac{1}{n}\sum_{i=1}^n x_{ij}\langle x_i,\beta\rangle\\
    &=\frac{1}{n}\sum_{i=1}^n\sum_{j'=1}^d x_{ij}x_{ij'}\beta_{j'}\\
    &=\frac{1}{n}\sum_{i=1}^n x_{ij}^2\beta_{j}+\frac{1}{n}\sum_{i=1}^n\sum_{j'\ne j} x_{ij}x_{ij'}\beta_{j'} \\
    &=\frac{\beta_j}{n}\sum_{i=1}^n x_{ij}^2+\frac{1}{n}\sum_{i=1}^nx_{ij}\sum_{j'\ne j} x_{ij'}\beta_{j'} \,.
\end{align*}
We thus notice that $\esp{H\beta}=\beta$, and 
\begin{equation*}
    (H\beta)_j= \beta_j + \frac{\beta_j}{n}\sum_{i=1}^n (x_{ij}^2-1)+\frac{1}{n}\sum_{i=1}^nz_i\,,
\end{equation*}
where $z_i=x_{ij}\sum_{j'\ne j} x_{ij'}\beta_{j'}$, and $\sum_{j'\ne j} x_{ij'}\beta_{j'}\sim\cN(0,\NRM{\beta}^2-\beta_j^2)$ and $\NRM{\beta}^2-\beta_j^2\leq 1$. Hence, $z_j+x_{ij}^2-1$ is a centered subexponential random variables (with a subexponential parameter of order 1). Thus, for $t\leq 1$:
\begin{equation*}
    \P\left(\left| \frac{\beta_j}{n}\sum_{i=1}^n (x_{ij}^2-1)+\frac{1}{n}\sum_{i=1}^nz_i \right|\geq t\right)\leq 2e^{-cnt^2}\,.
\end{equation*}
Hence, using an $\eps$-net of $\cC=\set{\beta\in\R^d:\, \NRM{\beta}_2\leq 1\,,\NRM{\beta}_0}$ (of cardinality less than $d^s\times(C/\eps)^s$, and for $\eps$ of order 1), we have, using the  classical $\eps$-net trick explained in [Chapt. 9, \cite{vershynin_2018} or [App. C, \citet{pmlr-v134-even21a}]:
\begin{equation*}
    \P\left(\sup_{\beta\in\cC,\,j\in[d]}\left| (H\beta)_j-\beta_j \right|\geq t\right)\leq d\times d^s(C/\eps)^s\times 2e^{-cnt^2}=\exp\left(-c\ln(2)nt^2 + (s+1)\ln(d) + s \ln(C/\eps)\right)\,.
\end{equation*}
Consequently, for $t=\eps$ and if $n\geq C_4 s\ln(d)/\eps^2$, we have:
\begin{equation*}
    \P\left(\sup_{\beta\in\cC,\,j\in[d]}\left| (H\beta)_j-\beta_j \right|\geq t\right)\leq\exp\left(-C_5nt^2 \right)\,.
\end{equation*}

\end{proof}

\begin{proof}[Proof of \cref{lemma:RIP_uncentered}]
We write $x_i=\sigma z_i + \mu$ where $z_i\sim \cN(0,I_d)$. We have:
\begin{align*}
    X^\top X\beta & = \frac{1}{n}\sum_{i=1}^n (\mu+\sigma z_i) \langle \mu+\sigma z_i,\beta\rangle\\
    &= \mu \langle \mu,\beta\rangle + \frac{\sigma^2}{n}\sum_{i=1}^n z_i\langle z_i,\beta\rangle + \frac{\sigma}{n}\sum_{i=1}^n \mu\langle z_i,\beta\rangle + \frac{\sigma}{n}\sum_{i=1}^n z_i \langle \mu,\beta\rangle\\
    &= \mu \langle \mu,\beta\rangle + \frac{\sigma^2}{n}\sum_{i=1}^n z_i\langle z_i,\beta\rangle + \sigma \mu\langle \frac{1}{n}\sum_{i=1}^nz_i,\beta\rangle + \frac{\sigma\langle \mu,\beta\rangle}{n}\sum_{i=1}^n z_i \,.
\end{align*}
The first term is deterministic and is to be kept.
The second one is of order $\sigma^2\beta$ whp using Lemma~\ref{lem:concentration4}.
Then, $\frac{1}{n}\sum_{i=1}^nz_i\sim\cN(0,I_d/n)$, so that 
\begin{equation*}
    \P\left(\left|\langle \frac{1}{n}\sum_{i=1}^nz_i,\beta\rangle\right|\geq t\right)\leq 2e^{-nt^2/(2\NRM{\beta}_2^2)}\,,
\end{equation*}
and 
\begin{equation*}
    \P\left(\left|\frac{1}{n}\sum_{i=1}^n z_{ij}\right|\geq t\right)\leq 2e^{-nt^2/2}\,.
\end{equation*}
Hence,
\begin{equation*}
    \P\left( \sup_{\beta\in\cC} \left\|\frac{1}{n}\sum_{i=1}^n z_{ij}\right\|_\infty\geq t\,,\, \sup_{\beta\in\cC} \left|\langle \frac{1}{n}\sum_{i=1}^nz_i,\beta\rangle\right|\geq t\right)\leq 4e^{cs\ln(d)}e^{-nt^2/2}\,.
\end{equation*}
Thus, with probability $1-Ce^{-n\eps^2}$ and under the assumptions of \cref{lem:concentration4}, we have $\NRM{X^\top X\beta - \mu \langle \mu,\beta\rangle - \sigma^2\beta}_\infty\leq \eps$
\end{proof}

\begin{proof}[Proof of \cref{lem:concentration_ripsgd}]To ease notations, we assume that $\sigma=1$.
We remind (\citet{odonell2021boolean}, Chapter~9 and \citet{tao_concentration_course}) that for \emph{i.i.d.} real random variables $a_1,\ldots,a_n$ that satisfy a tail inequality of the form
\begin{equation}\label{eq:tail}
    \P\big( |a_1-\E a_1|\geq t\big) \leq Ce^{-ct^p}\,,
\end{equation}
for $p<1$, then for all $\eps>0$ there exists $C',c'$ such that for all $t$,
\begin{equation*}
    \P\big( |\frac{1}{n}\sum_{i=1}^n a_i-\E a_1|\geq t\big) \leq C'e^{-c'nt^{p-\eps}}\,.
\end{equation*}
We now expand $\frac{1}{n}\sum_{i=1}^nx_i^2\langle x_i,\beta\rangle^2$:
\begin{align*}
    \frac{1}{n}\sum_{i=1}^nx_i^2\langle x_i,\beta\rangle^2&=\frac{1}{n}\sum_{i\in[n],k,\ell\in[d]} x_i^2 x_{ik}x_{i\ell}\beta_k\beta_\ell\\
    &=\frac{1}{n}\sum_{i\in[n],k\in[d]} x_i^2 x_{ik}^2\beta_k^2 + \frac{1}{n}\sum_{i\in[n],k\ne\ell\in[d]} x_i^2 x_{ik}x_{i\ell}\beta_k\beta_\ell\,.
\end{align*}
Thus, for $j\in[d]$, 
\begin{align*}
    \left(\frac{1}{n}\sum_{i=1}^nx_i^2\langle x_i,\beta\rangle^2\right)_j & =\sum_{k\in[d]}\frac{\beta_k^2}{n}\sum_{i\in[n]} x_{ij}^2 x_{ik}^2 + \sum_{k\ne\ell\in[d]}\frac{\beta_k\beta_\ell}{n}\sum_{i\in[n]} x_{ij}^2 x_{ik}x_{i\ell}\,.
\end{align*}
We notice that for all indices, all $x_{ij}^2 x_{ik}x_{i\ell}$ and $x_{ij}^2 x_{ik}^2$ satisfy the tail inequality \cref{eq:tail} for $C=8$, $c=1/2$ and $p=1/2$, so that for $\eps=1/4$:
\begin{equation*}
    \P\big( |\frac{1}{n}\sum_{i=1}^n x_{ij}^2 x_{ik}x_{i\ell}|\geq t\big) \leq C'e^{-c'nt^{1/4}}\quad,\quad \P\big( |\frac{1}{n}\sum_{i=1}^n x_{ij}^2 x_{ik}^2 - \esp{x_{ij}^2x_{ik}^2}|\geq t\big) \leq C'e^{-c'nt^{1/4}}\,.
\end{equation*}
For $j\ne k$, we have $\esp{x_{ij}^2x_{ik}^2}=1$ while for $j=k$, we have $\esp{x_{ij}^2x_{ik}^2}=\esp{x_{ij}^4}=3$.
Hence,
\begin{equation*}
    \P\left(\exists j,k\ne\ell\,,\,|\frac{1}{n}\sum_{i=1}^n x_{ij}^2 x_{ik}x_{i\ell}|\geq t\,,\, |\frac{1}{n}\sum_{i=1}^n x_{ij}^2 x_{ik}^2 - \esp{x_{ij}^2x_{ik}^2}|\geq t\right) \leq C'd^2e^{-c'nt^{1/4}}\,.
\end{equation*}
Thus, with probability $1-C'd^2e^{-c'nt^{1/4}}$, for all $j\in[d]$,
\begin{equation*}
    \left|\left(\frac{1}{n}\sum_{i=1}^nx_i^2\langle x_i,\beta\rangle^2\right)_j - 2\beta_j^2-\NRM{\beta}_2^2\right| \leq t \sum_{k,\ell} |\beta_k||\beta_\ell|= t\NRM{\beta}_1^2\,.
\end{equation*}
Using the classical technique of \citet{Baraniuk2008}, to make a union bound on all $s$-sparse vectors, we consider an $\eps$-net of the set of $s$-sparse vectors of $\ell^2$-norm smaller than 1. This $\eps$-net is of cardinality less than $(C_0/\eps)^s d^s$, and we only need to take $\eps$ of order 1 to obtain the result for all $s$-sparse vectors. This leads to:
\begin{align*}
    \P\left(\exists \beta\in \R^d \text{ $s$-sparse and }\NRM{\beta}_2\leq 1\,,\,\exists j\in\R^d\,,\quad \left|\left(\frac{1}{n}\sum_{i=1}^nx_i^2\langle x_i,\beta\rangle^2\right)_j - 2\beta_j^2-\NRM{\beta}_2^2\right| \geq t\NRM{\beta}_1^2\right) \\
    \leq C'd^2 e^{c_1 s + s\ln(d)}e^{-c'nt^{1/4}}\,.
\end{align*}
This probability is equal to $C'/d^2$ for $t=\left(\frac{(s+4)\ln(d)+c_1s}{c'n}\right)^4$. We conclude that with probability $1-C'/d^2$, all $s$-sparse vectors $\beta$ satisfy:
\begin{equation*}
    \left|\left(\frac{1}{n}\sum_{i=1}^nx_i^2\langle x_i,\beta\rangle^2\right)_j - 2\beta_j^2-\NRM{\beta}_2^2\right| \leq \left(\frac{(s+4)\ln(d)+c_1s}{c'n}\right)^4\NRM{\beta}_1^2\leq \left(\frac{(s+4)\ln(d)+c_1s}{c'n}\right)^4s\NRM{\beta}_2^2\,,
\end{equation*}
and the RHS is smaller than $\NRM{\beta}_2^2/2$ for $n\geq \Omega(s^{5/4}\ln(d))$.
\end{proof}

\begin{proof}[Proof of \cref{lem:concentration_ripsgd_uncentered}]
We write $x_i=\mu+\sigma z_i$ where $x_i\sim\cN(0,1)$.
We have:
\begin{equation*}
    \P\big(\forall i\in[n],\forall j\in[d],\, |z_{ij}|\geq t\big)\leq e^{\ln(nd)-t^2/2}=\frac{1}{nd}\,,
\end{equation*}
for $t=2\sqrt{\ln(nd)}$. Thus, if $\mu\geq 4\sigma\sqrt{\ln(nd)}$ we have $\frac{\mu}{2} \leq x_i\leq 2\mu$, so that
\begin{equation*}
    \frac{\mu^2}{2n}\sum_i\langle x_i,\beta\rangle^2\leq\frac{1}{n}\sum_i x_i^2\langle x_i,\beta\rangle^2\leq \frac{4\mu^2}{n}\sum_i\langle x_i,\beta\rangle^2 \,.
\end{equation*}
Then, $\langle x_i,\beta\rangle\sim\cN(\langle\mu,\beta\rangle,\sigma^2\NRM{\beta}_2^2)$. For now, we assume that $\NRM{\beta}_2=1$.
We have $\P(|\langle x_i,\beta\rangle^2-\langle\mu,\beta\rangle^2-\sigma^2\NRM{\beta}_2^2|\geq t)\leq Ce^{-ct/\sigma^2}$, and for $t\leq 1$, using concentration of subexponential random variables \cite{vershynin_2018}:
\begin{equation*}
    \P\left( \left|\frac{1}{n}\sum_i \langle x_i,\beta\rangle^2-\langle\mu,\beta\rangle^2-\sigma^2\NRM{\beta}_2^2\right|\geq t\right)\leq C'e^{-nc't^2/\sigma^4}\,,
\end{equation*}
and using the $\eps$-net trick of \citet{Baraniuk2008},
\begin{equation*}
    \P\left( \sup_{\beta\in\cC}\left|\frac{1}{n}\sum_i \langle x_i,\beta\rangle^2-\langle\mu,\beta\rangle^2-\sigma^2\NRM{\beta}_2^2\right|\geq t\right)\leq C'e^{s\ln(d)-nc't^2/\sigma^4}= \frac{C'}{d^2}\,,
\end{equation*}
for $t=\sigma^2\NRM{\beta}_2^2\sqrt{\frac{2(cs+2)\ln(d)}{n}}$.
Consequently, we have, with probability $1-\frac{C'}{d^2}-\frac{1}{nd}$:
\begin{equation*}
    \frac{\mu^2}{2} \left(\langle\mu,\beta\rangle^2+\frac{1}{2}\sigma^2\NRM{\beta}_2^2 \right)\leq\frac{1}{n}\sum_i x_i^2\langle x_i,\beta\rangle^2\leq4\mu^2\left(\langle\mu,\beta\rangle^2+2\sigma^2\NRM{\beta}_2^2 \right)\,.
\end{equation*}
\end{proof}

    \begin{proof}[Proof of Lemma~\ref{lem:concentration2}]
    First, we write $x_i=\mu\one+\sigma z_i$, where $z_i\sim\cN(0,I)$, leading to:
\begin{align*}
    \frac{1}{n}\sum_{i\in[n]}\NRM{x_i}_2^2 x_ix_i^\top &=\frac{1}{n}\sum_{i\in[n]}\big(\sigma^2\NRM{z_i}_2^2 + d\mu^2 + 2\sigma\mu\langle\one,z_i\rangle\big)  x_ix_i^\top 
\end{align*}
    We use concentration of $\chi_d^2$ random variables around $d$:
    \begin{equation*}
        \P(\chi^2_d>d + 2t+ 2\sqrt{dt})\geq t)\leq e^{-t} \quad\text{and}\quad \P(\chi^2_d>d  - 2\sqrt{dt})\leq t)\leq e^{-t}\,,
    \end{equation*}
    so that for all $i\in[n]$,
    \begin{equation*}
        \P(\NRM{z_i}_2^2\notin[d  - 2\sqrt{dt},d + 2t+ 2\sqrt{dt}] )\leq 2e^{-t}\,.
    \end{equation*}
    Thus,
    \begin{equation*}
        \P(\forall i\in[n], \,\NRM{z_i}_2^2\in[d  - 2\sqrt{dt},d + 2t+ 2\sqrt{dt}] ) \geq 1-2ne^{-t}\,.
    \end{equation*}
    Taking $t=d/16$, 
    \begin{equation*}
        \P(\forall i\in[n], \,\NRM{z_i}_2^2\in[\frac{d}{2},13d/8] ) \geq 1-2ne^{-d/16}\,.
    \end{equation*}
    Then, for all $i$, $\langle \one,z_i\rangle$ is of law $\cN(0,d)$, so that $\P(|\langle \one,z_i\rangle|\geq t)\leq2e^{-t^2/(2d)} $ and
    \begin{equation*}
        \P\big(\forall i\in[n],\,|\langle \one,z_i\rangle|\geq t\big)\leq 2ne^{-\frac{t^2}{2d}}\,.
    \end{equation*}
    Taking $t=\sqrt{2}d^{3/4}$,
    \begin{equation*}
                \P\big(\forall i\in[n],\,|\langle \one,z_i\rangle|\geq d^{3/4}\big)\leq 2ne^{-d^{1/2}}\,.
    \end{equation*}
    Thus, with probability $1-2n(e^{-d/16}+e^{-\sqrt{d}}$, we have $\forall i\in[n],\,|\langle \one,z_i\rangle|\geq d^{3/4}$ and $\NRM{z_i}_2^2\in[\frac{d}{2},13d/8]$, so that
    \begin{equation*}
        \big(\frac{d}{2}\sigma^2 + d\mu^2 -2\mu\sigma d^{3/4}\big) H\preceq \tilde H \preceq \big(\frac{13d}{8}\sigma^2 + d\mu^2 +2\mu\sigma d^{3/4}\big) H\,,
    \end{equation*}
    leading to the desired result.
\end{proof}

\begin{proof}[Proof of Lemma~\ref{lem:concentration3bis}]
We have:
\begin{align*}
    \tilde H_b&= \esp{\frac{1}{b^2}\sum_{i,j\in\cB}\langle x_i,x_j\rangle x_ix_j^\top}\\
    &= \esp{\frac{1}{b^2}\sum_{i\in\cB}\NRM{x_i}_2^2 x_ix_i^\top + \frac{1}{b^2}\sum_{i,j\in\cB,\,i\ne j}\langle x_i,x_j\rangle x_ix_j^\top}\\
    &=\frac{1}{b^2}\sum_{i\in[n]}\P(i\in\cB) \NRM{x_i}_2^2 x_ix_i^\top + \frac{1}{b^2} \sum_{i\ne j} \P(i,j\in\cB) \langle x_i,x_j\rangle x_ix_j^\top\,.
\end{align*}
Then, since $\P(i\in\cB)=\frac{b}{n}$ and $\P(i,j\in\cB)=\frac{b(b-1)}{n(n-1)}$ for $i\ne j$, we get that:
\begin{equation*}
    \tilde H_b=\frac{1}{bn}\sum_{i\in[n]}\NRM{x_i}_2^2 x_ix_i^\top + \frac{(b-1)}{bn(n-1)} \sum_{i\ne j} \langle x_i,x_j\rangle x_ix_j^\top\,.
\end{equation*}
Using Lemma~\ref{lem:concentration2}, the first term satisfies:
\begin{equation*}
    \P \Big( \frac{d(\mu^2+\sigma^2)}{b} C_2 H \preceq  \frac{1}{bn}\sum_{i\in[n]}\NRM{x_i}_2^2 x_ix_i^\top \preceq  \frac{d(\mu^2+\sigma^2)}{b} C_3 H\Big)\geq 1- 2ne^{-d/16}\,.
\end{equation*}
We now show that the second term is of smaller order.
Writing $x_i=\mu\one+\sigma z_i$ where $z_i\sim\cN(0,I_d)$, we have:
\begin{align*}
    \frac{(b-1)}{bn(n-1)} \sum_{i\ne j} \langle x_i,x_j\rangle x_ix_j^\top=\frac{(b-1)}{bn(n-1)} \sum_{i\ne j} \langle x_i,x_j\rangle x_ix_j^\top
\end{align*}

For $i\ne j$, $\langle x_i,x_j\rangle=\sum_{k=1}^d x_{ik}x_{jk}=\sum_{k=1}^da_k$ where $a_k=x_{ik}x_{jk}$ satisfies $\E a_k=0$, $\E a_k^2=1$ and $\P(a_k\geq t)\leq 2\P(|x_{ik}|\geq \sqrt{t})\leq 4 e^{-t/2}$.
Hence, $a_k$ is a centered subexponential random variables. Using concentration of  subexponential random variables \cite{vershynin_2018}, for $t\leq 1$,
\begin{equation*}
    \P\left(\frac{1}{d}| \langle x_i,x_j\rangle | \geq t\right)\leq 2e^{-cdt^2}\,.
\end{equation*}
Thus, 
\begin{equation*}
    \P\left(\forall i\ne j,\,\frac{1}{d}| \langle x_i,x_j\rangle | \leq t\right) \geq 1-n(n-1)e^{-cdt^2}\,.
\end{equation*}
Then, taking $t=d^{-1/2}4\ln(n)/c$, we have:
\begin{equation*}
        \P\left(\forall i\ne j,\,\frac{1}{d}| \langle x_i,x_j\rangle | \leq \frac{4\ln(n)}{c\sqrt{d}}\right) \geq 1-\frac{1}{n^2}\,.
\end{equation*}
Going back to our second term,
\begin{align*}
    \frac{(b-1)}{bn(n-1)} \sum_{i\ne j} \langle x_i,x_j\rangle x_ix_j^\top &= \frac{(b-1)}{bn(n-1)} \sum_{i<j} \langle x_i,x_j\rangle \big(x_ix_j^\top+x_jx_i^\top\big)\\
    &\preceq  \frac{(b-1)}{bn(n-1)} \sum_{i<j} \big|\langle x_i,x_j\rangle\big| \big(x_ix_i^\top+x_jx_j^\top\big)\,,
\end{align*}
where we used $x_ix_j^\top+x_jx_i^\top\preceq x_ix_i^\top+x_jx_j^\top$. Thus,
\begin{align*}
    \frac{(b-1)}{bn(n-1)} \sum_{i\ne j} \langle x_i,x_j\rangle x_ix_j^\top&\preceq \sup_{i\ne j}|\langle x_i,x_j\rangle| \times \frac{(b-1)}{bn(n-1)} \sum_{i<j} \big(x_ix_i^\top+x_jx_j^\top\big)\\
    &=\sup_{i\ne j}|\langle x_i,x_j\rangle| \times \frac{b-1}{b} \frac{1}{n-1} \sum_{i=1}^n x_ix_i^\top\\
    &=\sup_{i\ne j}|\langle x_i,x_j\rangle| \times \frac{b-1}{b} \frac{n}{n-1} H\,.
\end{align*}
Similarly, we have
\begin{equation*}
    \frac{(b-1)}{bn(n-1)} \sum_{i\ne j} \langle x_i,x_j\rangle x_ix_j^\top\succeq -\sup_{i\ne j}|\langle x_i,x_j\rangle| \times \frac{b-1}{b} \frac{n}{n-1} H\,.
\end{equation*}
Hence, with probability $1-1/n^2$,
\begin{equation*}
    - \frac{4\ln(n)}{c\sqrt{d}} \times \frac{b-1}{b} \frac{n}{n-1} H\preceq \frac{(b-1)}{bn(n-1)} \sum_{i\ne j} \langle x_i,x_j\rangle x_ix_j^\top \preceq  \frac{4\ln(n)}{c\sqrt{d}} \times \frac{b-1}{b} \frac{n}{n-1} H\,.
\end{equation*}
Wrapping things up, with probability $1-1/n^2-2ne^{-d/16}$,
\begin{equation*}
        \left(- \frac{4\ln(n)}{c\sqrt{d}}  \frac{b-1}{b} \frac{n}{n-1} + C_2\frac{d}{b}\right)\times H\preceq \tilde H_b \preceq  \left(\frac{4\ln(n)}{c\sqrt{d}}  \frac{b-1}{b} \frac{n}{n-1} + C_3\frac{d}{b}\right)\times H\,.
\end{equation*}
Thus, provided that $\frac{4\ln(n)}{c\sqrt{d}}  \leq \frac{C_2d}{2b}$ and $d\geq 48\ln(n)$, we have with probability $1-3/n^2$:
\begin{equation*}
     C_2'\frac{d}{b}\times H\preceq \tilde H_b \preceq   C_3'\frac{d}{b}\times H\,.
\end{equation*}
\end{proof}

\begin{proof}[Proof of Lemma~\ref{lem:concentration3bis}]
We have:
\begin{align*}
    \tilde H_b&= \esp{\frac{1}{b^2}\sum_{i,j\in\cB}\langle x_i,x_j\rangle x_ix_j^\top}\\
    &= \esp{\frac{1}{b^2}\sum_{i\in\cB}\NRM{x_i}_2^2 x_ix_i^\top + \frac{1}{b^2}\sum_{i,j\in\cB,\,i\ne j}\langle x_i,x_j\rangle x_ix_j^\top}\\
    &=\frac{1}{b^2}\sum_{i\in[n]}\P(i\in\cB) \NRM{x_i}_2^2 x_ix_i^\top + \frac{1}{b^2} \sum_{i\ne j} \P(i,j\in\cB) \langle x_i,x_j\rangle x_ix_j^\top\,.
\end{align*}
Then, since $\P(i\in\cB)=\frac{b}{n}$ and $\P(i,j\in\cB)=\frac{b(b-1)}{n(n-1)}$ for $i\ne j$, we get that:
\begin{equation*}
    \tilde H_b=\frac{1}{bn}\sum_{i\in[n]}\NRM{x_i}_2^2 x_ix_i^\top + \frac{(b-1)}{bn(n-1)} \sum_{i\ne j} \langle x_i,x_j\rangle x_ix_j^\top\,.
\end{equation*}
Using Lemma~\ref{lem:concentration2}, the first term satisfies:
\begin{equation*}
    \P \Big( \frac{d(\mu^2+\sigma^2)}{b} C_2 H \preceq  \frac{1}{bn}\sum_{i\in[n]}\NRM{x_i}_2^2 x_ix_i^\top \preceq  \frac{d(\mu^2+\sigma^2)}{b} C_3 H\Big)\geq 1- 2ne^{-d/16}\,.
\end{equation*}
We now show that the second term is of smaller order.
Writing $x_i=\mu\one+\sigma z_i$ where $z_i\sim\cN(0,I_d)$, we have:
\begin{align*}
    \frac{(b-1)}{bn(n-1)} \sum_{i\ne j} \langle x_i,x_j\rangle x_ix_j^\top&=\frac{(b-1)}{bn(n-1)} \sum_{i\ne j}\big( \sigma^2\langle z_i,z_j\rangle +\sigma\mu\langle\one,z_i+z_j\rangle +\mu^2 d \big)x_ix_j^\top\\
    &= \frac{(b-1)}{bn(n-1)} \sum_{i\ne j}\big( \sigma^2\langle z_i,z_j\rangle +\sigma\mu\langle\one,z_i+z_j\rangle\big)x_ix_j^\top +\frac{(b-1)}{bn(n-1)} \mu^2 d \sum_{i\ne j}x_ix_j^\top
\end{align*}
For $i\ne j$, $\langle z_i,z_j\rangle=\sum_{k=1}^d z_{ik}z_{jk}=\sum_{k=1}^da_k$ where $a_k=z_{ik}z_{jk}$ satisfies $\E a_k=0$, $\E a_k^2=1$ and $\P(a_k\geq t)\leq 2\P(|z_{ik}|\geq \sqrt{t})\leq 4 e^{-t/2}$.
Hence, $a_k$ is a centered subexponential random variables. Using concentration of  subexponential random variables \cite{vershynin_2018}, for $t\leq 1$,
\begin{equation*}
    \P\left(\frac{1}{d}| \langle x_i,x_j\rangle | \geq t\right)\leq 2e^{-cdt^2}\,.
\end{equation*}
Thus, 
\begin{equation*}
    \P\left(\forall i\ne j,\,\frac{1}{d}| \langle x_i,x_j\rangle | \leq t\right) \geq 1-n(n-1)e^{-cdt^2}\,.
\end{equation*}
Then, taking $t=d^{-1/2}4\ln(n)/c$, we have:
\begin{equation*}
        \P\left(\forall i\ne j,\,\frac{1}{d}| \langle x_i,x_j\rangle | \leq \frac{4\ln(n)}{c\sqrt{d}}\right) \geq 1-\frac{1}{n^2}\,.
\end{equation*}
For $i\in[n]$, $\langle \one , z_i\rangle\sim \cN(0,d)$ so that $\P(|\langle \one , z_i\rangle|\geq t)\leq 2e^{-t^2/(2d)}$, and
\begin{equation*}
    \P(\forall i\in[n],\,|\langle \one , z_i\rangle|\leq t)\geq 1-2ne^{-t^2/(2d)}=1-\frac{2}{n^2}\,,
\end{equation*}
for $t=3\sqrt{d}\ln(n)$. Hence, with probability $1-3/n^2$, for all $i\ne j$ we have $|\sigma^2\langle z_i,z_j\rangle +\sigma\mu\langle\one,z_i+z_j\rangle|\leq (\sigma^2+\sigma\mu)C\ln(n)/\sqrt{d}$.

Now,
\begin{align*}
    \frac{(b-1)}{bn(n-1)} \sum_{i\ne j} \big( \sigma^2\langle z_i,z_j\rangle +\sigma\mu\langle\one,z_i+z_j\rangle\big) x_ix_j^\top &=  \frac{(b-1)}{bn(n-1)} \sum_{i< j} \big( \sigma^2\langle z_i,z_j\rangle +\sigma\mu\langle\one,z_i+z_j\rangle\big) (x_ix_j^\top +x_jx_i^\top) \\
    &\preceq  \frac{(b-1)}{bn(n-1)} \sum_{i<j} \big| \sigma^2\langle z_i,z_j\rangle +\sigma\mu\langle\one,z_i+z_j\rangle\big)| \big(x_ix_i^\top+x_jx_j^\top\big)\,,
\end{align*}
where we used $x_ix_j^\top+x_jx_i^\top\preceq x_ix_i^\top+x_jx_j^\top$. Thus,
\begin{align*}
    \frac{(b-1)}{bn(n-1)} \sum_{i\ne j} \big( \sigma^2\langle z_i,z_j\rangle +\sigma\mu\langle\one,z_i+z_j\rangle\big) x_ix_j^\top&\preceq \sup_{i\ne j}\big| \sigma^2\langle z_i,z_j\rangle +\sigma\mu\langle\one,z_i+z_j\rangle\big)| \times \frac{(b-1)}{bn(n-1)} \sum_{i<j} \big(x_ix_i^\top+x_jx_j^\top\big)\\
    &=\sup_{i\ne j}\big| \sigma^2\langle z_i,z_j\rangle +\sigma\mu\langle\one,z_i+z_j\rangle\big| \times \frac{b-1}{b} \frac{1}{n-1} \sum_{i=1}^n x_ix_i^\top\\
    &=\sup_{i\ne j}\big| \sigma^2\langle z_i,z_j\rangle +\sigma\mu\langle\one,z_i+z_j\rangle\big| \times \frac{b-1}{b} \frac{n}{n-1} H\,.
\end{align*}
Similarly, we have
\begin{equation*}
    \frac{(b-1)}{bn(n-1)} \sum_{i\ne j} \big( \sigma^2\langle z_i,z_j\rangle +\sigma\mu\langle\one,z_i+z_j\rangle\big) x_ix_j^\top\succeq -\sup_{i\ne j}\big| \sigma^2\langle z_i,z_j\rangle +\sigma\mu\langle\one,z_i+z_j\rangle\big)| \times \frac{b-1}{b} \frac{n}{n-1} H\,.
\end{equation*}
Hence, with probability $1-3/n^2$,
\begin{align*}
    - \frac{(\sigma^2+\sigma\mu)C\ln(n)}{\sqrt{d}} \times \frac{b-1}{b} \frac{n}{n-1} H\preceq \frac{(b-1)}{bn(n-1)} \sum_{i\ne j} \big( \sigma^2\langle z_i,z_j\rangle +\sigma\mu\langle\one,z_i+z_j\rangle\big) x_ix_j^\top \\
    \preceq  \frac{(\sigma^2+\sigma\mu)C\ln(n)}{\sqrt{d}}  \times \frac{b-1}{b} \frac{n}{n-1} H\,.
\end{align*}
We thus have shown that this term (the one in the middle of the above inequality) is of smaller order.

We are hence left with $\frac{(b-1)}{bn(n-1)} \mu^2 d\sum_{i\ne j}x_i x_j^\top$.
Denoting $\bar x=\frac{1}{n}\sum_i x_i$, we have $\frac{1}{n^2}\sum_{i\ne j}x_ix_j^\top = \frac{1}{n^2}\sum_{i, j}x_ix_j^\top  -\frac{1}{n^2}\sum_{i}x_ix_i^\top $, so that:
\begin{equation*}
    \frac{(b-1)}{bn(n-1)} \mu^2 d\sum_{i\ne j}x_i x_j^\top = \frac{(b-1)n}{b(n-1)}\mu^2d \left( \bar x\bar x^\top - \frac{1}{n}H \right)\,.
\end{equation*}
We note that we have $H= \frac{1}{n}\sum_i  x_i x_i^\top = \frac{1}{n^2}\sum_{i<j}x_ix_i^\top + x_jx_j^\top\succeq \frac{1}{n^2}\sum_{i<j} x_i x_j^\top+x_jx_i^\top = \bar x\bar x^\top $ using $x_ix_i^\top + x_jx_j^\top \succeq x_ix_j^\top + x_jx_i^\top$. Thus, $H\succeq \bar x\bar x^\top\succeq 0$, and:
\begin{equation*}
    -\frac{(b-1)n}{b(n-1)} \mu^2 d \frac{1}{n} H\preceq \frac{(b-1)}{bn(n-1)} \mu^2 d\sum_{i\ne j}x_i x_j^\top \preceq \frac{(b-1)n}{b(n-1)} \mu^2 d (1-1/n) H\,.
\end{equation*}

We are now able to wrap everything together.
With probability $1-2ne^{-d/16}-3/n^2$, we have, for some numerical constants $c_1,c_2,c_3,C>0$:
\begin{equation*}
    \left( c_1\frac{d(\mu^2+\sigma^2)}{b} - c_2\frac{(\sigma^2+\mu^2)\ln(n)}{\sqrt{d}} - c_3\frac{\mu^2 d}{n}\right)H \preceq \tilde H_b \preceq C\left( \frac{d(\mu^2+\sigma^2)}{b} + \frac{(\sigma^2+\mu^2)\ln(n)}{\sqrt{d}} + \mu^2 d\right)
\end{equation*}

\end{proof}

\end{document}